\documentclass[twoside,11pt]{article}

\usepackage{blindtext}

\usepackage[preprint]{jmlr2e}

\usepackage{lastpage}

\ShortHeadings{Highway Reinforcement Learning}{Wang, Strupl, Faccio, Wu, Liu,  Grudzien, Tan, Schmidhuber}
\firstpageno{1}

\usepackage{arydshln}

\usepackage{amsmath,amsfonts,bm}

\def\eqref#1{equation~\ref{#1}}

\def\1{\bm{1}}

\DeclareMathAlphabet{\mathsfit}{\encodingdefault}{\sfdefault}{m}{sl}
\SetMathAlphabet{\mathsfit}{bold}{\encodingdefault}{\sfdefault}{bx}{n}

\def\sZ{{\mathbb{Z}}}

\newcommand{\E}{\mathbb{E}}

\newcommand{\R}{\mathbb{R}}

\newcommand{\softmax}{\mathrm{softmax}}

\DeclareMathOperator*{\argmax}{arg\,max}

\NewDocumentCommand{\rangeInt}{ O{a} O{b} }{ [ #1..#2 ] }

\usepackage{url}

\usepackage[dvipsnames]{xcolor}

\definecolor{usccardinal}{rgb}{0.6, 0.0, 0.0}

\newif\ifshowtmp
\showtmptrue

\ifshowtmp
    \newcommand{\todom}[1]{{\textcolor{usccardinal}{(todo minor: #1)}}}
	\newcommand{\todo}[1]{{\textcolor{red}{(TODO: #1)}}}
    \newcommand{\tocheck}[1]{{\textcolor{red}{(TO CHECK: #1)}}}

	\newcommand{\T}[1]{\textcolor{usccardinal}{#1}}

    \makeatletter
    \def\w#1 {\T{#1}~}
    \makeatother

\else
    \newcommand{\todom}[1]{}
    \newcommand{\tocheck}[1]{}
	
	\newcommand{\T}{}
	\newcommand{\todo}[1]{}

    \newcommand{\w}{}
\fi
\newcommand{\todoLessImportant}[1]{}
\newcommand{\removeForSavingSpace}[1]{}

\newif\ifPPT
\PPTtrue

\newif\ifshowstruct
\showstructtrue
\showstructfalse %

\ifshowstruct
	\newcommand{\tstructz}[1]{[#1] }
\else
	\newcommand{\tstructz}[1]{}
\fi

\label{packages}
\usepackage{mfirstuc}
\usepackage{graphicx}%
\usepackage{subcaption}
\usepackage{caption}

\usepackage{multirow}
\usepackage{multicol}
\usepackage{array}
\usepackage{booktabs}%

\usepackage{amsmath}
\usepackage{amssymb, graphicx, url, mathtools}%
\usepackage[overload]{empheq}

\usepackage{algorithmic,algorithm}   %

\usepackage{footnote}
\makesavenoteenv{tabular}
\makesavenoteenv{table}
\usepackage{pdfpages}
\usepackage{textcomp}
\usepackage{chngcntr}
\usepackage{verbatim}

\usepackage{xparse,calc}
\usepackage{afterpage}
\usepackage{comment}

\usepackage{hyperref}
\usepackage{cleveref}
\hypersetup{
    colorlinks,
    linkcolor={black},
    citecolor={black},
    urlcolor={black}
}

\label{Command}

\counterwithin*{question}{section}

\makeatletter
\def\renewtheorem#1{%
  \expandafter\let\csname#1\endcsname\relax
  \expandafter\let\csname c@#1\endcsname\relax
  \gdef\renewtheorem@envname{#1}
  \renewtheorem@secpar
}
\def\renewtheorem@secpar{\@ifnextchar[{\renewtheorem@numberedlike}{\renewtheorem@nonumberedlike}}
\def\renewtheorem@numberedlike[#1]#2{\newtheorem{\renewtheorem@envname}[#1]{#2}}
\def\renewtheorem@nonumberedlike#1{  
\def\renewtheorem@caption{#1}
\edef\renewtheorem@nowithin{\noexpand\newtheorem{\renewtheorem@envname}{\renewtheorem@caption}}
\renewtheorem@thirdpar
}
\def\renewtheorem@thirdpar{\@ifnextchar[{\renewtheorem@within}{\renewtheorem@nowithin}}
\def\renewtheorem@within[#1]{\renewtheorem@nowithin[#1]}
\makeatother

\renewtheorem{theorem}{Theorem}
\renewtheorem{remark}{Remark}
\renewtheorem{definition}{Definition}
\renewtheorem{example}{Example}
\newtheorem{mdpexample}{MDP Example}
\crefname{mdpexample}{MDP Example}{MDP Examples}

\usepackage{thmtools} 
\usepackage{thm-restate}

\newcommand{\EE}{\mathbb{E}}

\usepackage{xstring}

\newcolumntype{+}{>{\global\let\currentrowstyle\relax}}
\newcolumntype{Y}{>{\currentrowstyle}}\label{key}
\newcolumntype{-}{>{\currentrowstyle}}

\newcommand{\breakingcomma}{
  \begingroup\lccode`~=`,
  \lowercase{\endgroup\expandafter\def\expandafter~\expandafter{~\penalty0 }}
}

\DeclareFontFamily{U}{mathx}{\hyphenchar\font45}
\DeclareFontShape{U}{mathx}{m}{n}{<-> mathx10}{}
\DeclareSymbolFont{mathx}{U}{mathx}{m}{n}
\DeclareMathAccent{\widebar}{0}{mathx}{"73}

\label{Settings}

\usepackage{makecell}

\Crefname{figure}{Fig.}{Fig.}
\Crefname{equation}{Eq.}{Eq.}
\Crefname{assumption}{Assumption}{assumption}

\newcommand{\crefnop}[1]{\namecref{#1} \ref{#1}}

\newcommand{\Crefnop}[1]{\nameCref{#1} \ref{#1}}

\newcommand{\reftodo}[1]{\todo{cref}}

\newcommand{\Creftodo}[1]{\todo{cref}}
\newcommand{\creftodo}[1]{\todo{cref}}
\newcommand{\citetodo}[1]{\todo{cite}}

\newcommand{\Creftd}[1]{\todo{cref}}
\newcommand{\creftd}[1]{\todo{cref}}
\newcommand{\citetd}[1]{\todo{cite}}

\usepackage[export]{adjustbox}

\usepackage{xcolor}
\definecolor{bluemy}{HTML}{1f77b4}

\usepackage{enumitem}

    \usepackage{xr}
    \makeatletter
    \newcommand*{\addFileDependency}[1]{
      \typeout{(#1)}
      \@addtofilelist{#1}
      \IfFileExists{#1}{}{\typeout{No file #1.}}
    }
    \makeatother

\makeatletter
\def\widebreve{\mathpalette\wide@breve}
\def\wide@breve#1#2{\sbox\z@{$#1#2$}%
     \mathop{\vbox{\m@th\ialign{##\crcr
\kern0.08em\brevefill#1{0.8\wd\z@}\crcr\noalign{\nointerlineskip}%
                    $\hss#1#2\hss$\crcr}}}\limits}
\def\brevefill#1#2{$\m@th\sbox\tw@{$#1($}%
  \hss\resizebox{#2}{\wd\tw@}{\rotatebox[origin=c]{90}{\upshape(}}\hss$}
\makeatletter

\newcommand\scaleMathInside[2]{\scalebox{#1}{\mbox{\ensuremath{\displaystyle #2}}}}

\usepackage{accents}

\usepackage[makeroom]{cancel}

\definecolor{blueMy}{HTML}{1f77b4}
\definecolor{orangeMy}{HTML}{ff7f0e}
\definecolor{greenMy}{HTML}{2ca02c}
\definecolor{redMy}{HTML}{d62728}
\definecolor{purpleMy}{HTML}{9467bd}
\definecolor{brownMy}{HTML}{8c564b}

\NewDocumentCommand{\characteristic}{ O{} }{\emph{(#1)}}

\def\softmax{\mathop{smax}\nolimits^{\alpha}}

\NewDocumentCommand{\nstepReturn}{O{Q} O{n}   }{ {G}_{#1}^{#2} }
\NewDocumentCommand{\nstepReturnAvg}{O{Q} O{n}   }{ {\bar G}_{#1}^{#2} }

\def\lanpaper/{[Lan, 2020]}
\def\youngpaper/{[Young, 2019]}

\NewDocumentCommand{\K}{ O{k} }{ _{#1} }

\newcommand{\sSpace}{  \mathcal{S}} 
\newcommand{\aSpace}{  \mathcal{A}}

\definecolor{expectationPolicyStep}{named}{black}
\definecolor{maxStepOneN}{named}{red}
\definecolor{maxStep}{named}{maxStepOneN}
\definecolor{maxPolicyStep}{named}{blue}
\definecolor{softmaxPolicyStep}{HTML}{4682B4}
\definecolor{stepFirst}{named}{pink}
\NewDocumentCommand{\EPolicyStep}{ O{expectationPolicyStep}  }
{{
    \color{#1}
    \E_{ \pi \sim \policyDist,  n  \sim \stepDist }
}}

\NewDocumentCommand{\MaxPolicyMaxStep}{  }
{{
    { \color{maxPolicyStep} 
       \max\limits_{\pi \in \policySet } 
    }\,
   {
    \color{maxPolicyStep} 
    \max\limits_{ n \in \stepSet }
   }
    {
        \color{maxStepOneN}
        \max\limits_{n' \in \{1, n\} }
    }
}}

\NewDocumentCommand{\SoftmaxPolicySoftmaxStep}{ }
{{
    { \color{softmaxPolicyStep} 
    \underset{\pi\in \policySet } {\mathop {smax}\nolimits^{\alpha} }
    }
   { 
   \color{softmaxPolicyStep}
    \underset{n\in  \stepSet } {\mathop {smax}\nolimits^{\alpha} }
   }
    {
        \color{maxStepOneN}
        \max\limits_{n' \in \{1, n\} }
    }
}}

\NewDocumentCommand{\dataset}{ O{m} O{s,a} }{ {\cal D}_{#2}^{\ifthenelse{\equal{#1}{}}{}{(#1)}} }
\NewDocumentCommand{\traj}{ O{\pi} O{s_t,a_t} }{ {\tau}_{#2}^{#1} }
\NewDocumentCommand{\trajnew}{ O{\pi} O{s_t,a_t} O{n} }{ {\tau}_{#2}^{#3} \sim {#1} }

\def\step/{lookahead depth}
\def\steps/{lookahead depths}
\def\stepText/{lookahead depth}
\def\stepsText/{lookahead depths}

\def\stepSet{{\cal N}}
\def\stepSetMath/{ $\stepSet$ } 

\def\stepSetText/{set of \steps/}
\def\stepSetTextMath/{set of \steps/ $\stepSet$}
\def\StepSetText/{Set of \steps/}

\def\policyText/{behavioral policy}
\def\policiesText/{behavioral policies}
\def\policySetText/{set of behavioral policies}
\def\PolicySetText/{Set of behavioral policies}
\def\policyDistText/{selection distribution of behavioral policies }

\NewDocumentCommand{\policyDist}{ O{} O{\policySet} }{ \mathcal{P}^{#1}_{#2} }

\NewDocumentCommand{\stepDist}{ O{} O{\stepSet} }{ \mathcal{P}^{#1}_{#2} }

\NewDocumentCommand{\policySet}{ }{ \widehat{\Pi} }

\NewDocumentCommand{\policySetModelFree}{ }{ {\mathcal{M}} }

\def\policySetMath/ { $\policySet$ }

\def\policySetTextMath/{\policySetText/ $\policySet$}
\def\policyDistTextMath/{\policyDistText/ $\policyDist$}

\NewDocumentCommand{\dist}{ O{} O{} }{ \mathcal{P}^{#1}_{#2} }

\NewDocumentCommand{\bellmanOperator}{ O{} O{} }{ \mathcal{B}^{#1}_{#2} }

\NewDocumentCommand{\bellmanOptOperator}{ O{} O{} }{ \mathcal{B}^{#1}_{#2} }
\NewDocumentCommand{\BOOperator}{ O{} O{} }{ \mathcal{B}^{#1}_{#2} }
\NewDocumentCommand{\multistepBOOperator}{ O{\policyDist} O{\stepDist} }{\overline{\mathcal{B}}^{#1}_{#2} }
\def\multistepBOOperatorText/{Multi-step BO Operator}
\def\multistepBOOperatorTextMath/{Multi-step BO Operator $\multistepBOOperator[][]$}

\def\multistepBEOperatorText/{Multi-step IS-based BE Operator}
\NewDocumentCommand{\multistepBEOperator}{ O{\policyDist} O{\stepDist} O{} }{ \overset{ {\tiny #3} }{ \widebreve{ \mathcal{B} } }^{#1}_{#2} }
\def\multistepBEOperatorTextMath/{\multistepBEOperatorText/ $\multistepBEOperator[][]$}

\def\stepIndex{{\check{n}}}
\def\stepIndexTwo{{\check{n}}^\prime}

\def\stStepSet/{\text{s.t. }}

\def\highwayPrefix/{Highway}
\def\highwayPrefixFull/{Highway}

\def\highwayQLearning/{\highwayPrefix/ Q-Learning}
\def\highwayDQN/{\highwayPrefix/ DQN}

\def\softHighwayDQN/{Soft \highwayPrefix/ DQN}
\def\nstepHighwayDQN/{$n$-step \highwayPrefix/ DQN}

\def\highwayQLearningFull/{\highwayPrefixFull/ Q-Learning}
\def\highwayDQNFull/{\highwayPrefixFull/ DQN}

\def\highwayValueIteration/{\highwayPrefixFull/ Value Iteration}

\def\highwayOperatorText/{\highwayPrefix/ Operator}

\def\highwayGenOperatorText/{\highwayPrefix/ Generalized Operator}

\NewDocumentCommand{\policy}{ O{} }{   \pi^{(#1)}  }
\NewDocumentCommand{\highwayOperator}{ O{\policyDist} O{\stepDist} }{   {\overline{{\mathcal{G}}}}^{ {#1} }_{ {#2} }   }

\NewDocumentCommand{\highwayOperatorQ}{ O{\policyDist} O{\stepDist} }{   \overline{{\mathcal{G}}}^{ {#1} }_{ {#2} }   }
\NewDocumentCommand{\wronghighwayOperator}{ O{\policyDist} O{\stepDist} }{   {\bcancel{\overline{\mathcal{G}}} }^{ {#1} }_{ {#2} }   }

\NewDocumentCommand{\highwayOptOperator}{ O{\policySet} O{\stepSet} }{   {\accentset{\mbox{\large\bfseries .}}{\mathcal{G}}}^{ {#1} }_{ {#2} }   }
\def\highwayOptOperatorText/{Highway Optimality Operator}
\def\highwayOptOperatorTextMath/{Highway Optimality Operator $\highwayOptOperator[][]$}
\def\highwayOperatorTextMath/{\highwayOperatorText/ $\highwayOperator[][]$}
\def\highwayGenOperatorTextMath/{\highwayGenOperatorText/ $\highwayOperator[][]$}

\def\highwaySoftmaxOperatorText/{\highwayPrefix/ Softmax Operator}
\NewDocumentCommand{\highwaySoftmaxOperatorMath}{ O{\stepSet} O{\policySet,\alpha} }{   {\widetilde{\mathcal{G}}}_{ {#1} }^{ {#2} }   }
\def\highwaySoftmaxOperatorTextMath/{\highwaySoftmaxOperatorText/ $\highwaySoftmaxOperatorMath$}

\def\highwayEquation/{\highwayPrefix/ Equation}
\def\highwayEquationFull/{\highwayPrefixFull/ Bellman Optimality Equation}

\def\QQQQ{}

\def\highwayQQPrefix/{\highwayPrefix/ $Q$}
\def\highwayQQPrefixFull/{\highwayPrefixFull/ $Q$}
\def\highwayQQOperatorText/{\highwayQQPrefix/ Operator}
\def\highwayQQOperatorTextFull/{\highwayQQPrefixFull/ Bellman Optimality Operator}
\def\highwayQQOperatorTextMath/{\highwayQQOperatorText/ $\highwayQQOperator$}

\def\highwayQQEquation/{\highwayQQPrefix/ Equation}
\def\highwayQQEquationFull/{\highwayQQPrefixFull/ Bellman Optimality Equation}

\def\greedyReturn/{\highwayPrefix/ return}
\def\GreedyReturn/{\HighwayPrefix Return}

\NewDocumentCommand{\distance}{ O{} }{ \mathbf{d}_{#1}^Q(s,a) } 
\NewDocumentCommand{\distanceAll}{ O{} }{ d_{#1}^{Q} } 

\NewDocumentCommand{\assumptionAn}{ O{n} O{\widehat\Pi} O{\pi^*} }{Assumption A$(#2,#1)$}

\NewDocumentCommand{\stateSetNStep}{ O{s} O{\pi} O{n} }{ \mathcal{S}_{#1, #3}^{#2} }

\NewDocumentCommand{\ourExpOpt}{ O{\widehat{\Pi}} O{\mathcal{N}} }{ \overline{\mathcal{G}}_{#1}^{#2} }
\def\ourExpOptText/{Expectation Highway Operator}

\def\oneStepOperatorText/{BO operator}
\NewDocumentCommand{\oneStepOperatorMath}{ O{}  }{ \multiStepOperatorMath[#1][] }
\NewDocumentCommand{\oneStepOperator}{ O{}  }{ \multiStepOperatorMath[#1][] }
\def\oneStepOperatorTextMath/{\oneStepOperatorText/ $\oneStepOperatorMath$}

\def\bellmanOptText/{Bellman Optimality Operator}
\NewDocumentCommand{\bellmanOpt}{ O{}  }{ \mathcal{B}^{#1} }
\def\bellmanExpText/{Bellman Expectation Operator}
\NewDocumentCommand{\bellmanExp}{ O{#1}  }{ \bellmanOpt[#1] }
\NewDocumentCommand{\bellmanIS}{ O{} O{}  }{{\breve{\mathcal{B}}}^{#1}_{#2} }

\def\oneStepQQOperatorText/{Bellman Optimality Operator}
\NewDocumentCommand{\oneStepQQOperator}{ O{}  }{ \multiStepQQOperator[#1][] }
\NewDocumentCommand{\oneStepQQOperatorMath}{ O{}  }{ \multiStepQQOperator[#1][] }
\def\oneStepQQOperatorTextMath/{\oneStepQQOperatorText/ $\oneStepQQOperatorMath$}

\NewDocumentCommand{\multiStepOnPolicyOperator}{ O{N} O{\pi} }{ {\mathcal{B}}^{{#2}}_{#1} }
\NewDocumentCommand{\multiStepOnPolicyOperatorMath}{ O{N} O{\pi} }{ \multiStepOnPolicyOperator[#1][#2] }

\def\multiStepOnPolicyOperatorText/{Multi-Step Bellman Expectation Operator}
\def\multiStepOnPolicyOperatorTextMath/{ \multiStepOnPolicyOperatorText/ $\multiStepOnPolicyOperatorMath$ }

\NewDocumentCommand{\multiStepOnPolicyQQOperatorMath}{ O{N} O{\pi} }{ \QQQQ{\mathcal{B}}^{{#2}}_{#1} }

\NewDocumentCommand{\multiStepOffPolicyOperatorMath}{ O{N} O{\pi} O{\policyDist} }{ {\mathcal{B}}^{{#2}}_{#1} }
\NewDocumentCommand{\multiStepOffPolicyQQOperatorMath}{ O{N} O{\pi} O{\policySet} }{ \QQQQ{\mathcal{B}}^{{#2,#3}}_{#1} }

\NewDocumentCommand{\nreturnWeight}{ O{n} O{\pi}  }{ w_{#1}^{#2} }

\def\multiStepOperatorText/{Multi-step Bellman Optimality operator}
\NewDocumentCommand{\multiStepOperator}{ O{N} O{\policySet} }{ \mathcal{B}^{{#1}}_{#2} }
\NewDocumentCommand{\multiStepOperatorMath}{ O{N} O{\policySet} }{ \multiStepOperator[#1][#2] }
\def\multiStepOperatorTextMath/{\multiStepOperatorText/ $\multiStepOperatorMath$}

\def\multiStepQQOperatorText/{\multiStepOperatorText/}
\NewDocumentCommand{\multiStepQQOperator}{ O{N} O{\policyDist} }{ \QQQQ{\mathcal{B}}^{{#1}}_{#2} }
\NewDocumentCommand{\multiStepQQOperatorMath}{ O{N} O{\pi} }{ \multiStepQQOperator[#2][#1] }
\def\multiStepQQOperatorTextMath/{\multiStepOperatorText/ $\multiStepQQOperatorMath$}

\def\highwayPolicyIteration/{\highwayPrefixFull/ Policy Iteration}

\def\HighwayEvaluatePrefix/{Highway}
\def\highwayEvaluatePrefix/{Highway}
\def\highwayEvaluatePrefixFull/{Highway}
\def\highwayEvaluateValueIteration/{\highwayPrefixFull/ Value Iteration}
\def\highwayEvaluateEquationFull/{\highwayPrefixFull/ Recombined Equation}

\def\highwayEvaluateEquation/{\highwayPrefix/ Recombined Equation}
\def\highwayEvaluateOperatorText/{\highwayPrefix/ Recombined Operator}
\def\highwayEvaluateOperatorTextFull/{\highwayPrefixFull/ Recombined Operator}
\NewDocumentCommand{\highwayEvaluateOperatorMath}{ O{N} O{\policySet} }{ \mathcal{G}^{ {#1} }_{ {#2} } }

\def\recombinedPolicySetText/{set of recombined behavioral policies}

\NewDocumentCommand{\highwayRecombinationOperatorMath}{ O{N} }{ \mathcal{G}^{{#1}} }

\definecolor{QPiFixedPoint}{named}{usccardinal}

\def\nrandomseed/{5}

\usepackage{pgf,tikz}

\usepackage{url}
\usepackage{usebib}
\newbibfield{author} %
\bibinput{lib_my/bib}

\title{Highway Reinforcement Learning}

\author{\name Yuhui Wang$^1$ 
\email yuhui.wang@kaust.edu.sa 
\vspace{0.07cm}
\\
\name  Miroslav Strupl$^2$   
\email  miroslav.strupl@idsia.ch
\vspace{0.07cm}
\\
\name  Francesco Faccio$^{1,2}$ 
\email  francesco@idsia.ch
\vspace{0.07cm}
\\
\name  Qingyuan Wu$^3$ 
\email  qingwu2@liverpool.ac.uk
\vspace{0.07cm}
\\
\name  Haozhe Liu$^1$  
\email  haozhe.liu@kaust.edu.sa
\vspace{0.07cm}
\\ 
\name  Michał Grudzień$^4$
\email  michal.grudzien@worc.ox.ac.uk
\vspace{0.07cm}
\\
\name  Xiaoyang Tan$^5$
\email  x.tan@nuaa.edu.cn
\vspace{0.07cm}
\\
\name  J\"urgen Schmidhuber$^{1,2, 6}$
\email  juergen.schmidhuber@kaust.edu.sa\vspace{0.2cm}
\\
$^1$ {\small AI Initiative, King Abdullah University of Science and Technology, Saudi Arabia} \\
$^2$ {\small The Swiss AI Lab IDSIA/USI/SUPSI, Switzerland }\\
$^3$ {\small The University of Liverpool, England} \\
$^4$ {\small University of Oxford, England} \\
$^5$ {\small Nanjing University of Aeronautics and Astronautics, China}\\
$^6$ {\small NNAISENSE}
}

\DeclareMathOperator{\supp}{\mathrm{supp}}

\begin{document}

\maketitle

\begin{abstract}

Learning from multi-step off-policy data collected by a set of policies is a core problem of reinforcement learning (RL). 
Approaches based on importance sampling (IS) often suffer from large variances due to products of IS ratios. 
Typical IS-free methods, such as $n$-step Q-learning, look ahead for $n$ time steps along the trajectory of actions (where $n$ is called the lookahead depth) and utilize off-policy data directly without any additional adjustment. 
They work well for proper choices of $n$. 
We show, however, that such IS-free methods underestimate the optimal value function (VF), especially for large $n$, restricting their capacity to efficiently utilize information from distant future time steps. 
To overcome this problem, we introduce a novel, IS-free, multi-step off-policy method that avoids the underestimation issue and converges to the optimal VF.
At its core lies a simple but non-trivial \emph{highway gate}, 
which controls the information flow from the distant future by comparing it to a threshold.
The highway gate guarantees convergence to the optimal VF  for arbitrary $n$ and arbitrary behavioral policies. It gives rise to a novel family of off-policy RL algorithms that safely learn even when $n$ is very large, facilitating rapid credit assignment from the far future to the past.
On tasks with greatly delayed rewards, including video games where the reward is given only at the end of the game, our new methods outperform many existing multi-step off-policy algorithms.

\end{abstract}

\section{Introduction}
Recent advances in reinforcement learning (RL) have achieved remarkable empirical success \citep{horgan2018distributed, barth2018distributed,degrave2022magnetic,wurman2022outracing}. 
However, addressing environments where rewards are significantly delayed remains a challenge. In such settings, there is often a time lag between pivotal actions and their associated rewards \citep{kaelbling1996reinforcement,rlblogpost,han2022off,NEURIPS2019_16105fb9,zhang2023multi}.
This delay poses a significant challenge in effectively attributing rewards to the actions that actually contributed to them. In RL, this attribution is critical for learning optimal policies. Herein lies the importance of Value Functions (VFs)---a fundamental concept in RL that aids in this attribution process.
Value Functions estimate the expected reward for states (and actions in certain algorithms) under a specific policy, offering a way to quantify the long-term impact of actions. They enable an agent to predict future rewards and thus, assign credit to actions even when rewards are not immediate.
Classical one-step bootstrapping methods, such as Q-Learning, suffer from slow backward credit assignment: one time step per trial, 
which is typically inefficient in environments where rewards are not immediate \citep{rudder2019}. 
Multi-step bootstrapping methods, which assign credit to past actions by spanning multiple time steps, are typical cheap solutions for efficient long-term credit assignments in these complex environments \citep{singh1996reinforcement,sutton2018reinforcement}. 
{However, multi-step bootstrapping methods encounter difficulties when applied in off-policy scenarios, where the data are obtained by various behavioral policies other than the policy we aim to learn (also known as the target policy).}
Previous methods mostly rely on importance sampling (IS)-based techniques to correct the discrepancy between the behavioral policy and the target policy \citep{precup2000eligibility,harutyunyan2016q,munos2016safe, sutton2018reinforcement, Schulman2016HighDimensional}. 
However, they often suffer from high variance due to the product of IS ratios over multiple time steps \citep{cortes2010learning,metelli2018}.
Recently, several variance reduction methods have been proposed and demonstrated to be effective in practice. These methods include Retrace ($\lambda$) \citep{munos2016safe}, C-trace \citep{rowland2020adaptive}, as well as others \citep{espeholt2018impala,horgan2018distributed,asis2017multi}. 
Some works utilize multi-step off-policy data without any adjustment (such as IS technique) and have shown promising results in practice \citep{horgan2018distributed,barth2018distributed,hessel2018rainbow}.
However, the lack of adjustments on the off-policy data in these methods raises significant concerns. Specifically, it brings into question whether the estimated value function might actually converge to the optimal value function \citep{hernandez2019understanding}. This represents a potential limitation of the effectiveness of these approaches for certain RL applications.

{To address these challenges, we study the convergence of existing IS-free multi-step off-policy methods and provide new tools for IS-free multi-step off-policy learning with guaranteed convergence to the optimal VF.} Our contributions are as follows:
\begin{enumerate}[wide=0em,itemindent=0em,leftmargin=0em]
    \item 
        We study a family of existing IS-free multi-step off-policy methods, which look ahead for $n$ time steps along the trajectory and utilize the multi-step off-policy data without any adjustment mechanism such as importance sampling (herein, $n$ is referred to as the \emph{lookahead depth}).
        We theoretically show that these methods suffer from an underestimation issue concerning the optimal VF when incorporating the off-policy data collected by non-optimal behavioral policies---a common scenario in practice.
        This underestimation issue could potentially lead to a sub-optimal policy. 
        Moreover, this issue is aggravated when increasing $n$, restricting the capacity of these methods to efficiently utilize information from distant future time steps.
        Refer to \Cref{sec_multi_step_BO} for details.
    \item 
We introduce a simple yet non-trivial mechanism, named \emph{highway gate}, which controls the information flow by comparing the information from a lookahead depth of $n$ with that from a lookahead of $1$.
We demonstrate that our new \emph{Highway Operator}, featuring the integration with the \emph{highway gate}, converges
to the optimal VF when learning with arbitrary lookahead depth $n$ and data collected by arbitrary policies.
This characteristic is substantiated through both theoretical analysis and empirical investigations.
Therefore, our operator can ``safely'' and ``rapidly'' transfer credit information from arbitrarily distant future time steps to the past (hence the name ``highway'') without suffering from the underestimation issue. %
Refer to \Cref{sec_highwayOperator} and \Cref{sec_experiment} for details.
We also propose an additional mechanism to increase the learning speed by aggregating the information from various lookahead depths and behavioral policies.
We theoretically show that our new operator, named \emph{Highway Optimality Operator}, provides the optimal choice for determining which lookahead depth and behavioral policy to use in practical cases.
Refer to \Cref{sec_highwayOptOperator} for details.
Furthermore, we propose a \emph{Highway Softmax Operator} to balance the trade-off between rapid learning and alleviating the overestimation issues of the Highway Optimality Operator in the function approximation case.
Similar ideas for comparing information from various lookahead depths have been explored in existing works \citep{wright2013exploiting}. However, we study more general mechanisms for aggregating information from various lookahead depths and policies, including the key components of our proposed highway gate.\footnote{
This work extends a previous shorter report on Greedy Multi-Step Off-Policy
RL \citep{wang2020GreedyMultiStep}. This paper aims to provide a comprehensive analysis, identify the key components, and present a more general form,  accompanied by extensive experiments. 
}
    \item 
Based on the new operators, we propose a family of novel multi-step off-policy algorithms that enable the efficient transfer of reward information across long time lags without underestimation risk.
These algorithms, named Highway Value Iteration, Highway Q-Learning, and Highway DQN, 
effectively utilize information from a very distant future, making them more efficient in delayed reward environments than traditional RL algorithms.
Refer to \Cref{sec_method} for details.
\item 
 We evaluate our algorithms in toy tasks and MinAtar environments where the reward delays range from hundreds to thousands of timesteps, including the most challenging scenarios where the reward is given only at the end of the game.
Our algorithms demonstrate remarkable performance, outperforming many existing multi-step off-policy algorithms that fail when the reward is delayed.
Refer to \Cref{sec_experiment} for  details.
\end{enumerate}

\section{Preliminaries}

A \emph{Markov Decision Processes (MDP)} is described by the tuple $(\mathcal{S},\mathcal{A},\gamma,\cal T, \mathcal{R} )$ \citep{puterman2014markov}, 
where $\mathcal{S}$ is the state space, $\mathcal{A}$ is the action space, and $\gamma \in [0,1)$ is the discount factor.
We assume MDPs with countable $\mathcal{S}$ (discrete topology) and finite $\mathcal{A}$. The transition probability is denoted as ${\cal T}: \sSpace \times \aSpace \rightarrow \Delta(\sSpace) $,
and the reward probability function is denoted as $\mathcal{R}: \sSpace \times \aSpace \rightarrow \Delta(\R) $.
We use the following symbols to denote conditional probabilities: ${\cal T}(s'|s,a)$, $\mathcal{R}(\cdot|s,a)$, where $s,s' \in \sSpace $ and $a \in \aSpace$.
In order to ensure the completeness of the value function space (assumption of Banach fixed point theorem), we assume bounded rewards that, when discounted, produce bounded value functions.
We denote $l_{\infty}( \mathcal{X} )$ as the space of bounded sequences with supremum norm $\|\cdot\|_{\infty}$ (we will use  $\|\cdot\|$ for simplicity), where the support $\mathcal{X}$ is assumed to be a countable set with a discrete topology.
The completeness of our value spaces then follows from the completeness of $l_{\infty}(\mathbb{N})$\footnote{The space of all bounded sequences with supremum norm, known to be a complete metric space.}.

The goal of RL is to find a policy $\pi: \sSpace \rightarrow \Delta(\aSpace)$ that yields the maximal \emph{expected return}, where the {return} is defined as the accumulated discounted reward from time step $t$, given by $G_t = \sum_{n=0}^{\infty} \gamma^n r_{t+n}$, and the expected return is defined as the expectation of return.
The \emph{state-action value function (VF)} of a policy $\pi$ is defined as the expected return when starting in state-action $(s,a)$ and following policy $\pi$, denoted as $Q^\pi (s,a) \triangleq \E \left[ G_t  | s_t=s, a_t=a; \pi \right] $. 
Let $\Pi$ denote the space of all policies.
The \emph{optimal VF} is defined as $Q^*(s,a)= \max_{\pi \in \Pi} Q^\pi(s,a)$, which is called the \emph{optimal expected return} from the state-action $(s,a)$.
It is also convenient to define the \emph{state VF} of a policy $\pi$, 
and the optimal state $V^\pi (s) \triangleq \E \left[ G_t  | s_t=s; \pi \right] $. 
For convenience, we use VF to refer to the state-action VF throughout the paper when there is no ambiguity.

To obtain the VF $Q^*$ or $Q^\pi$, researchers develop various operators.
We usually care about the operators that can \emph{converge} to the VF $Q^*$ or $Q^\pi$, formally defined below \citep{szepesvari2010algorithms}.

\def\operator{\mathcal{K}}
\def\VF{ {\accentset{\mbox{\large\bfseries .}}{{Q}}} }
\begin{definition}
An operator $\operator$ can converge to a specific VF $\VF$ iff 
for all $Q \in  l_{\infty}(\sSpace \times \aSpace)$, it holds that $ \lim_{n \rightarrow \infty} (\operator)^{n} Q = \VF $.
\end{definition}

\begin{remark}
An operator $\operator$ can converge to a specific VF $\VF$ if the following two conditions hold:
(1) The operator $\operator$ is a contraction on the complete metric space $l_{\infty}(\sSpace \times \aSpace)$, i.e., there exists $\xi < 1$ such that for any $Q, Q' \in l_{\infty} (\sSpace \times \aSpace) $ we have $ \| \operator Q - \operator Q' \| \leq \xi \| Q-Q'\| $;
{or there exists $\xi < 1$ such that for any $Q \in l_{\infty} (\sSpace \times \aSpace) $ we have $ \| \operator Q - \VF \| \leq \xi \| Q- \VF \| $;}
(2)
$\VF$ is a unique fixed point of $\operator$, i.e., 
$\operator \VF = \VF$.
\end{remark}

\paragraph{Bellman Optimality/Expectation Operator}
The Bellman Optimality (BO) Operator (eq. \ref{eq_oneStepOperator})
and 
Bellman Expectation (BE) Operator (eq. \ref{eq_oneStepExpectationOperator}) 
, are defined as follows: 
\begin{align}
(\oneStepOperatorMath[] Q )(s,a) \triangleq &
\mathbb{E}_{ r \sim \mathcal{R}(\cdot|s,a), s' \sim \mathcal{T}(\cdot|s,a) }
\left[
  r+\gamma
  \max_{a'\in \aSpace}
 Q\left( s',a'\right) 
\right].
 \label{eq_oneStepOperator}
\\
({\mathcal{B}}^{\pi} Q )(s,a) \triangleq &  \mathbb{E}
_{ r \sim \mathcal{R}(\cdot|s,a), s' \sim \mathcal{T}(\cdot|s,a),a' \sim {\pi}(\cdot | s') } 
\left[
  r+
\gamma Q\left( s',a' \right) 
 \right].
 \label{eq_oneStepExpectationOperator}
\end{align}
BO Operator and BE Operator can converge to the $Q^*$ and $Q^\pi$, respectively.
We have the following BO Equation and BE Equation:
\begin{equation*}
    \oneStepOperatorMath[] Q^* = Q^*,
    \quad
    {\mathcal{B}}^\pi Q^\pi = Q^\pi.
\end{equation*}

\section{The Underestimation Issue of Multi-step Bellman Optimality Operator}\label{sec_multi_step_BO}
In this section, we study a family of widely-used multi-step off-policy methods, which utilize the data collected by various policies and perform multi-step bootstrapping directly without any adjustment mechanism such as the importance sampling technique \citep{hessel2018rainbow,horgan2018distributed,daley2019reconciling}. 
The underlying operator of these methods can be formulated in the following generalized form:

\begin{equation}\label{eq_multistepBOOperator}
\begin{aligned}
 \multistepBOOperator
 Q(s,a)  
\triangleq 
& 
\EPolicyStep
            \left( \left( \mathcal{B} ^{\pi} \right) ^{n-1}\mathcal{B}  Q \right)  
   (s,a)
\\= &
\EPolicyStep
    \mathbb{E} _{
    	\tau _{s,a}^{n}\sim \pi
    }
    \left[ \sum_{\stepIndex=0}^{n-1}{\gamma ^{\stepIndex}}r_{\stepIndex}+\gamma ^{n}\max_{a_{n}' \in \mathcal{A}} 
    Q\left(
    s_{n},a_{n}'\right)
    \right]
\end{aligned}
\end{equation}
We give the meaning of the notations below:
\begin{itemize}[wide=0em,itemindent=0em,leftmargin=0em, itemsep=-0.2em]
    \item 
    $\pi$ denotes the \emph{behavioral policy} that collects the data. 
    $\emptyset \neq \policySet:=\{ \pi_1, \cdots, \pi_M \} \subset \Pi $ is a (finite) \emph{set of behavioral policies}.
    \item
$n \in \sZ^+$ is the \emph{lookahead depth},
and {$\stepSet \neq \emptyset$} is a (finite) \emph{\stepSetText/}.
\item
$\policyDist$ and $\stepDist$ are the \emph{selection distributions} over the sets $\policySet$ and $\stepSet$, respectively.

\item Note that ($\policySet, \stepSet, \policyDist, \stepDist)$ can vary for different $(s,a)$. 
For simplicity, we do not explicitly emphasize this variation in the notation.
\item 
$\traj[n][s, a] \triangleq (s_0, a_0, r_0, s_1,\cdots, s_{\stepIndex}, a_{\stepIndex}, r_{\stepIndex}, \cdots, s_{n}, a_{n} | s_0=s, a_0=a) $ is a {trajectory} data starting from $s, a$.
\item 
For simplicity, we will frequently omit the superscript and the subscript in $\multistepBOOperator$ and instead use the notation $\multistepBOOperator[][]$. This convention is also applied to other notations throughout the paper.
\end{itemize}

We illustrate some instances of \multistepBOOperatorText/ $\multistepBOOperator[\policyDist][\stepDist]$ which uses different settings of $\stepSet$ and $\stepDist$: 
$n$-step DQN uses $\stepSet=\{n\}$ and $\stepDist(n)=1$ \citep{hessel2018rainbow,horgan2018distributed,daley2019reconciling} with $n$ as a critical hyperparamter for the algorithm;
$\lambda$-return DQN uses $\stepSet=\{ 1,2,\cdots, N \}$, and $
\stepDist(n)=
\begin{cases}
(1-\lambda)^n \lambda ^{n-1} &,  n<N
\\
\lambda ^{n-1} &, n=N
\end{cases}
$ 
\citep{daley2019reconciling}.
These multi-step methods utilize more distant future information than the one-step method, leading to more efficient credit assignments, particularly in delayed reward environments.
They perform well in some practical tasks with proper tuning of the lookahead depths. 
However, since these methods do not adjust for the off-policy data, potential convergence issues may arise. We illustrate these issues below.

\paragraph{The underestimation issue.} 
As indicated by \cref{eq_multistepBOOperator}, the Multi-step BO Operator directly employs the $n$-step expected return, $\left( \mathcal{B} ^{\pi} \right) ^{n-1}\mathcal{B}  Q$, which is highly dependent on the behavioral policy $\pi$.
When applied to the optimal VF $Q^*$, we observe that $\left( \mathcal{B} ^{\pi} \right) ^{n-1}\mathcal{B}  Q^* \leq Q^* $ for $\pi\neq \pi^*$ and $n>1$, indicating the occurrence of the underestimation issue.
Formally, the following theorem is presented.
\begin{restatable}[]{theorem}{theoremMultiStepBOOperatorBiased} \label{theoremMultiStepBOOperatorBiased}
\def\QFixedPoint{  \widehat{Q}^* }
\def\nExist{ \grave{n} }
Given a \multistepBOOperatorText/ $\multistepBOOperator$, where there exists at least one $\nExist \in \stepSet$ such that $\nExist > 1$ and $\stepDist(\nExist) > 0$,
we have \\
1) $\multistepBOOperator$ has a unique fixed point $ \QFixedPoint $, i.e., $\multistepBOOperator \QFixedPoint = \QFixedPoint $, and 
$ \QFixedPoint \leq Q^* $;
\\
2) $ \QFixedPoint = Q^*$
if and only if any $\pi \in \policySet$, $\policyDist(\pi) >0$ satisfies 
$\pi(s) = \pi^*(s)$ for any $s\in \{ s_1 \in \sSpace | \exists s_0,a_0, \mathcal{T}( s_1 |s_0,a_0 ) >0 \}$, where $\pi^*$ is an optimal policy.
\end{restatable}

All the proofs are provided in \Cref{ap_sec_theoy}.
As the theorem implies, the Multi-step BO Operator converges to the optimal VF only when nearly all the \policiesText/ are optimal. 
In other words, this operator cannot guarantee convergence to the optimal VF as long as it incorporates off-policy data collected by non-optimal behavioral policies, which is a scenario commonly encountered in practice.
To provide an intuitive understanding, we present an illustrative example below. This example will be used throughout the paper.

\begin{mdpexample}\label{example_MDP}
\Cref{fig_example} illustrates a deterministic MDP problem.
The {\color{blue}blue}, {\color{orange}orange} and {\color{red}red} policies execute the upper right action $\nearrow$ at the initial state $S_A$, and obtain a reward of $-9$, $3$ and $9$ respectively.
For convenience, we use $\gamma=1$ for this problem.
\end{mdpexample}
We evaluate the $n$-step BO Operator until it converges to a fixed point $\widehat Q^*$, where $\stepSet=\{n\}$ and $\policyDist$ is a uniform distribution over all behavioral policies, i.e., $\policyDist(\pi)=1/3$ for $\pi \in \{ {\color{blue}\pi_b} , {\color{orange}\pi_o}, {\color{red}\pi_r}\}$. 
As shown in \Cref{fig_Final_Q_wrt_nstep_BO}, even a $2$-step BO Operator, can lead to underestimation where the fixed point satisfies $\widehat{Q}^*(S_A, \nearrow)=3$, which is smaller than the optimal Q value $Q^*(S_A, \nearrow)=9$.
This is because the {\color{black}red policy}, which obtains a low reward of $-9$ within a lookahead depth of $2$, causes the $n$-step BO operator to underestimate the maximal expected return of state-action $(S_A, \nearrow)$ when the $n$-step BO Operator takes an expectation over all the $2$-step expected returns induced by various behavioral policies.

\paragraph{Impact of larger lookahead depth on underestimation}
We also find that the underestimation issue can worsen with an increase in the \step/.
Formally, we have the following theorem.
\begin{restatable}[]{theorem}{theoremMultiStepBOOperatorDiverge} \label{theoremMultiStepBOOperatorDiverge}
\def\QFixedPoint{  \widehat{Q}^* }
Given the same $\policySet, \policyDist$, and two sets $\stepSet =\{ n \}$,  $\stepSet^\prime =\{ n' \}$.
Let $\QFixedPoint_{n}$ and $\QFixedPoint_{n^\prime}$ denote the fixed point of the corresponding \multistepBOOperatorText/s.
For simplicity, assume that $|\policySet|=1$.
1) If $ n \geq n' $ and {$n \mod n' = 0$}, 
then 
$
    \QFixedPoint_{n}  \leq \QFixedPoint_{n'} \leq Q^*,
$
2) $\lim\limits_{n \rightarrow \infty } \QFixedPoint_{n} = \E_{\pi \sim  \policyDist } \left[ Q^\pi  \right]  $.
\end{restatable}

\def\height{0.24}
\begin{figure*}[!b]
    \centering
    \graphicspath{{figsv2/illustrate_example/}}
    \subfloat[An Illustrative Example]{
        \includegraphics[height=\height\linewidth]{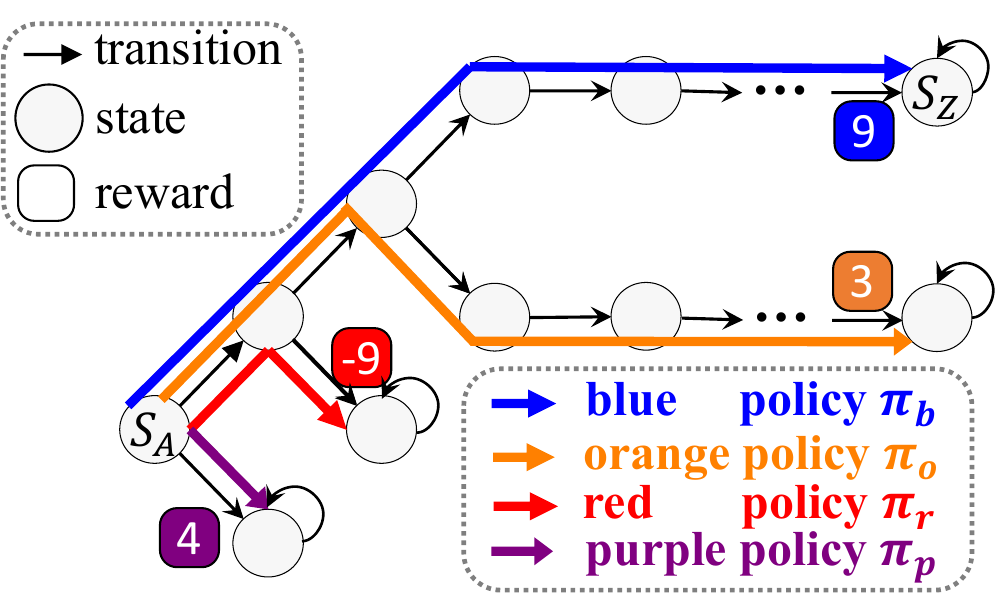}
         \label{fig_example}
    }
    \subfloat[The Fixed Point]{
        \includegraphics[height=\height\linewidth]{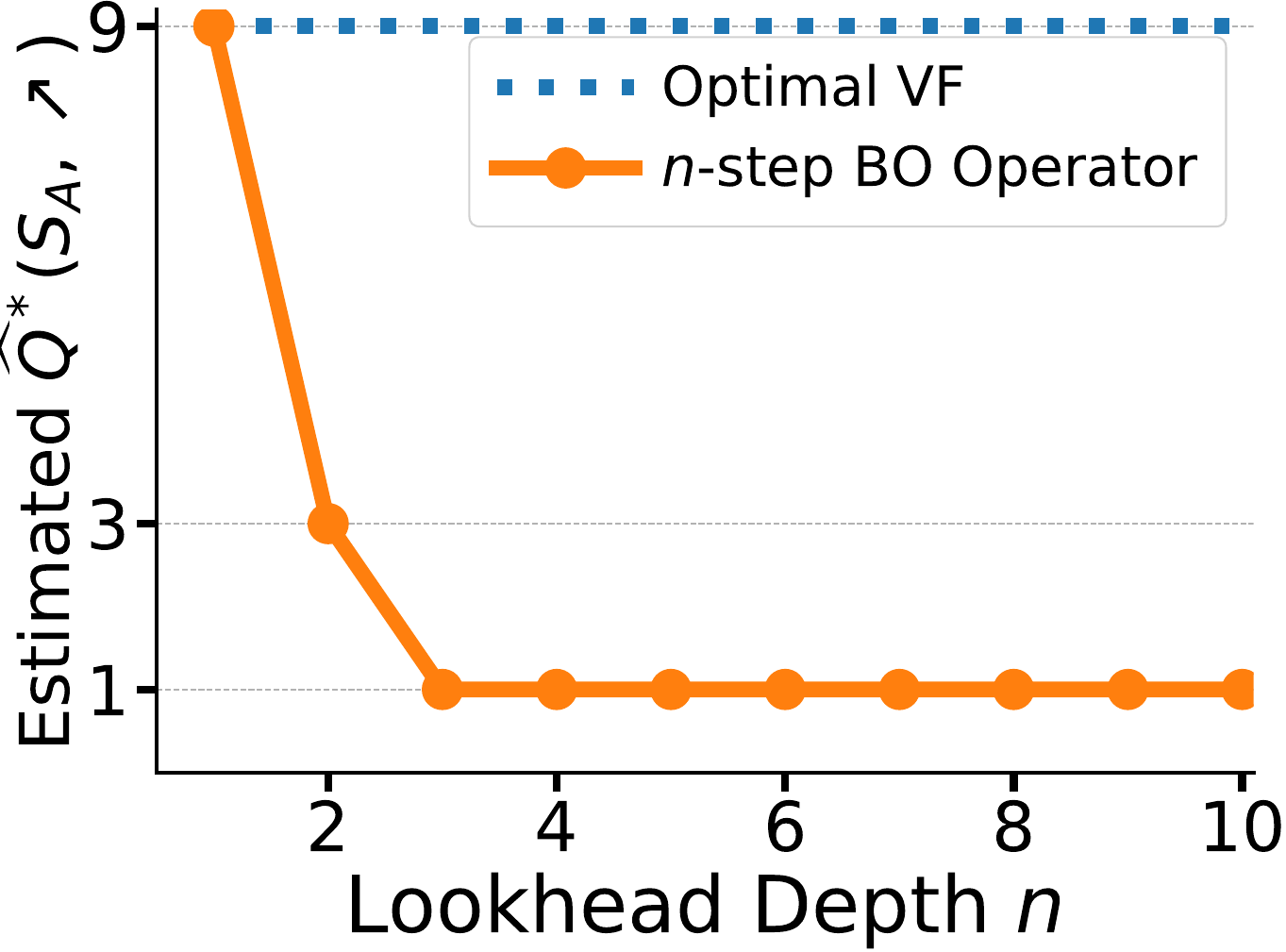}
         \label{fig_Final_Q_wrt_nstep_BO}
    }
   \caption{
   \subref{fig_example} shows a deterministic MDP. 
   The number of transitions between $S_A$ and $S_Z$ is 10.
   There are three behavioral policies, including the {\color{blue}blue}/{\color{orange}orange}/{\color{red}red} policies $ {\color{blue}\pi_b} / {\color{orange}\pi_o}/{\color{red}\pi_r}$, executing the upper right action $\nearrow$ starting from the initial state $S_A$.
   \subref{fig_Final_Q_wrt_nstep_BO} shows the fixed point of $n$-step BO Operator with various lookahead depths $n$.
   }
\end{figure*}

As implied by the theorem, the Q values at the fixed point will decrease as the lookahead depth increases.
In the extreme case where the \step/ $N\rightarrow \infty$, the VF will converge to the expectation of the VFs of the behavioral policies rather than the optimal VF.
Consider the \Cref{example_MDP}: as shown in \Cref{fig_Final_Q_wrt_nstep_BO}, the Q value $\widehat{Q}^*(S_A, \nearrow)$ obtained by the $10$-step BO Operator is smaller than that obtained by the $2$-step BO Operator.
This problem arises mainly due to the {\color{black}orange policy}, which attains a low reward of $3$ within a lookahead depth of $10$, leading to a more severe underestimation for state-action $(S_A, \nearrow)$.

\paragraph{{Underestimation can lead to a sub-optimal policy.}}
The underestimation problem can result in a sub-optimal policy.
The final policy is usually constructed by greedily taking actions according to the final value function $\widehat Q^*$, i.e., $ \argmax_a \widehat Q^*(s,a)$. 
For the \Cref{example_MDP}, for both the $2$-step and $10$-step BO operators, we have $\argmax_a {\widehat{Q}^*}(S_A,a)=\searrow$, because ${\widehat{Q}^*(S_A, \searrow)>\widehat{Q}^*(S_A, \nearrow)}$, while the optimal action on state $S_A$ should be $\nearrow$.

\paragraph{Summary.}
In summary, non-optimal behavioral policies can lead to the underestimation issue of the Multi-step BO Operator, and a larger lookahead depth could aggravate this issue.
Note that here, we aim to assess the ability of the Multi-step BO Operator to utilize multi-step off-policy data.
These conclusions suggest that one needs to deliberately tune the set of behavioral policies and the lookahead depth to balance the trade-off between addressing the underestimation issues and achieving efficient multi-step learning. 
In practice, discarding the oldest data from the replay buffer as the learning process progresses and appropriately tuning the lookahead depth could yield good performance \citep{hessel2018rainbow,horgan2018distributed,daley2019reconciling}.

\section{Highway Reinforcement Learning} 
\label{sec_foundation}

In this section, we first focus on the safe use of off-policy data with arbitrary lookahead depth, aiming to converge to the optimal value function without encountering underestimation issues. Next, we explore how to aggregate information from various lookahead depths to further improve learning efficiency.

\subsection{Safe Learning with Arbitrary Lookahead Depth}\label{sec_highwayOperator}

\begin{figure*}[!t]
    \centering
    \includegraphics[width=0.98\linewidth]{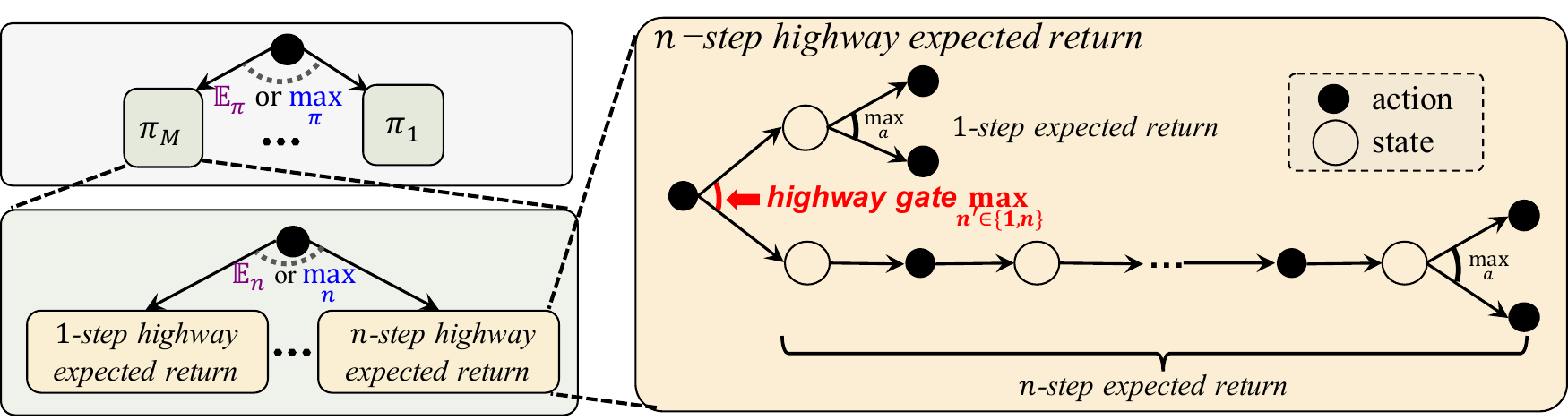}
   \caption{
   {Backup diagram} of \highwayGenOperatorText/ $
   \highwayOperator$ (eq. \ref{eq_highwayOperator})
and \highwayOptOperatorText/ $\highwayOptOperator$ (eq. \ref{eq_highwayOptOperator}.)
   }
   \label{fig_backupdiagram}
\end{figure*}

The previous section analyzes the convergence properties of a family of IS-free multi-step BO operators, demonstrating that they suffer from an underestimation issue due to the lack of adjustments on off-policy data, which are typically collected by non-optimal behavioral policies. Furthermore, it is shown that the risk increases as the step size increases, limiting their capacity to effectively leverage information from distant future time steps.
To address this issue, we introduce a mechanism to safely control the utilization of multi-step information. Formally, we propose the following \emph{\highwayOperatorText/} (also referred to as \emph{\highwayGenOperatorText/} to distinguish it from other newly proposed operators), defined by:
\begin{equation}\label{eq_highwayOperator}
\begin{aligned}
 & 
 \highwayOperator
 Q
 (s,a)  
\triangleq 
\EPolicyStep
\overset{ \emph{\footnotesize \text{{$n$-step highway expected return}}} }{\overbrace{
    {
        \underset{ \emph{ \footnotesize  \color{maxStepOneN} \text{highway gate}}  }{{
            \color{maxStepOneN}
            \underbrace{
                \max_{n' \in \{1, n\} }
            }
        }}
    }
    \underset{   \emph{ \footnotesize \text{$n'$-step expected return} } }{\underbrace{
            \left( \left( \mathcal{B} ^{\pi} \right) ^{n'-1}\mathcal{B}  Q \right) 
            (s,a) 
    }}
}}
\end{aligned}
\end{equation}
Compared to the existing \multistepBOOperatorText/ (see eq. \ref{eq_multistepBOOperator}), the only difference is the use of a \emph{highway gate}, which compares the $n$-step expected return with the $1$-step expected return and selects the maximum value.
We refer to the maximum among these two returns as the \emph{$n$-step highway expected return} (as shown in eq. 
\ref{eq_highwayOperator}).
\Cref{fig_backupdiagram} illustrates the backup diagram of this operator. 
Below we detail the properties of this mechanism.

\paragraph{Significance of the highway gate.}
In the following, we demonstrate that the highway gate is able to mitigate the underestimation issue and ensure convergence to the optimal VF. 
As implied by \cref{eq_highwayOperator}, the highway gate selects between the $n$-step and $1$-step expected return.
Recall that the $n$-step expected return, $\left( \mathcal{B} ^{\pi} \right) ^{n-1}\mathcal{B}  Q$, incorporates more distant future information but is highly dependent on the behavioral policy $\pi$, posing the risk of underestimation.
The $1$-step expected return, $\mathcal{B}  Q$, is less efficient but more reliable than the $n$-step expected return because it is less dependent on the behavioral policy.
The highway gate combines the best of these two returns depending on the behavioral policy $\pi$ and the VF $Q$, as detailed below.
Under the case $Q=Q^*$, 
the highway gate will always ignore the $n$-step expected return induced by the non-optimal behavioral policies and take the $1$-step expected return instead.
This is because $\left( \mathcal{B} ^{\pi} \right) ^{n-1}\mathcal{B}  Q^* \leq \mathcal{B}  Q^* $ holds for any $\pi\neq \pi^*$. 
Then we have
$
\max\limits_{n' \in \{1, n\} }
\left( \left( \mathcal{B} ^{\pi} \right) ^{n'-1}\mathcal{B}  Q^* \right) 
= \mathcal{B} Q^* =  Q^*
$, indicating that the underestimation issue vanishes.
Formally, we have the following theorem.

\begin{restatable}[]{theorem}{theoremHighwayEquation} \label{theoremHighwayEquation}
    For any $  ( \policySet, \stepSet, \policyDist[], \stepDist[] )$ , we have \\
    \noindent{}1)~
         $\highwayOperator$ is a contraction on the complete metric space $l_{\infty}( \sSpace \times \aSpace )$, i.e.
        for any $Q, Q' \in  l_{\infty}( \sSpace \times \aSpace )$, we have $
    \|  \highwayOperator Q - \highwayOperator Q' \|
    \leq \| \oneStepOperatorMath Q - \oneStepOperatorMath Q' \|
    	 \leq \gamma 	\|   Q -  Q' \|;
    $
         
    \noindent{}2)~ (\highwayEquationFull/) $Q^*$ is the only fixed point of $\highwayOperator$, which means that for all $Q \in l_{\infty}(\sSpace \times \aSpace)$ it holds that $Q = \highwayOperator Q$ if and only if $Q=Q^*$. Formally, we have
    $
    Q^* = 
\EPolicyStep[black]
    \max_{n' \in \{1, n\} }
    \left( \left( \mathcal{B} ^{\pi} \right) ^{n'-1}\mathcal{B}  Q^* \right).
    $
\end{restatable}

As the theorem implies, our \highwayOperatorText/ converges to the optimal VF with arbitrary lookahead depths and off-policy data collected by arbitrary behavioral policies. 
We illustrate this convergence property intuitively by the aforementioned \Cref{example_MDP}. 
Similar to the $n$-step BO Operator, we consider $n$-step Highway Operator with $\stepSet=\{n\}$ and a uniform distribution of $\policyDist$.
As shown in \Cref{fig_Final_Q_wrt_nstep_BO}, the $n$-step Highway Operators can always converge to the correct optimal VF with any lookahead depth $n$, with the presence of non-optimal {\color{black}orange} and {\color{black}red} behavioral policies.
Besides, as shown in \Cref{fig_chosen_nstep}, as the updated VFs of $10$-step \highwayOperatorText/ converge after iteration $3$, the $1$-step expected return is always chosen for the non-optimal {\color{black}orange} and {\color{black}red} policies.

Moreover, under the case $Q\neq Q^*$, the $n$-step expected return will be selected as long as it provides a larger value than the $1$-step expected return, i.e., $\left( \mathcal{B} ^{\pi} \right) ^{n-1}\mathcal{B}  Q \geq \mathcal{B}  Q $. Therefore, the selection depends on both the behavioral policy $\pi$ and the VF $Q$.
For example, as shown in \Cref{fig_chosen_nstep}, for the {\color{black}orange policy}, which ranks the medium in performance among the three policies, the $n$-step expected return is chosen at the beginning of the learning process, and the $1$-step expected return is chosen at the end.
For the {\color{black}red policy}, which performs the worst, the $1$-step expected return is always chosen.
In contrast, the {\color{black}blue policy}, which performs the best, the $n$-step expected return is always chosen.

Based on the above discussion, it is noteworthy that adaptively adjusting the lookahead depth in a proper manner can be an alternative to the importance sampling technique. This approach guarantees convergence to the optimal VF under multi-step off-policy learning.
Additionally, the highway gate is the essential mechanism contributing to the convergence property. 
The absence of the highway gate of \multistepBOOperatorText/ leads to the underestimation issue (as formally stated in \Cref{theoremMultiStepBOOperatorBiased}).

\paragraph{Critical components within the highway gate.}
Below, we highlight several critical components within the highway gate that contribute to the convergence property.

    {First}, the maximization operation in the highway gate, i.e., ${\color{maxStepOneN}\max_{n' \in \{1, n\} }}( \cdot )$, plays an important role in ensuring the convergence property. 
    If we replace the maximization with expectation, i.e., ${\color{maxStepOneN}\E_{n' \sim \dist[][\{1, n\}] }}(\cdot)$, which takes the {expectation} over both the $1$-step and $n$-step expected returns according to a specified distribution $\dist[][\{1, n\}]$, then the convergence property cannot be ensured. This is because $ {\color{expectationPolicyStep} \E_{n \sim \dist[][\stepSet]} } {\color{maxStepOneN}\E_{n' \sim \dist[][\{1, n\}] }}(\cdot)$ reduces to ${\color{expectationPolicyStep}\E_{n \sim \dist[][\{1\} \cup \stepSet]}} $, which is exactly a special case of \multistepBOOperatorTextMath/.

    Second, the $1$-step expected return, which can be regarded as a threshold within the highway gate ${\color{maxStepOneN}\max_{n' \in \{ 1, n\} }}( \cdot )$, is also critical. 
    If we replace the $1$-step expected return with the $n_1$-step expected return, i.e., ${\color{maxStepOneN}\max_{n' \in \{n_1, n\} }}$ where $n_1>1$, then the convergence property cannot be guaranteed.
    \citet{she2017learning} proposed using $n_1=0$\footnote{
    For notation simplicity, we define $\max\limits_{n' \in \{n_1, n\} }  \left( \left( \mathcal{B} ^{\pi} \right) ^{n'-1}\mathcal{B}  Q \right) \triangleq \max \left\{ Q, \left( \mathcal{B} ^{\pi} \right) ^{n-1}\mathcal{B}  Q  \right\} $ when $n_1=0$.}, which refer to using the $Q$ VF itself as the threshold. However, this variant cannot guarantee convergence either.
    We provide a formal statement in the theorems below.

\begin{restatable}[]{remark}{theoremNoonestepBiased} \label{theoremNoonestepBiased}
\def\QFixedPoint{  \widehat{Q}^* }
\def\nExist{ \grave{n} }
For the operator $\wronghighwayOperator Q \triangleq 
\EPolicyStep
{\color{maxStepOneN}
    \max_{n' \in \{n_1, n\} } 
}
\left( \left( \mathcal{B} ^{\pi} \right) ^{n'-1}\mathcal{B}  Q \right) 
$ (where $n_1 > 1$), which replace ${\color{maxStepOneN} \max_{n' \in \{1,n\}}}$ with ${\color{maxStepOneN}\max_{n' \in \{n_1, n\} }}$ in \cref{eq_highwayOperator}.
Assume that there exists at least one $\nExist \in \stepSet$ such that $\nExist > 1$ and $\stepDist( \nExist ) > 0$.
1) $\wronghighwayOperator$ has a fixed point $ \QFixedPoint $, i.e., $\wronghighwayOperator \QFixedPoint = \QFixedPoint $, and we have
$ \QFixedPoint \leq Q^* $.
2) $\QFixedPoint = Q^*$
if and only if any $\pi \in \policySet$, $\policyDist(\pi) >0$ satisfies 
$\pi(s) = \pi^*(s)$ for any $s\in \{ s_1 \in \sSpace | \exists s_0,a_0, \mathcal{T}( s_1 |s_0,a_0 ) >0 \}$ and $\pi^*$ an optimal policy.
\end{restatable}

\begin{restatable}[]{remark}{theoremNoonestepBiasedZero} \label{theoremNoonestepBiasedZero}
\def\QFixedPoint{  \widehat{Q}^* }
For the operator $\wronghighwayOperator  $ with $n_1=0$, for any $\QFixedPoint \geq Q^*$, we have $\wronghighwayOperator \QFixedPoint = \QFixedPoint$.
\end{restatable}

The above theorems highlight the importance of employing the $1$-step expected return as a threshold in the highway gate.
Note that this conclusion does not imply that our operator collapses to always using the $1$-step expected return (which is what exactly the classical (One-Step) BO Operator does).
As illustrated in \Cref{fig_chosen_nstep}, our Highway Operator alternates between $n$-step and $1$-step expected returns in practice.

\begin{figure*}[!b]
    \centering
    \graphicspath{{figsv2/illustrate_example/}}
    
    \def\height{0.27}
    \subfloat[The Fixed Point]{
        \includegraphics[height=\height\linewidth]{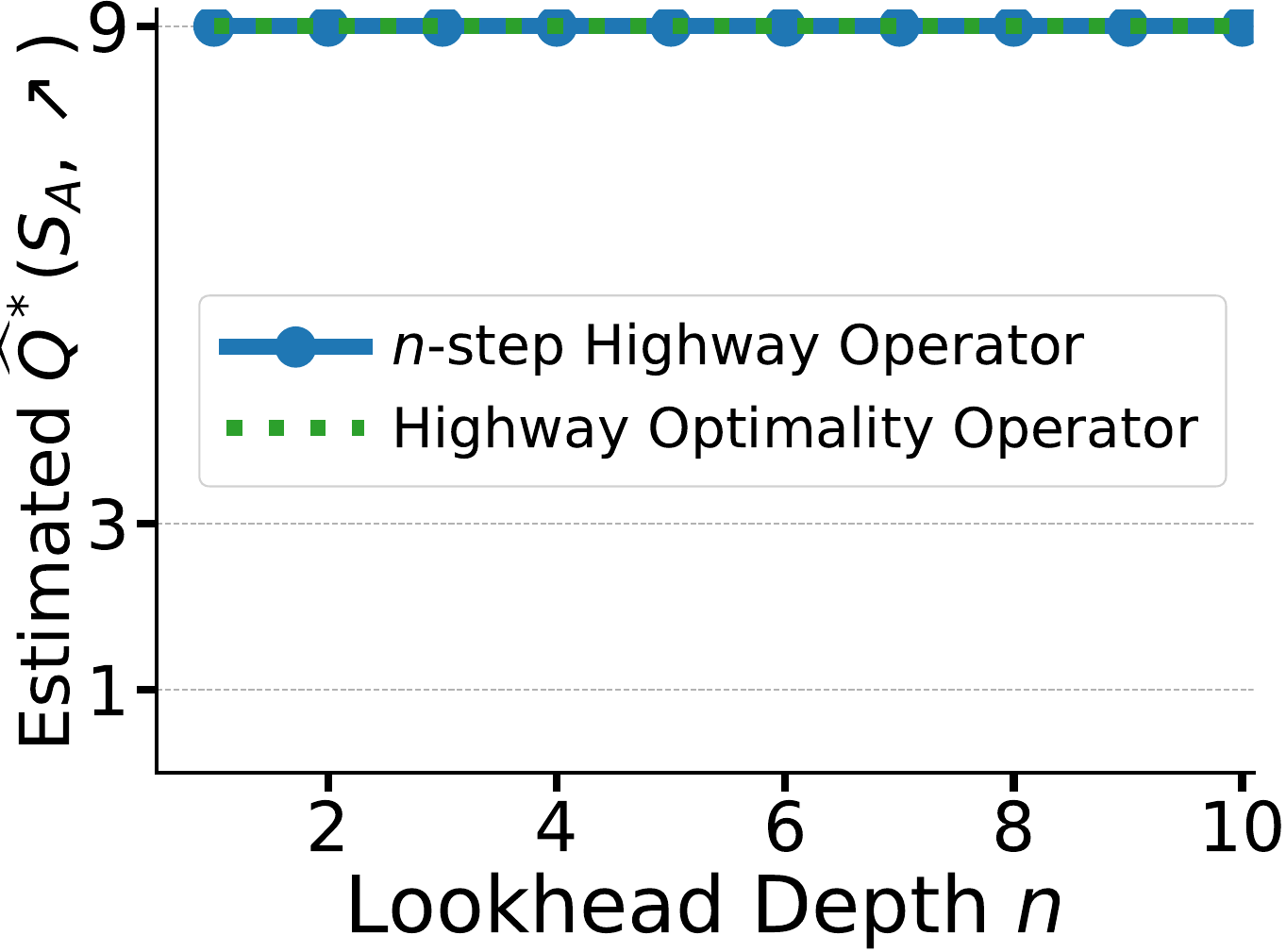}
         \label{fig_Final_Q_wrt_nstep_Highway}
    }
    \subfloat[Iterations of Convergence]{
        \includegraphics[height=\height\linewidth]{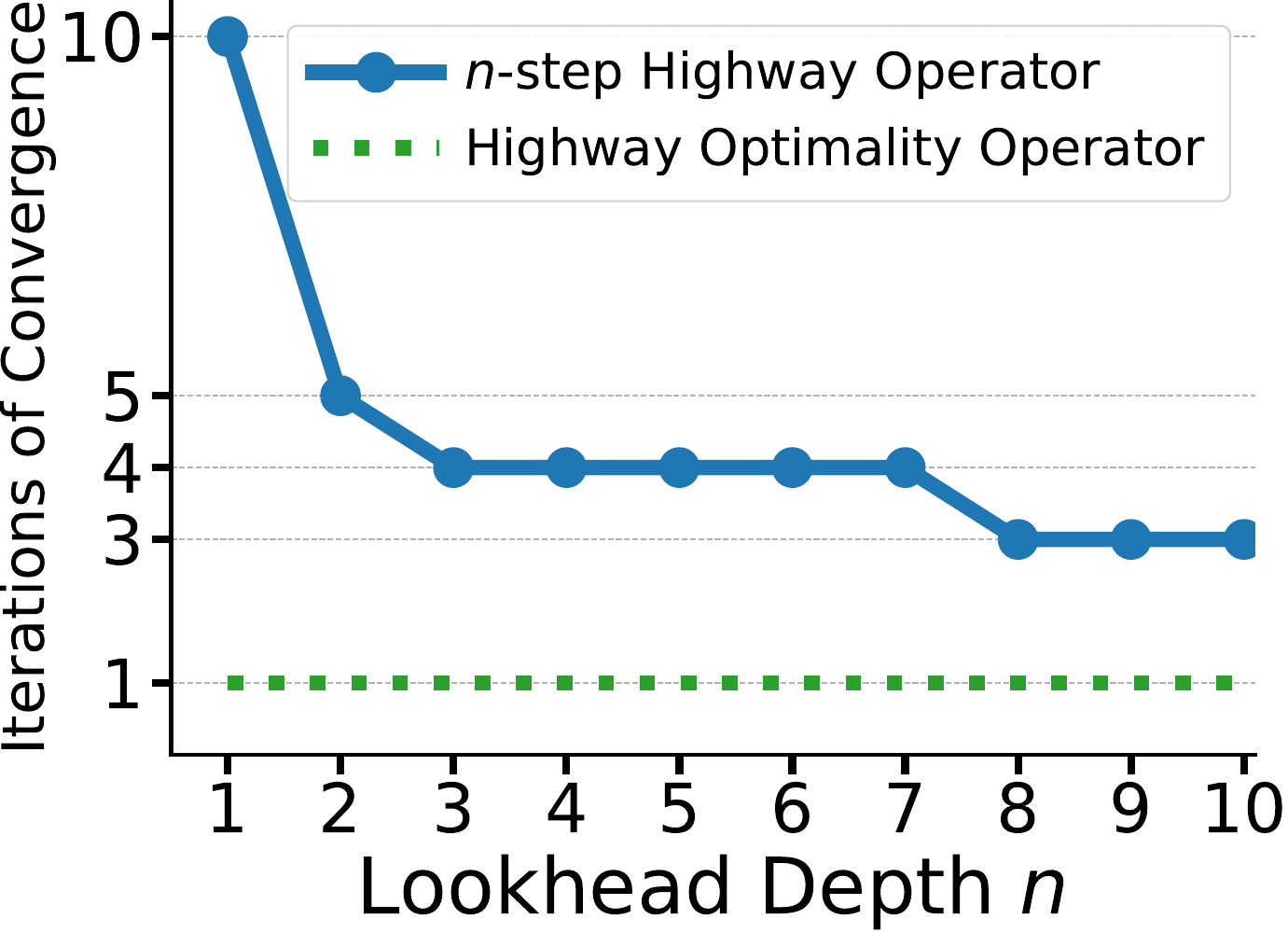}
         \label{fig_iterations_wrt_nstep_Highway}
    }
    \subfloat[{The chosen $n'$-step during the update}]{
        \includegraphics[height=\height\linewidth]{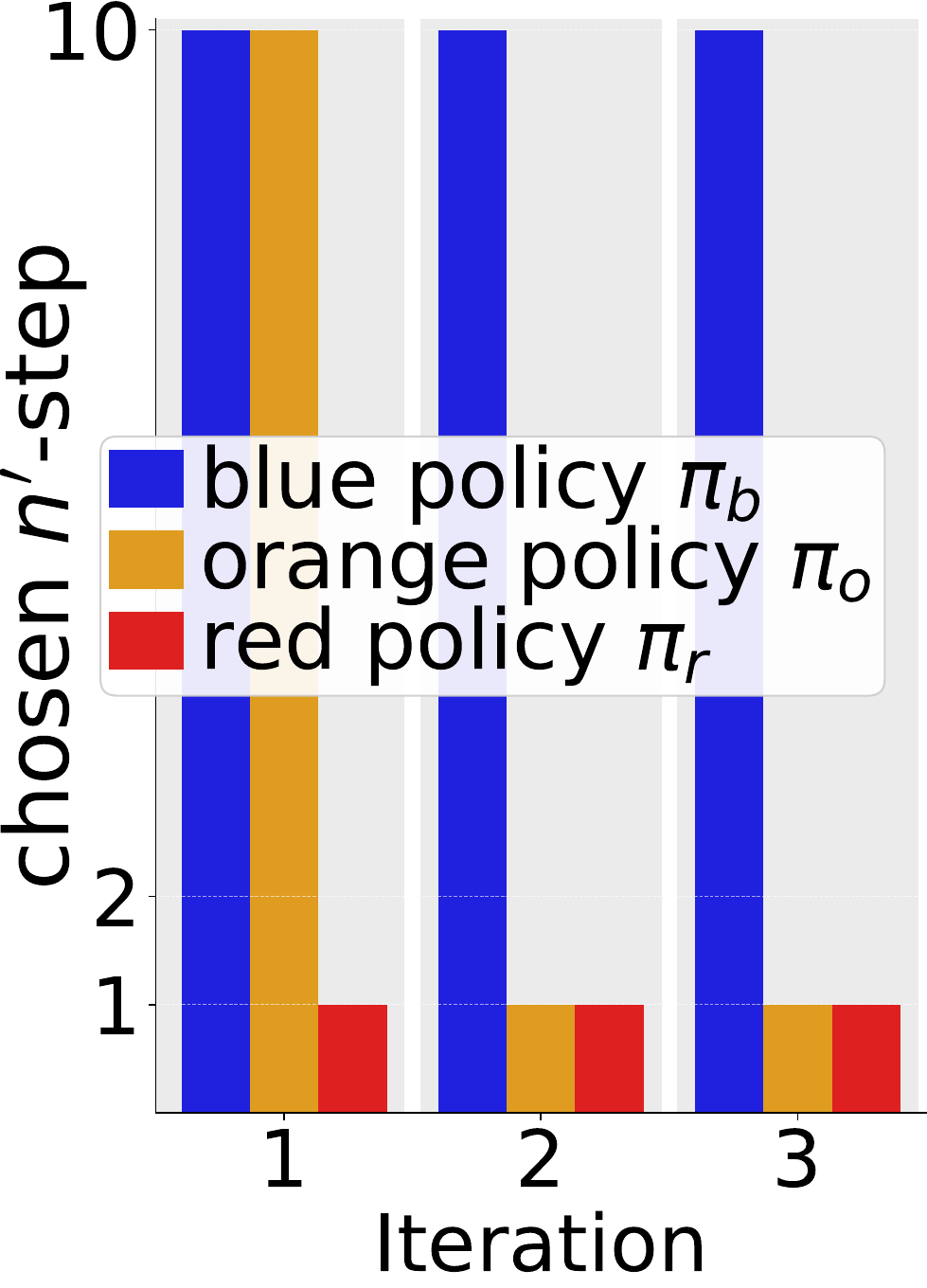}
         \label{fig_chosen_nstep}
    }
   \caption{
   \subref{fig_Final_Q_wrt_nstep_Highway} shows the fixed point of $n$-step \highwayGenOperatorText/ $\highwayOperator[][]$ with various lookahead depths $n$ and \highwayOptOperatorTextMath/ in the \Cref{example_MDP}, respectively.
   \subref{fig_iterations_wrt_nstep_Highway} shows the required iterations of $n$-step \highwayGenOperatorText/ $\highwayOperator[][]$  with various lookahead depths $n$ and \highwayOptOperatorTextMath/, respectively. The initial value function $Q(s,a)=0$ for all $(s,a)$.
   \subref{fig_chosen_nstep} shows the actual chosen $n'$-step for $10$-step \highwayOperatorText/, i.e., $\argmax_{n'\in\{1,n\}}\left((\mathcal{B}^\pi)^{n'-1}\mathcal{B}Q\right)$ for $\pi \in \{ {\color{blue}\pi_b} , {\color{orange}\pi_o}, {\color{red}\pi_r} \}$.
   }
\end{figure*}

\subsection{Accelerated Learning by Aggregating Information from Various Lookahead Depths and Behavioral Policies}\label{sec_highwayOptOperator}
In this section, we study how to further facilitate rapid learning by aggregating information from various behavioral policies and lookahead depths.
Specifically, we tackle this problem by determining the selection distribution over the set of behavioral policies and lookahead depths, i.e., $\policyDist$ and $\stepDist$, respectively.
To achieve this, we introduce a new \emph{\highwayOptOperatorText/}, defined by:

\begin{equation}\label{eq_highwayOptOperator}
   \highwayOptOperator Q  \triangleq 
   \MaxPolicyMaxStep
    \left( \left( \mathcal{B} ^{\pi} \right) ^{n'-1}\mathcal{B}  Q \right).
\end{equation}
As implied by the equation, this operator maximizes the $n$-step highway expected returns induced by various behavioral policies $\pi$ and lookahead depths $n$.
This operator can be viewed as a special case of the \highwayGenOperatorTextMath/ (eq. \ref{eq_highwayOperator}), as it assigns a {probability of one} to the behavioral policy $\pi$ and \stepText/ $n$ that yield the highest value. 
Interestingly, it can be shown that this operator provides the optimal choice for the selection distributions $\policyDist$ and $\stepDist$ when $Q\leq Q^*$.
To verify this, we first define a measure of the superiority of a given operator $\mathcal{K}$. We propose using the distance between the updated VF $(\mathcal{K} Q)$ and the optimal VF $Q^*$, as the goal of value-based RL is to minimize this distance. 
Formally, the distance is defined by:

\begin{equation*}
\begin{aligned}
\distance[ \mathcal{K} ] \triangleq
\left| (\mathcal{K} Q) (s,a) -Q^*(s,a) \right|,
\quad\quad
\distanceAll[ \mathcal{K} ] \triangleq \max_{s\in \sSpace,a \in \aSpace }| d_{\mathcal{K}}(s,a) |.
\end{aligned}
\end{equation*}
Note that a smaller distance indicates a faster convergence speed. 
We have the following theorem.

\begin{restatable}[]{theorem}{theoremCompareHighwayOperatorHighwayOptOperator} \label{theoremCompareHighwayOperatorHighwayOptOperator}
(Comparison between \highwayOptOperatorText/ $\highwayOptOperator$ and \highwayGenOperatorText/ $\highwayOperator$)
Given the same $\policySet$ and $\stepSet$,
if $Q\leq Q^*$,
then we have: 
\begin{equation*}
    d_{\highwayOptOperator}(s,a) = \min_{ \policyDist, \stepDist  } d_{\highwayOperator}(s,a), \text{ for any } s \in \sSpace, a \in \aSpace.
\end{equation*}
\end{restatable}

This theorem suggests that the maximization operation in \highwayOptOperatorText/ provides the optimal choice for aggregating bootstrapped values induced by various \policiesText/ and \steps/.
Consider the \Cref{example_MDP} again: we compare the Highway Generalized and Optimality Operators $\highwayOperator$/$\highwayOptOperator$, using the same $\policySet=\{{\color{blue}\pi_b} , {\color{orange}\pi_o}, {\color{red}\pi_r} \}$ and $\stepSet=\{1, 2, \cdots, 10 \}$. The $\highwayOperator$ uses a uniform distribution for both $\policyDist$ and $\stepDist$.
As shown in \Cref{fig_iterations_wrt_nstep_Highway}, \highwayOptOperatorText/ requires one iteration to converge, which is fewer than the three iterations needed for the \highwayGenOperatorText/.

Note that the optimality of \highwayOptOperatorText/ mainly applies to the tabular case.
In the function approximation scenario, this maximization operation can potentially aggravate the overestimation issue due to function approximation errors \citep{van2016deep}. 
Therefore, we propose using the softmax operation instead of the maximization operation, which has been widely employed in recent RL literature \citep{fox2015taming, haarnoja2017reinforcement, schulman2017equivalence, song2019revisiting}. 
The new variant, named \emph{\highwaySoftmaxOperatorText/}, is defined as follows:
\begin{equation}\label{eq_softmaxHighwayQOperator}
\highwaySoftmaxOperatorMath Q 
\triangleq
\SoftmaxPolicySoftmaxStep
\left(
(\oneStepOperator[\pi] )^{n'-1} \oneStepOperator Q 
\right)
\end{equation}
\begin{equation*}
\text{ where }
    \underset{x\in \mathcal{X}}{\softmax } 
f\left( x \right) 
\triangleq 
\sum_{x\in \mathcal{X}}
\underset{\text{\emph{softmax weight}}}{
\underbrace{
\frac
{{\exp \left( \alpha f\left( x \right) \right)}  }
{\sum_{x'\in \mathcal{X}}{\exp \left( \alpha f\left( x' \right) \right)}}
}
}
f\left( x \right). 
\end{equation*}
This operator can be regarded as a generalization of the \highwayOptOperatorTextMath/ (eq. \ref{eq_highwayOptOperator}), where $\softmax$ reduces to $\max$ when $\alpha\rightarrow \infty$.
It can also be regarded as a special case of the \highwayGenOperatorTextMath/, where the probability $\policyDist(\pi)$ and $\stepDist(n)$ are equal to the corresponding softmax weights\footnote{Note that the softmax temperature $\alpha$ can be different for $\softmax_{\pi \in \policySet}$ and $\softmax_{n \in \stepSet}$. However, we use the same symbol for convenience throughout the paper.}.
In \Cref{sec_method}, we will derive algorithms based on this operator where the Q VF is approximated using a deep neural network.

In summary, the highway gate, which performs maximization over the $1$-step and the $n$-step expected returns, ${\color{maxStepOneN} \max_{ n' \in \{1,n\} } }(\cdot)$, ensures the convergence to the optimal VF; 
the maximization operation over various behavioral policies and lookahead depths, ${\color{maxPolicyStep} \max\limits_{\pi \in \policySet} \max\limits_{ n \in \stepSet } }(\cdot)$, can increase learning efficiency; 
the softmax operation, ${\color{softmaxPolicyStep}{ \underset{\pi\in \policySet } {\mathop {smax}\nolimits^{\alpha} } \underset{n\in \mathcal{N}} {\mathop {smax}\nolimits^{\alpha} } }}(\cdot)$, offers a mechanism to balance a trade-off between efficiency and overestimation.

Note that both the \highwayOptOperatorTextMath/ and the \highwaySoftmaxOperatorTextMath/ can be regarded as a special case of the \highwayGenOperatorTextMath/. Therefore, they also converge to the optimal VF (please refer to \Cref{theorem_sofmaxHighwayOperator} for a formal statement).
\Cref{table_summary} provides a summary of the proposed \highwayOperatorText/s, 
their variants with certain limitations, classical Bellman operators, and other advanced operators.

\subsection{Comparison of Convergence Speed with One/Multi-step BO Operators}

We now compare the convergence speed of the \highwayOperatorText/s to that of the One/Multi-Step BO Operators.
We first analyze the \highwayGenOperatorText/ (\highwayOperatorText/ in short) and then discuss the \highwayOptOperatorText/.
Formally, we present the following theorems.

\begin{restatable}[]{theorem}{theoremCompareOnestepBOOperator} \label{theoremCompareOnestepBOOperator}
(Comparison between the \highwayGenOperatorText/ $\highwayOperator$ and the One/Multi-Step BO Operator\footnote{Recall that the Multi-step BO Operator generally cannot converge to the optimal VF $Q^*$ in the off-policy setting. Nevertheless, we analyze its convergence speed toward the optimal VF $Q^*$.} $\BOOperator / \multistepBOOperator$)
For all $Q \in l_{\infty}(\sSpace \times \aSpace )$, the following holds:1) Assume $Q\le Q^*$.
For any $s\in \sSpace, a \in \aSpace$, we have 
$\distance[\highwayOperator[][]] \leq \min \left\{ \distance[\BOOperator], \distance[\multistepBOOperator[][]] \right\}$;
2) 
$
\distanceAll[\highwayOperator[][]] \leq \distanceAll[\BOOperator]
$.
\end{restatable}

\begin{restatable}[]{remark}{theoremHighwayStrictlyBetterOnestepMultistepBO} \label{theoremHighwayStrictlyBetterOnestepMultistepBO}
Given a specific $(s,a)$.
Assume $Q\leq Q^*$.
Let us define $\supp_{ } f \triangleq \{ x \in \mathcal{X}|  f(x) \neq 0\}$ for a function $f$. 
Under the following two conditions: (1) 
there exists $\pi' \in \supp \policyDist$ and $n' \in \supp \stepDist$ such that ${\argmax _{n \in  \{1, n' \}  }}$ $ \left( \left( \mathcal{B} ^{\pi'} \right) ^{n-1}\mathcal{B}  Q \right) \left( s, a \right) > 1$;
(2) there exists $\pi' \in \supp \policyDist$ and $n' \in \supp \stepDist$ such that 
$\argmax _{n \in \left\{1, n' \right\} } \left( \left( \mathcal{B} ^{\pi'} \right) ^{n-1}\mathcal{B}  Q \right)\left( s, a \right) = 1$.
If condition (1) is satisfied, then $\distance[\highwayOperator[][]] <  \distance[\BOOperator]$.
If condition (2) is satisfied, then $\distance[\highwayOperator[][]] <  \distance[\multistepBOOperator[][]]$.
If both conditions (1) and (2) are satisfied, we have
$\distance[\highwayOperator[][]] < \min \left\{ \distance[\BOOperator], \distance[\multistepBOOperator[][]] \right\}.$
\end{restatable}

As \Cref{theoremCompareOnestepBOOperator} 1) implies, the \highwayOperatorText/ is generally more efficient than both One/Multi-Step BO Operators when $Q\leq Q^*$.
\Cref{theoremHighwayStrictlyBetterOnestepMultistepBO} specifies the situations in which the Highway Operator is strictly more efficient than the One/Multi-Step BO Operators.
For the case $ Q \nleqslant Q^* $, the \highwayOperatorText/ remains comparable to the One-step BO Operator in the worst case, as \Cref{theoremCompareOnestepBOOperator} 2) implies.
Note that the condition $Q \leq Q^*$ is not strong.
Provided that the initial VF $Q_0 \leq Q^*$, then the resulting VFs $(\highwayOperator[][])^k Q_0  $, obtained by applying $\highwayOperator[][]$ multiple times, will not exceed the optimal VF $Q^*$. Formally, we present the following remark.

\begin{restatable}[]{remark}{theoremAlwaysLeqQopt} \label{theoremAlwaysLeqQopt}
If the {initial} VF $Q_0 \leq Q^*$, then $(\highwayOperator[][])^k Q_0 \leq Q^* $ for any $k>0$.
The same applies to the operators $ \multistepBOOperator[][]$, $\multistepBEOperator[][]$.
\end{restatable}

Note that as the \highwayOptOperatorText/ is superior to the \highwayOperatorText/ (as shown in  \Cref{theoremCompareHighwayOperatorHighwayOptOperator}), the above theorems also hold for the \highwayOptOperatorText/ $\highwayOptOperator[][]$.

\begin{table*}[!t]

        
        \def\YES{{\color{blue} YES}}
        \def\NO{{\color{red} NO}}
        \NewDocumentCommand{\vfTable}{O{}}{ Q^{#1}  }
		\setlength{\tabcolsep}{0pt} 
		\renewcommand\theadfont{\tiny}
		\renewcommand\theadgape{\Gape[0in]}
		\renewcommand\cellgape{\Gape[0in]}
		\tiny
		\centering
			\begin{tabular}{@{}
				+m{1.98in}<{\centering}
                    Ym{2.29in}<{\centering}
				Ym{0.25in}<{\centering}
				Ym{1.1in}<{\centering}
				Ym{0.3in}<{\centering}
				@{}}
				\toprule
                    Operator
                    & Formulation
                    & Fixed Point 
                    & \thead{Guaranteed \\ convergence} 
                    & \thead{NOT \\ IS-Based?}
                     \\
				\midrule
				\multirow{2}{*}{}
                        Bellman Optimality (BO) Operator (\crefnop{eq_oneStepOperator})
                        &$\oneStepOperatorMath$ (see \crefnop{eq_oneStepOperator})
    				& $\color{blue}{ \vfTable[*]}$ 
    				& {\color{blue}{/}} 
        				& \YES
				\\
                        Bellman Expectation (BE) Operator (\crefnop{eq_oneStepExpectationOperator})
                        &$\oneStepOperatorMath[\pi']$(see \crefnop{eq_oneStepExpectationOperator})
    				& $\color{QPiFixedPoint}{ \vfTable[\pi^\prime] }$ 
    				& {\color{blue}{/}} 
        				& /
				\\
				\cdashline{1-5}[4pt/2pt]
                    \noalign{\smallskip}
                    \thead{
    				     \multistepBEOperatorText/ 
    				} 
                        &
                        (see \Cref{sec_advanced_multistep_offpolicy_operator} ) \citep{sutton2018reinforcement}
    				& $\color{QPiFixedPoint}{ \vfTable[\pi']}$  
                        & \color{blue}{\thead{For any $\policySet,\policyDist$} } 
                        & \NO
    				\\
    				\noalign{\smallskip}
				\thead{
				Q($\lambda$) Operator 
				}
                        & (see \Cref{sec_advanced_multistep_offpolicy_operator} ) \citep{harutyunyan2016q}
    				& $\color{QPiFixedPoint}{ \vfTable[{\pi'}] }$ & 
        
                        {
                        \color{red}{
                        \thead{
                        Requring for $\pi \in \supp \policyDist$, \\ $\pi$ is close to $\pi'$
    				}
                        }
    				}  
                        & \YES
				\\
				\noalign{\smallskip}
				    \thead{
				        Retrace($\lambda$) Operator 
				    }
                        & (see \Cref{sec_advanced_multistep_offpolicy_operator} ) \citep{munos2016safe}
    				& $\color{QPiFixedPoint}{ \vfTable[{\pi'}] }$  
                        & \color{blue}{For any $\pi$} 
    				& \NO
				\\
				\cdashline{1-5}[4pt/2pt]
				\noalign{\smallskip}
                        \thead{Multi-Step BO Operator (\crefnop{eq_multistepBOOperator})  
    				}
                        &$\multistepBOOperator Q \triangleq \EPolicyStep
            \left( \left( \mathcal{B} ^{\pi} \right) ^{n-1}\mathcal{B}  Q \right)   $
    				& $\color{red}{ \leq \vfTable[*]}$ 
                        & 
				\color{blue}{
				\thead{For any $\policySet, \stepSet, \policyDist, \stepDist$} 
				}				
                        & \YES
                    \\
                    \noalign{\smallskip}
                        \thead{Variants of \highwayGenOperatorText/ 
                        \\
                        $
                        {\color{black}{(n_1 \neq 1)}}
                        $
                        (\Cref{theoremNoonestepBiased} and \ref{theoremNoonestepBiasedZero})
    				}
                        &$ \scaleMathInside{0.9}{ \wronghighwayOperator Q \triangleq \EPolicyStep
                                {
                                \color{maxStepOneN}
                                \max\limits_{n' \in \{n_1, n\} }
                                }
            \left( \left( \mathcal{B} ^{\pi} \right) ^{n-1}\mathcal{B}  Q \right) }  $ 
    				& $\color{red}{ \neq \vfTable[*]}$ 
                        & 
				\color{blue}{
				\thead{For any $\policySet, \stepSet,$$\policyDist, \stepDist$} 
				}				
                        & \YES
                        \\
                        \\
				\cdashline{1-5}[4pt/2pt]
				\noalign{\smallskip}
				\thead{
				\textbf{\highwayGenOperatorText/} (\crefnop{eq_highwayOperator})
				}
                    & $\highwayOperator Q \triangleq 
                 \EPolicyStep
                                {
                                \color{maxStepOneN}
                                \max\limits_{n' \in \{1, n\} }
                                }
                                        \left( \left( \mathcal{B} ^{\pi} \right) ^{n'-1}\mathcal{B}  Q \right) 
                                    $
				& $\color{blue}{ \vfTable[*] }$  
				&  
				\color{blue}{
				\thead{For any $\policySet, \stepSet,$$\policyDist, \stepDist$} 
				}
                    & \YES
				\\
    				\thead{
        				\textbf{\highwayOptOperatorText/} (\crefnop{eq_highwayOptOperator})
    				}
                        &$\highwayOptOperator Q  \triangleq
                                \MaxPolicyMaxStep
                                \left( \left( \mathcal{B} ^{\pi} \right) ^{n'-1}\mathcal{B}  Q \right)
                            $
    				& $\color{blue}{ \vfTable[*] }$  
    				&  
    				\color{blue}{
    				\thead{For any $\policySet, \stepSet$} 
    				}
                        & \YES
				\\
				\noalign{\smallskip}
				\thead{
				    \textbf{\highwaySoftmaxOperatorText/}  (\crefnop{eq_softmaxHighwayQOperator})
				}
                    &$\highwaySoftmaxOperatorMath Q \triangleq 
                    \SoftmaxPolicySoftmaxStep
                \left(\left( \mathcal{B} ^{\pi} \right) ^{n'-1}\mathcal{B} Q\right) $
				& $\color{blue}{ \vfTable[*] }$  
				&  \color{blue}{For any $\policySet$, $ \stepSet$, $\policyDist$, $\alpha$}
                & \YES
				\\
				\bottomrule
			\end{tabular}
			\caption{
			Properties of the operators. 
                $ \policySet$ denotes the \policiesText/, $\pi'$ denotes the target policy for evaluation and $\pi^*$ denote the optimal policy.
			}\label{table_summary}
\end{table*}

\section{Algorithms }\label{sec_method}

Our \highwayOperatorText/s allow for the natural derivation of several RL algorithms for model-based and model-free settings.

\subsection{\highwayValueIteration/}
We first present our method in the model-based setting.
For convenience, we use the state VF $V$ instead of the state-action VF $Q$.
The state VF $V$ is updated as follows\footnote{
We use the same notation $\highwayOperator$ consistently for both cases of the state VF $V$ and the state-action VF $Q$ when there is no ambiguity. Likewise, this convention is also applied to the Bellman operators.
}:

\begin{equation}\label{eq_highwayOptOperator_V}
   V_{k} 
   \leftarrow \highwayOperator V_{k-1}
   \triangleq 
   { \color{maxPolicyStep} 
   \EPolicyStep
   }
    {
    \color{red}
    \max_{n' \in \{1, n\} }
    }
    \left( \left( \mathcal{B} ^{\pi} \right) ^{n'-1}\mathcal{B}  V_{k-1} \right).
\end{equation}
We incorporate the new greedy policies into the \policySetText/ $\policySet$ during the learning process and implement the set $\policySet$ as a first-in-first-out queue with a maximal capacity.
The resulting algorithm, \emph{\highwayValueIteration/}, is presented in \Cref{ap_alg_GreedyMultistepValueIteration}.

\begin{algorithm}[t]
\definecolor{colorDifference}{named}{black}
\counterwithin{algorithm}{section}
\caption{ \highwayValueIteration/ }%
\begin{algorithmic}\label{ap_alg_GreedyMultistepValueIteration}
\STATE \textbf{Input:} 
The error bound $\epsilon$;
{
\color{colorDifference}{
initial \policySetText/ $\policySet_0$; 
the size of the \policySetText/ $M$;
the selection distribution of the \policySetText/ $\policyDist$;
the \stepSetText/ $\stepSet$ and the corresponding selection distribution $\stepDist$; 
the interval for adding new policy $K$.
}
}
\STATE \textbf{Initialize:} Initialize VF $V\K[0] \in \R^{|\sSpace|} $; %
{\color{colorDifference}{
$\policySet \leftarrow \policySet_0  $. // $\policySet$ is implemented as a first-in-first-out queue with a maximal capacity $M$.
}}
\FOR{$k=1,2,\ldots$}
        {
        \color{colorDifference}{
 	\IF {$(k-1) \:\text{mod}\: K == 0$}
	\STATE {\small $ \pi_{k}(s)=\argmax_{a} [ r(s,a) + \gamma \E_{s'}[ V_{k-1}(s') ] ]$} %
	\STATE $\widehat{\Pi} \leftarrow \widehat{\Pi }_{}\cup \left\{ \pi _{k } \right\}$
	\ENDIF
        }
	\STATE $V\K[k] \leftarrow  \highwayOperator V\K[k-1]$
        }
    \STATE \textbf{if} $ \| V\K[k]-V\K[k-1] \|_\infty \leq \epsilon $ \textbf{then break end if}
\ENDFOR
\end{algorithmic}
\end{algorithm}

\subsection{Highway DQN}

In the following, we extend our {\highwayOperatorText/s} to the function approximation case, where the VF $Q_{\theta}(s,a)$ is represented by a deep neural network parametrized by $\theta$ \citep{mnih2015human,van2015deep}.  
The VF $Q_{\theta}$ is optimized by the following objective function:
$$
	L\left( \theta \right) =
	\sum_{\left( s,a \right) \in \widehat{ \mathcal D} }^{}
	\left[ Q_{\theta}\left( s,a \right) 
	-
        Q^{\rm target}_{\theta'} \left(s,a \right)
\right]^2,
$$
where $ \widehat{\mathcal D} =\{(s,a )\} $ represents the sampled batch data of state-action pairs;
$Q^{\rm target}_{\theta'}(s,a )$ is the computed \emph{target Q value} and $\theta'$ are the parameters of the target Q network $Q_{\theta'}$, which are occasionally copied from $\theta$.

We extend the \highwaySoftmaxOperatorTextMath/ in \cref{eq_softmaxHighwayQOperator}, which is designed to aggregate information from various lookahead policies and lookahead depths and achieve a balance between efficiency and the reducing overestimation.
The target Q value is computed as follows:
\begin{equation}\label{eq_highwayDQN}
Q^{\rm target}_{\theta'}(s,a) 
= 
{\color{softmaxPolicyStep}
\underset{m\in \policySetModelFree_{s,a}  }{\mathop {smax}\nolimits^{\alpha}}
 \underset{n\in \stepSet }{\mathop {smax}\nolimits^{\alpha}}
}
{\color{red}\max_{ n' \in \{1,n\} }}
    \widehat{\mathbb{E} }_{\mathcal{D}_{s,a}^{( m )}}
    \left[ \sum_{\stepIndex=0}^{n'-1}{\gamma ^{\stepIndex}}r_{\stepIndex}+\gamma ^{n'}\max_{a_{n'}' \in \mathcal{A}} 
    Q_{\theta'}\left(
    s_{n'},a_{n'}'\right)
    \right],
\end{equation}
where 
$\dataset[m][s,a]=\{ \traj[][s,a] | \traj[][s,a]\sim \policy[m] \}$ is the $m$-indexed dataset containing the subsequences of trajectory data starting from $(s,a)$ collected by the $m$-indexed behavioral policy $\policy[m]$;
$\widehat{\mathbb{E} }_{\mathcal{D}_{s,a}^{( m )}}\left[  \cdot \right]=\frac{1}{\left|\mathcal{D} _{s,a}^{\left( m \right)}\right|}\sum_{\tau _{s,a}^{n}\in \mathcal{D} _{s,a}^{\left( m \right)}}^{}{\left[ \cdot \right]}$ is the empirical averaged;
$ \policySetModelFree_{s,a} \subseteq \left\{ m \big| | \dataset[m][s,a] | \neq 0 \right\}$ is a sampled batch of indexes of the behavioral policy.
Note that to compute $Q^{\rm target}_{\theta'}(s,a)$ it is sufficient to compute various $n$-step returns and then apply max and softmax operations over them. 
Therefore, the only requirement is to store the trajectory data generated by the behavioral policy $\policy[m]$ in the corresponding $m$-indexed datasets $\dataset[m][s,a]$, without the necessity of accessing or saving the behavioral policy $\policy[m]$. 
The computational complexity is associated with the number of the lookahead depths $|\stepSet|$ and the number of the sampled behavioral policies $|\policySetModelFree|$. It is comparable to existing $\lambda$-return-based methods when $|\policySetModelFree|=1$ \citep{Schulman2016HighDimensional,munos2016safe}. 
The resulting algorithm, \emph{\softHighwayDQN/}, is presented in \Cref{ap_alg_GM_DQN}. 
By setting $\stepSet=\{n\}$, we obtain a specific instance of \softHighwayDQN/, referred to as \emph{\nstepHighwayDQN/}. 
In \Cref{ap_sec_algorithm}, we also derive an algorithm named \highwayQLearning/ for the model-free tabular case.

\newlength{\widthtmpadhaaswdhd}
\setlength{\widthtmpadhaaswdhd}{0.37in}
\def\setalignAUDIJEIHDKJKD{ \quad\quad\quad\quad }
\def\highwayAlgorithmInput{
\StepSetText/ $\stepSet$; \\
\setalignAUDIJEIHDKJKD Number of sampled behavioral policies $M$; \\
\setalignAUDIJEIHDKJKD Epochs of rolling-out each policy $\epochRolloutPolicy$; \\
\setalignAUDIJEIHDKJKD {Epochs of running algorithm $\epochRunAlg$; \\
\setalignAUDIJEIHDKJKD Epochs of updating value function $\epochUpdateValueFunction$;} \\
\setalignAUDIJEIHDKJKD Exploration rate $\epsilon$;
}
\def\highwayAlgorithmWhere
{
\begin{flalign*}
\hspace{1in} \text{where} \quad &
    \widehat{\mathbb{E} }_{\mathcal{D}_{s,a}^{( m )}}\left[  \cdot \right]=\frac{1}{\left|\mathcal{D} _{s,a}^{\left( m \right)}\right|}\sum_{\tau _{s,a}^{n}\in \mathcal{D} _{s,a}^{\left( m \right)}}^{}{\left[ \cdot \right]}
\\
\end{flalign*}
}
\def\epochRunAlg{I_{\rm run}}
\def\epochRolloutPolicy{I_{\rm rollout}}
\def\epochUpdateValueFunction{I_{\rm update}}
\def\batchSize{ b }

\begin{algorithm}[]
\small
\caption{ \softHighwayDQN/ }%
\begin{algorithmic}\label{ap_alg_GM_DQN}
\STATE \textbf{Input:} 
\highwayAlgorithmInput \\
\setalignAUDIJEIHDKJKD Softmax temperature $\alpha$; \\
\setalignAUDIJEIHDKJKD Epochs of updating value function $\epochUpdateValueFunction$; \\
\setalignAUDIJEIHDKJKD Batch size of updating value function $\batchSize$; \\
\STATE \textbf{Initialize:} 
Replay buffer $\dataset[][] = \emptyset$;
Initialize parameter $\theta$ of $Q$ function;
Initialize parameter $\theta'$ of target $Q$ function by $\theta'\leftarrow \theta$.
\FOR{$m=1,\cdots, \epochRunAlg$}
    \STATE Set $\policy[m]$ to be $\epsilon$-greedy policy with $Q_{\theta}$
    \STATE Initialize the dataset for the policy $\policy[m]$ $\dataset[m][s,a] \leftarrow \emptyset$, for all $ (s,a) \in \sSpace \times \aSpace$
    \FOR{$j=1,\cdots,\epochRolloutPolicy$ }
    	\STATE  Starting from initial state $s_0$, collect a trajectory $\tau=( s_{0}, a_{0}, r_{0}, \cdots, s_{t}, a_{t}, r_{t}  , \cdots, s_T)$ with $\policy[m]$
    	\FOR{$t=0,1,\cdots,T-1$ }
    	    \STATE 

    		Add the the subsequence of trajectory $(s_t,a_t, r_t, \cdots, s_{t+\stepIndex}, a_{t+\stepIndex}, r_{t+\stepIndex}, \cdots, s_T)$ to $\dataset[m][s_t,a_t] $
            \STATE
            Add $\left( s_{t}, a_{t}\right)$ to $\dataset[][] $
    	\ENDFOR
    \ENDFOR
	\FOR{$j=1, \cdots, \epochUpdateValueFunction$ }
    \STATE Sample a mini-batch $ {\widehat{ \mathcal D} }^{}= \{(s,a \} $ from $\dataset[][]$ 
    \STATE Sample a mini-batch $\policySetModelFree_{s,a}$ with maximal size $M$ from $\left\{ m \big| | \dataset[m][s,a] | \neq 0 \right\}$ for each sampled $(s,a) \in {\widehat{ \mathcal D} }^{} $ 
	\STATE 
	Update $\theta$ by minimizing the following loss function: 
	\begin{equation}
	L\left( \theta \right) =
	\sum_{\left( s,a \right) \in \widehat{ \mathcal D} }^{}
	\left[ Q_{\theta}\left( s,a \right) 
	-
        Q^{\rm target}_{\theta'} \left(s,a \right)
\right]^2,
	\end{equation}
\begin{flalign*}
\text{where} \quad &
Q^{\rm target}_{\theta'}(s,a) 
= 
{\color{softmaxPolicyStep}
\underset{m\in \policySetModelFree_{s,a}  }{\mathop {smax}\nolimits^{\alpha}}
 \underset{n\in \stepSet }{\mathop {smax}\nolimits^{\alpha}}
}
{\color{red}\max_{ n' \in \{1,n\} }}
    \widehat{\mathbb{E} }_{\mathcal{D}_{s,a}^{( m )}}
    \left[ \sum_{\stepIndex=0}^{n'-1}{\gamma ^{\stepIndex}}r_{\stepIndex}+\gamma ^{n'}\max_{a_{n'}' \in \mathcal{A}} 
    Q_{\theta'}\left(
    s_{n'},a_{n'}'\right)
    \right]
\end{flalign*}
	\ENDFOR
	\STATE $\theta' \leftarrow \theta$ every $C$ timesteps
\ENDFOR
\end{algorithmic}
\end{algorithm}

\section{Related Work}\label{sec_related_work}
Multi-step RL methods have a long history. Examples of such methods include multi-step SARSA \citep{sutton2018reinforcement}, Tree Backup \citep{precup2000eligibility}, Q($\sigma$) \citep{asis2017multi}, and Monte Carlo methods (which can be seen as an example of $\infty$-step learning).
These methods are based on the concept of $\lambda$-return, which involves assigning exponentially decaying weights to future rewards based on a decay factor $\lambda$ \citep{sutton2018reinforcement,Schulman2016HighDimensional,WhiteW16a}.
The $\lambda$-return weights can also be learned directly end-to-end \citep{SharmaRJR17}.
The choice of the \step/, decay factor, and weights represents prior knowledge or biases regarding appropriate credit assignment intervals, typically tuned on a case-by-case basis. In contrast, our method adaptively adjusts the lookahead depth based on the quality of the data and the learned VF. 
The RUDDER method \citep{NEURIPS2019_16105fb9} trains an LSTM to re-assign credits to previous actions. 
It is quite different from our simple but sound principle for credit transportation at minimal cost.
Existing off-policy learning methods often employ additional corrections to handle off-policy data.
Classical importance sampling (IS)-based methods suffer from high variance due to the products of IS ratios \citep{cortes2010learning,metelli2018}. 
To address this issue, various variance reduction techniques have been proposed.
Retrace($\lambda$) \citep{munos2016safe} reduces the variance by clipping the IS ratios, and has achieved great success in practice.
Other work ignores IS:
TB($\lambda$) \citep{precup2000eligibility} corrects off-policy data by computing expectations with respect to the data and estimates a VF using the probabilities of the target policy. 
Q($\lambda$) \citep{harutyunyan2016q} corrects off-policy data by incorporating an off-policy correction term.
Our method offers an alternative approach to off-policy learning that does not rely on IS. 
Similar to Retrace($\lambda$), it can safely use arbitrary off-policy data. Moreover, it is highly efficient and offers faster convergence rates under mild conditions. 
However, further research is needed to understand how the variance of our method compares to that of advanced IS-based variance reduction methods like Retrace($\lambda$). 

Searching over various lookahead depths and policies to improve the convergence of RL systems is an active area of research. 
\citet{wright2013exploiting} were the first to propose searching for maximal bootstrapping values using different lookahead depths.
\citet{she2017learning} suggested searching over various lookahead depths along the trajectory data, incorporating the maximal value as a regularization term rather than directly as the Q target value. Additionally, their bootstrapping values in the maximization do not involve the one-step return, which we have identified as a critical factor for ensuring correct convergence (see \Cref{sec_highwayOperator}).
Recent work on Monte Carlo augmented Actor-Critic \citep{wilcox2022monte} proposed maximizing the one-step return and Monte Carlo returns, which is a special case of our method where the lookahead depth equals the length until the end of the trial. Generalized policy update \citep{barreto2020fast} focuses on searching over various policies while using a fixed lookahead depth until the end of the trial but does not consider searching over multiple policies.

Our work introduces a theoretically justified Bellman operator that searches both policies and lookahead depths, providing greater flexibility and scalability across different settings and applications. We also address the issue of overestimation that can arise from maximization operations in previous methods \citep{barreto2020fast,she2017learning} by introducing a novel softmax operation that alleviates this problem while maintaining convergence to the optimal value function.
Compared to previous methods that search over the product of the original policy space and action space \citep{efroni2018beyond,efroni2019combine, efroni2020online, jiang2018feedback, tomar2020multi}, our approach has a smaller search space, limited to a specific set of policies.

Our method can also be viewed as combining the best of both worlds: (1) direct policy search using policy gradients \citep{Williams:92} or evolutionary computation \citep{barricelli1954esempi, Rechenberg:71, Fogel1966ArtificialIT} 
where the lookahead is equal to the trial length and there is no runtime-based identification of useful subprograms and subgoals, and (2) dynamic programming-based \citep{Bellman:1957} identification of useful subgoals during runtime. This naturally and safely combines best subgoals/sub-policies \citep{Schmidhuber:91icannsubgoals} derived from data collected by arbitrary previously encountered policies. It can be viewed as a soft hierarchical chunking method \citep{Schmidhuber:92ncchunker} for rapid credit assignment, with a sound Bellman-based foundation.

\section{Experiments}\label{sec_experiment}

{We conduct a series of experiments under various scenarios to answer the following questions:}
(1) Do existing IS-free vanilla multi-step methods suffer from the underestimation issue? Does a larger lookahead depth exacerbate this problem?
(2) Can our proposed highway methods alleviate the underestimation issue and safely learn with various lookahead depths?
(3) How do the components of our methods, such as the highway gate and softmax aggregation, contribute to the overall performance of our proposed method?

\begin{figure*}[!b]
       	\def\height{0.25}
       	\centering{
        \includegraphics[height=\height\linewidth]{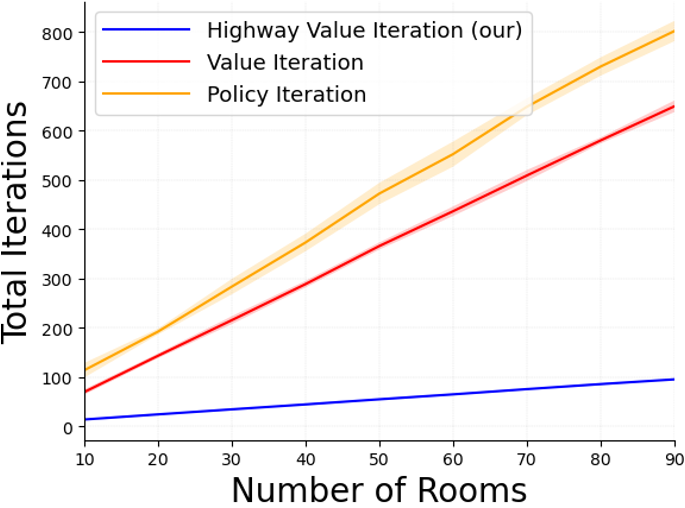}
        \includegraphics[height=\height\linewidth]{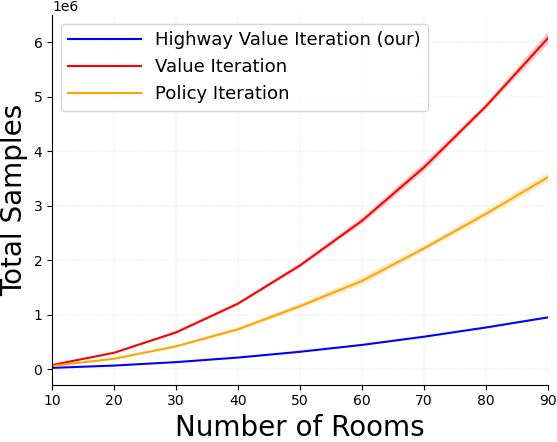}
      		}
    \caption{
        Performance of model-based algorithms in model-based Multi-Room environments. The x-axis is the number of rooms. 
        The y-axes represent the total iteration and total samples.
    }\label{fig_model_based}
\end{figure*}

\begin{figure*}[!b]
       	\def\height{0.25}
        \centering{
            \includegraphics[width=0.7\linewidth]{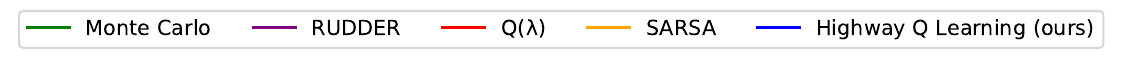}
        }
       	\centering{
                    \includegraphics[height=\height\linewidth]{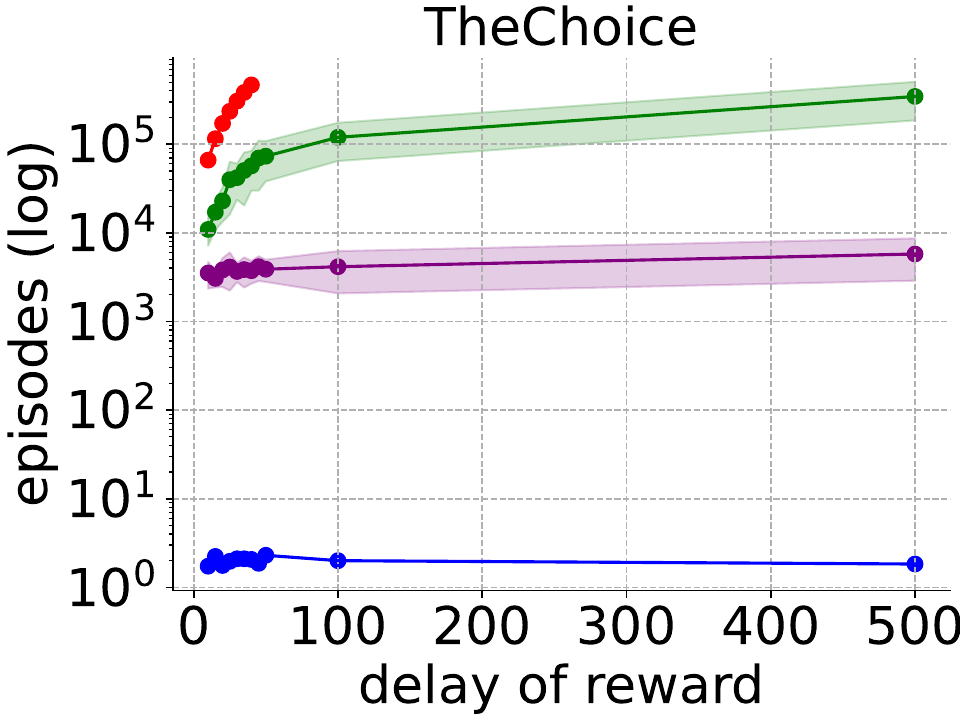}
                    \includegraphics[height=\height\linewidth]{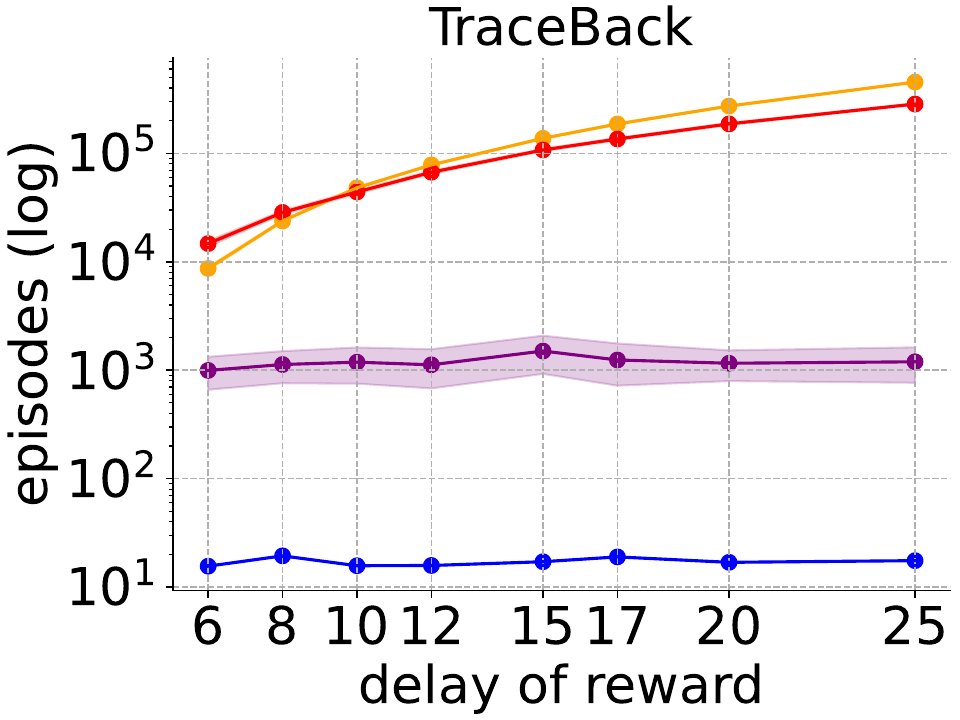}
   		}
    \caption{
        Performance of model-free algorithms in the Choice and Trace Back environments, respectively. The x-axis represents the delay of reward, while the y-axis represents the number of episodes required to solve the task. The values are averaged over 100 seeds, with one standard deviation shown.
    }\label{fig_model_free}
\end{figure*}

\subsection{Model-based Toy Tasks}
We evaluate our \highwayValueIteration/ (Highway VI) on a model-based task: Multi-Room environments with varying numbers of rooms. The agent's objective is to navigate through the rooms and reach the goal to receive a reward.
We compare our method to the classical Value Iteration (VI) and Policy Iteration (PI) algorithms, evaluating them until convergence. 
As depicted in the plots of \Cref{fig_model_based}, our Highway VI algorithm outperforms VI and PI in terms of the number of iterations (updates) required and the total number of samples (representing the total number of queries made to the MDP model).

\subsection{Model-free Toy Tasks}
We then evaluate our \highwayQLearning/ algorithm on two toy tasks where the reward is only provided at the end of the episode \citep{NEURIPS2019_16105fb9}. 
These tasks are designed to test the algorithms' ability to perform long-term credit assignments.
We compare our method to classical multi-step return-based methods, including Q($\lambda$) \citep{wiewiora2003potential}, SARSA($\lambda$) \citep{rummery1994line}, and Monte Carlo methods \citep{sutton2018reinforcement}, as well as the advanced credit assignment method RUDDER \citep{NEURIPS2019_16105fb9}.
As shown in \Cref{fig_model_free}, our method significantly outperforms all competitors on both tasks.
Our method solves the tasks within just 20 episodes for the Choice task and 100 episodes for the Trace Back task, while the second-best algorithm, RUDDER, requires over 2000 and 1000 episodes, respectively.
Notably, the number of episodes required to achieve convergence for \highwayQLearning/ does not visibly increase with the reward delays.
In contrast, other methods, such as Q($\lambda$), require exponentially increasing numbers of trials.

\begin{figure*}[!b]
    \centering
   	
\def\widthScore{0.11}
\def\width{0.11}
\def\height{0.18}
\def\heightb{\height}
\def\widthDelay{0.09}
    \graphicspath{{figsv2/exp_main/}}
    \centerline{
        \includegraphics[width=0.98\linewidth]{legend}
    }
    \centerline{
     \subfloat[MinAtar]{
            \includegraphics[height=\height\linewidth]{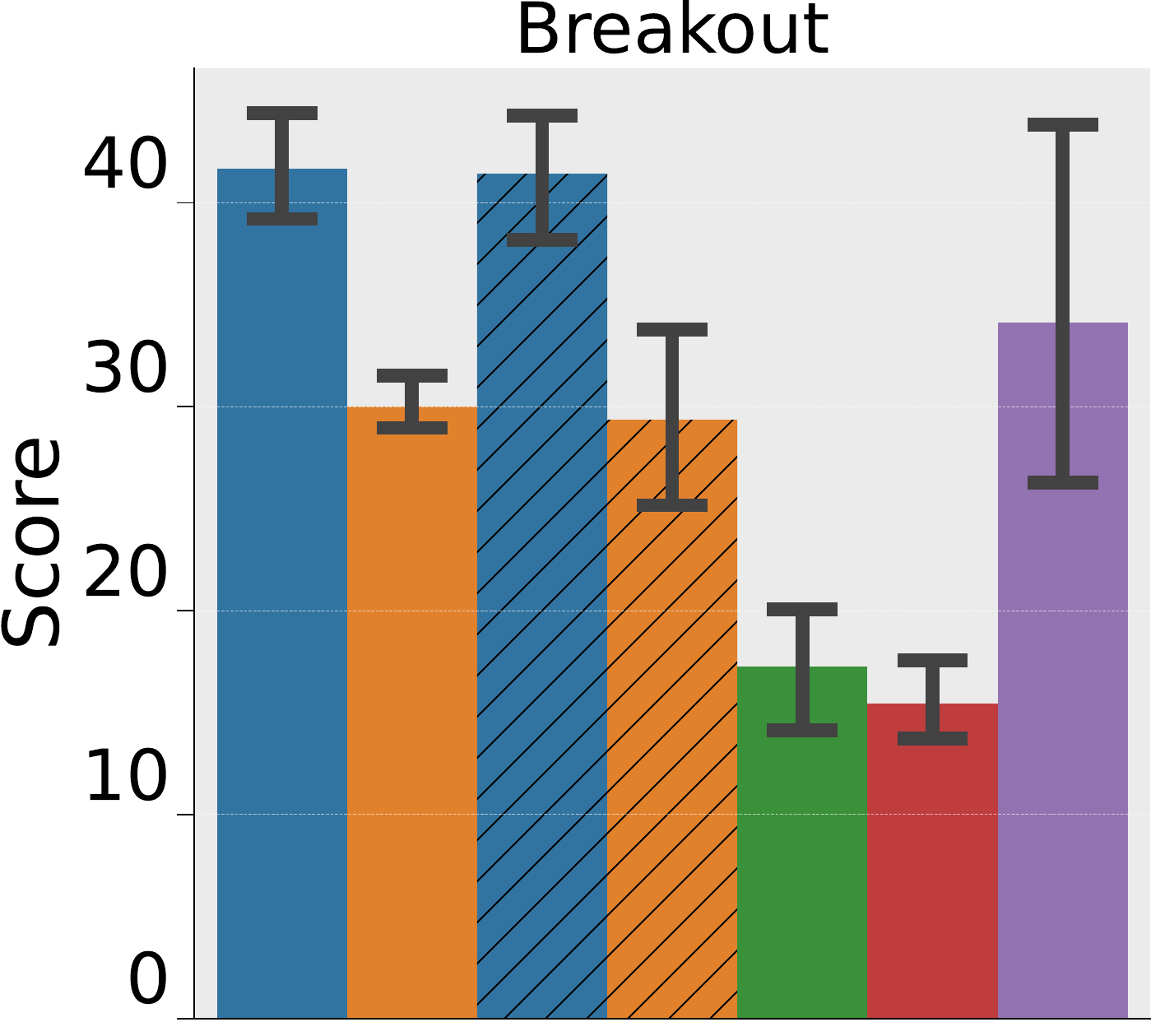}
            \includegraphics[height=\height\linewidth]{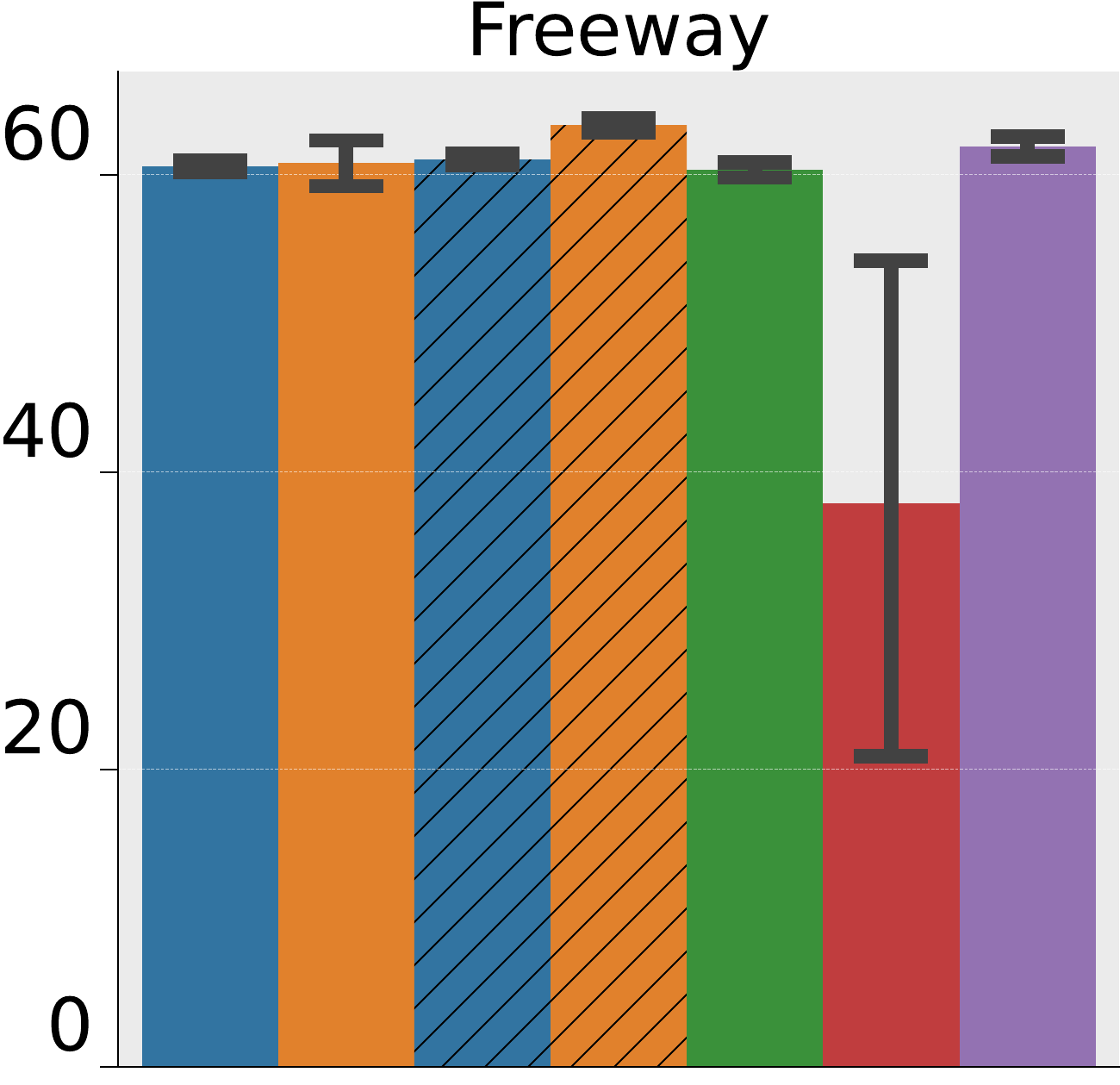}
            \includegraphics[height=\height\linewidth]{SpaceInvaders}
            \includegraphics[height=\height\linewidth]{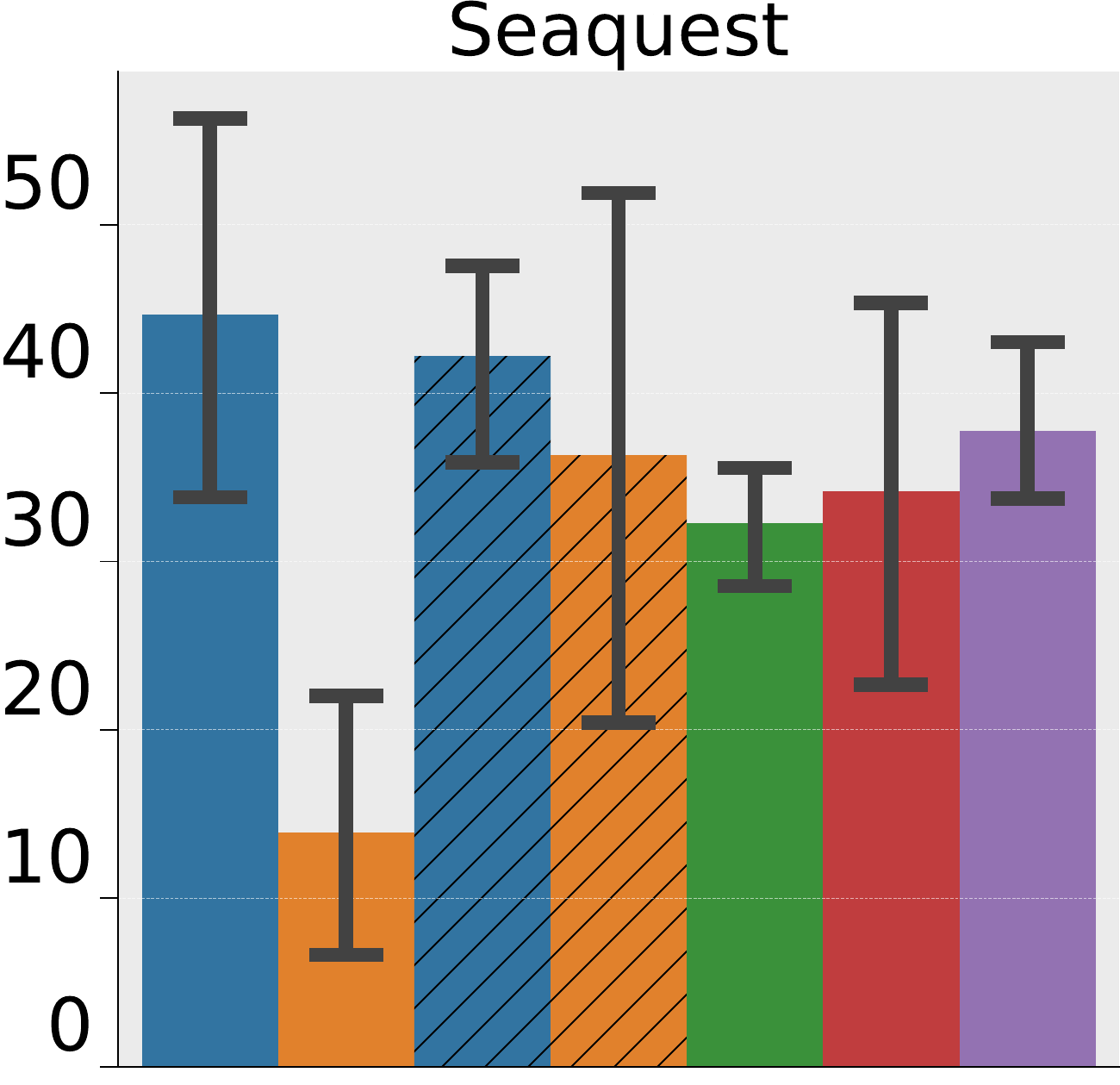}
            \includegraphics[height=\height\linewidth]{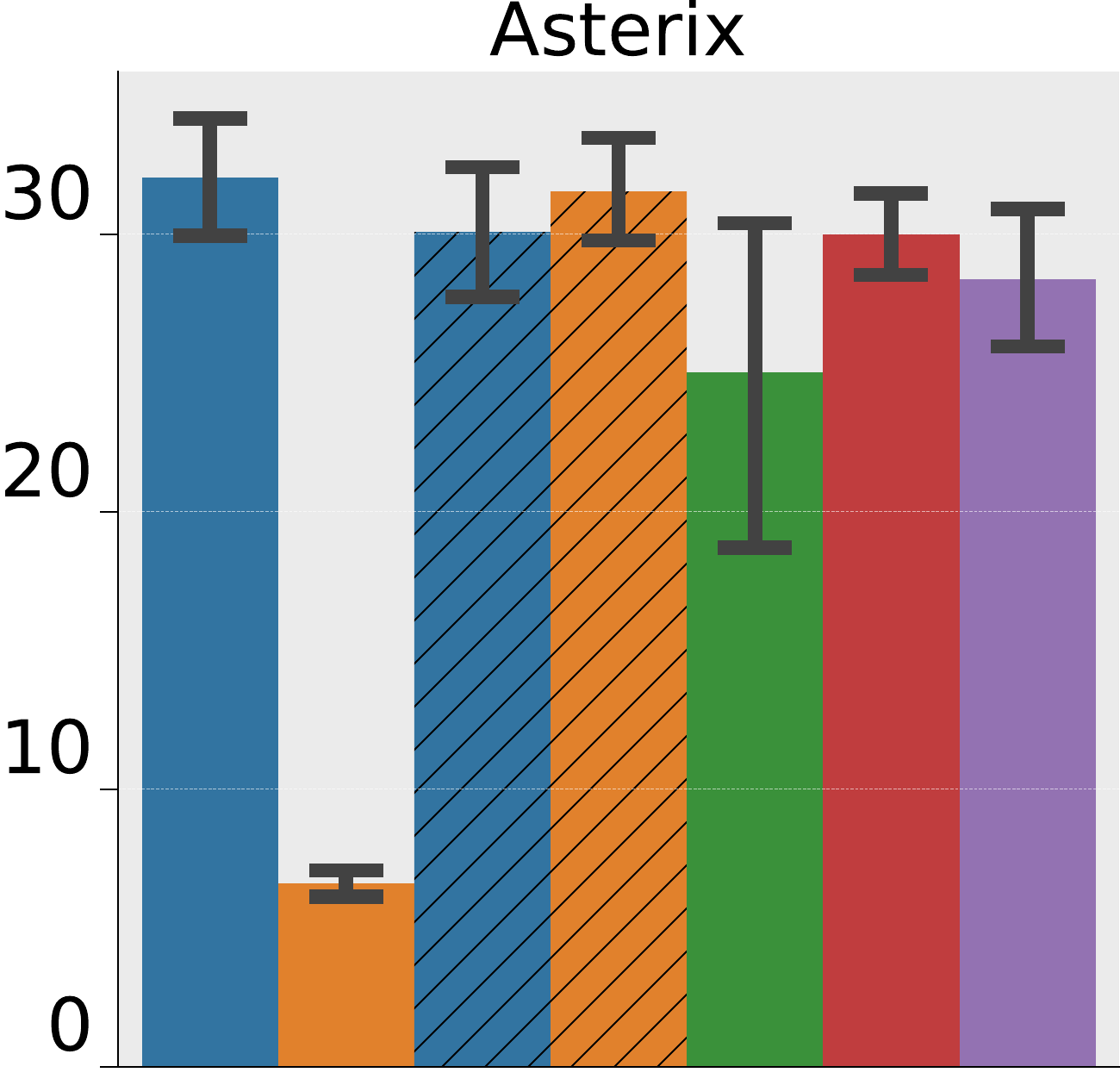}
        }
    }
    \centerline{
    \subfloat[MinAtar-Delay with a score given at the end]{
            \includegraphics[height=\heightb\linewidth]{Breakout-Delay}
            \includegraphics[height=\heightb\linewidth]{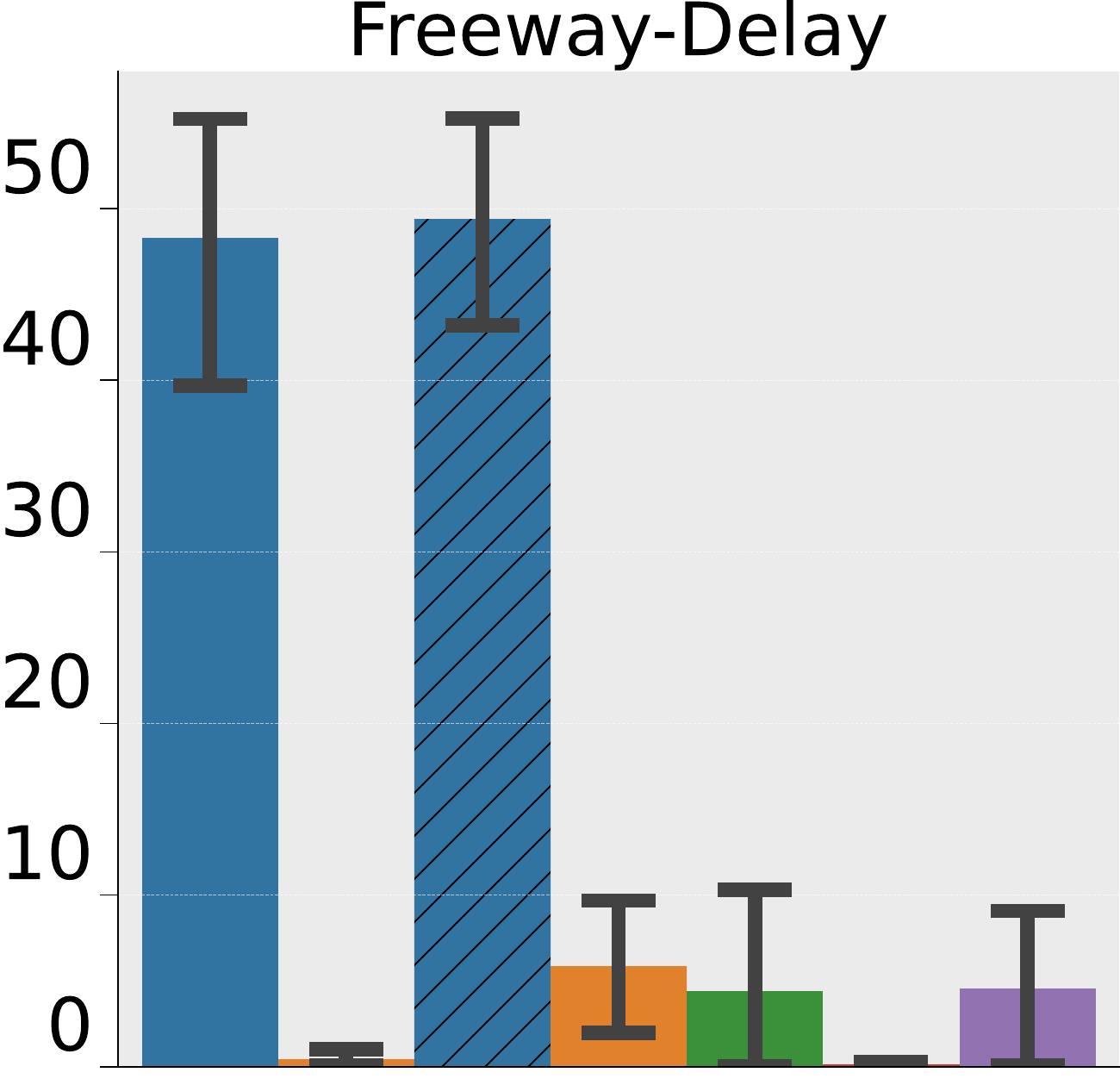}
            \includegraphics[height=\heightb\linewidth]{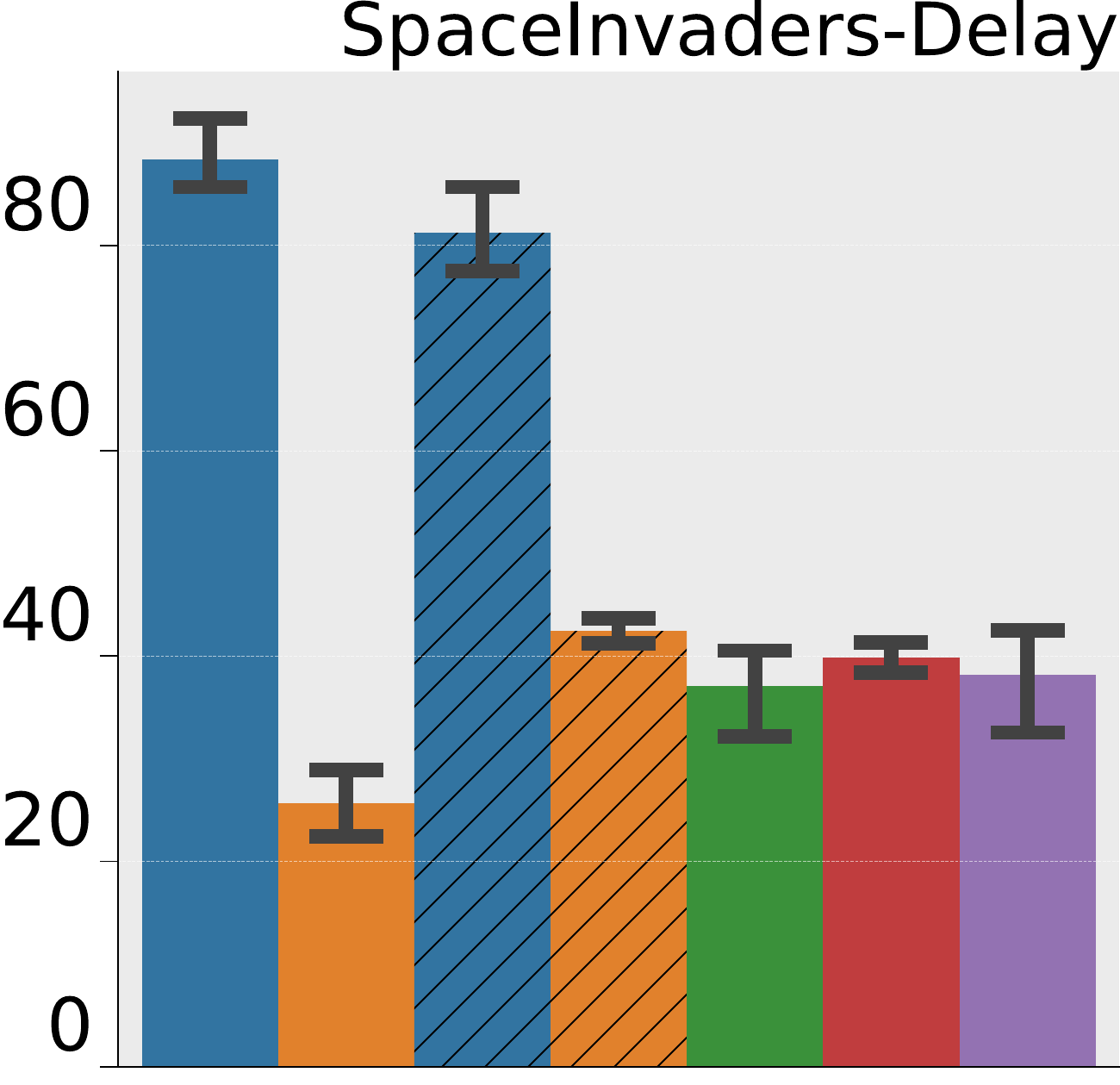}
            \includegraphics[height=\heightb\linewidth]{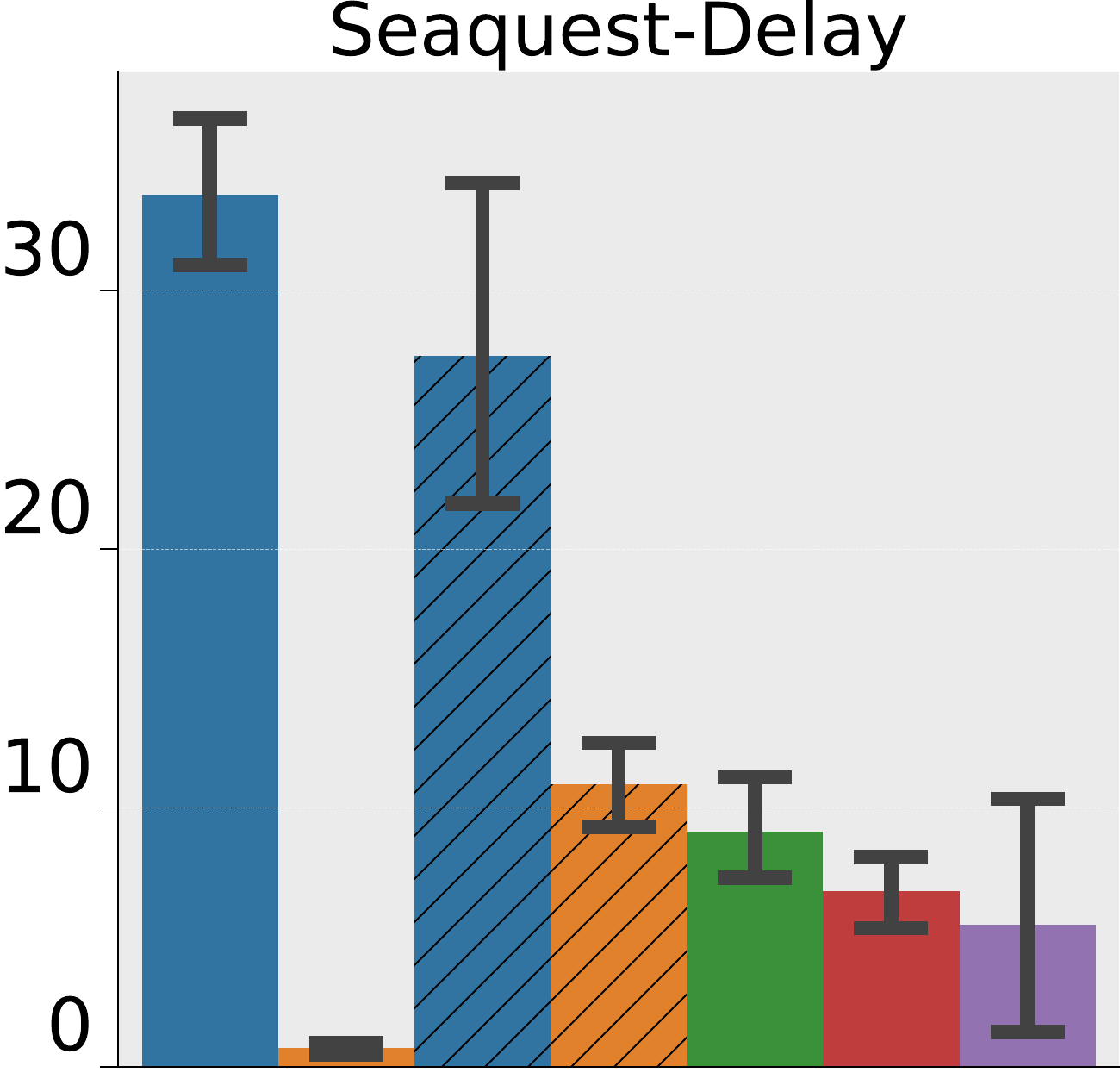}
    	\includegraphics[height=\heightb\linewidth]{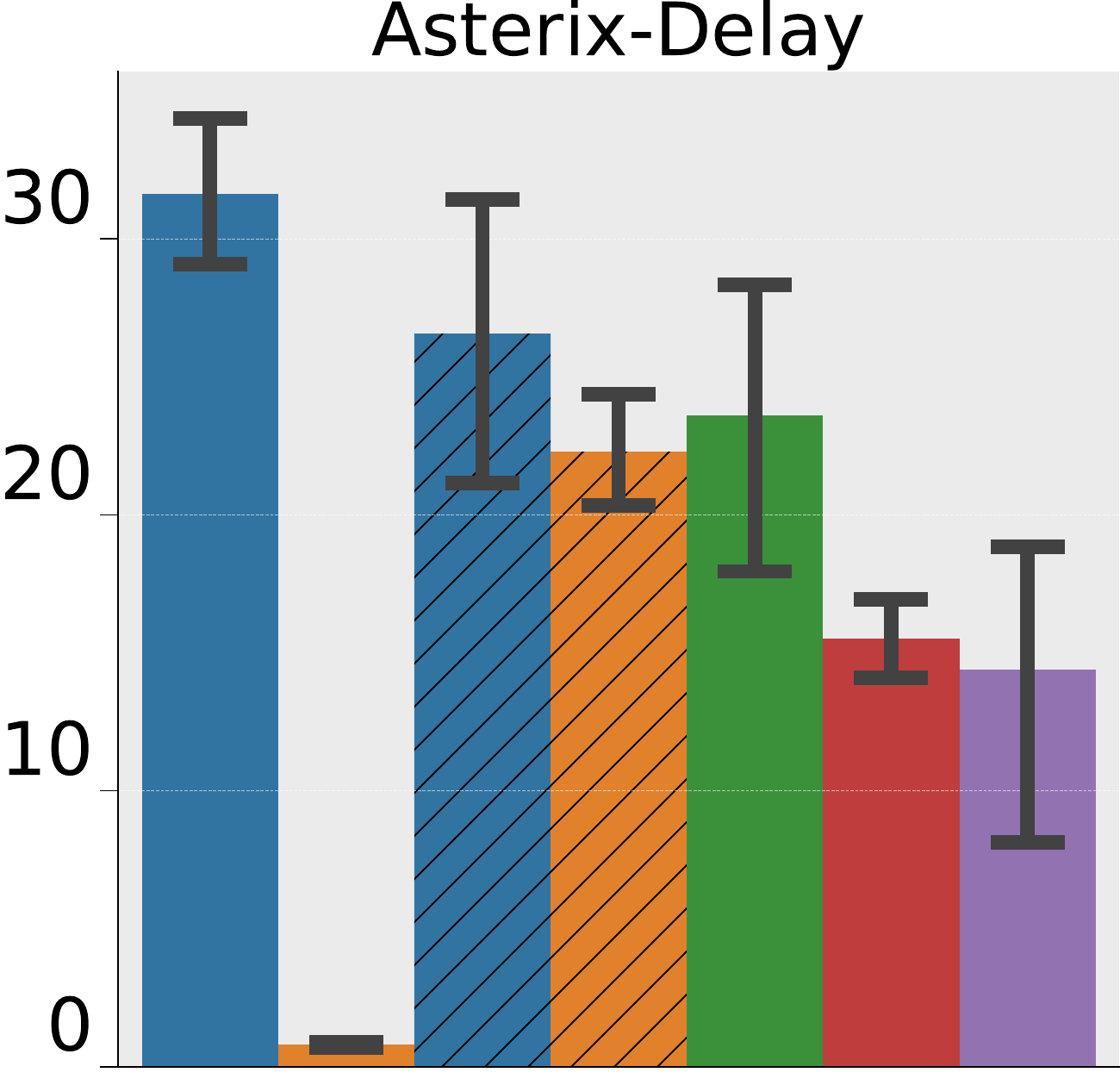}
    }
    }
    \caption{
    Average scores of the trained policies of the algorithms on MinAtar and MinAtar-Delay environments.
    }\label{fig_score_minatar_greedy_step}
\end{figure*}

\subsection{MinAtar Games}
We evaluate our Highway DQN algorithms on MinAtar benchmark tasks \citep{Kyo}.
To increase the difficulty of the reward delay, 
we propose new \emph{MinAtar-Delay} tasks, where only a total score is given at the end of the game, without any intermediate reward during the agent's interaction.
We compare our Highway DQN algorithms to several advanced multi-step off-policy methods, including \emph{Retrace($\lambda$)} \citep{munos2016safe} and \emph{$n$-step DQN} \citep{horgan2018distributed, barth2018distributed}.
For our $n$-step Highway DQN,
we choose the best lookahead depth $n$ from $\{3,\infty\}$, where $\infty$ means the depth equals the number of steps until the end of the trajectory. For $n$-step DQN, we choose the best lookahead depth $n$ from $\{2,3,4,5,8, 16,\infty\}$ to ensure we have searched for the optimal hyperparameter.
For both Soft Highway DQN and Retrace($\lambda$), the set of lookahead depths is $\stepSet=\{1,2,\infty\}$.
Additionally, for our Soft Highway DQN, we set the softmax temperature to $\alpha=0.005$ for all tasks and limit the number of sampled behavioral policies to $|\mathcal{M}|=1$ to ensure a fair comparison.
For Retrace($\lambda$), we set $\lambda=1$, as suggested by the authors \citep{munos2016safe}.
Each algorithm is run with $10$ random seeds. We test the performance of the trained models for $100$ epochs {after training with $5\times 10^6$ samples}.
To avoid overestimation, we implement all the multi-step methods based on {Maxmin DQN} \citep{Lan2020Maxmin}, which selects the smallest Q-value among multiple target Q networks.
To ensure a fair comparison, all the compared methods adopt the same implementations except for the core components.
Please refer to \Cref{ap_sec_experiment_minatar} for detailed information on the implementations and hyperparameters.

\begin{figure*}[!b]
\def\widthproperty{0.19}
\def\widthpropertylegend{0.3}
\def\height{0.26}
    \centering
    \centerline{
    \hspace*{.08in}
    \includegraphics[height=.03\linewidth]{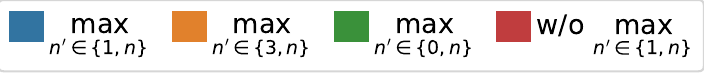}
    \hspace*{.8in}
    \includegraphics[height=.025\linewidth]{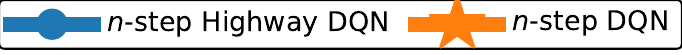}
    }
    \centerline{
        \subfloat[Highway Gate]{
            \label{fig_highway_gate}
            \includegraphics[height=\height\linewidth]{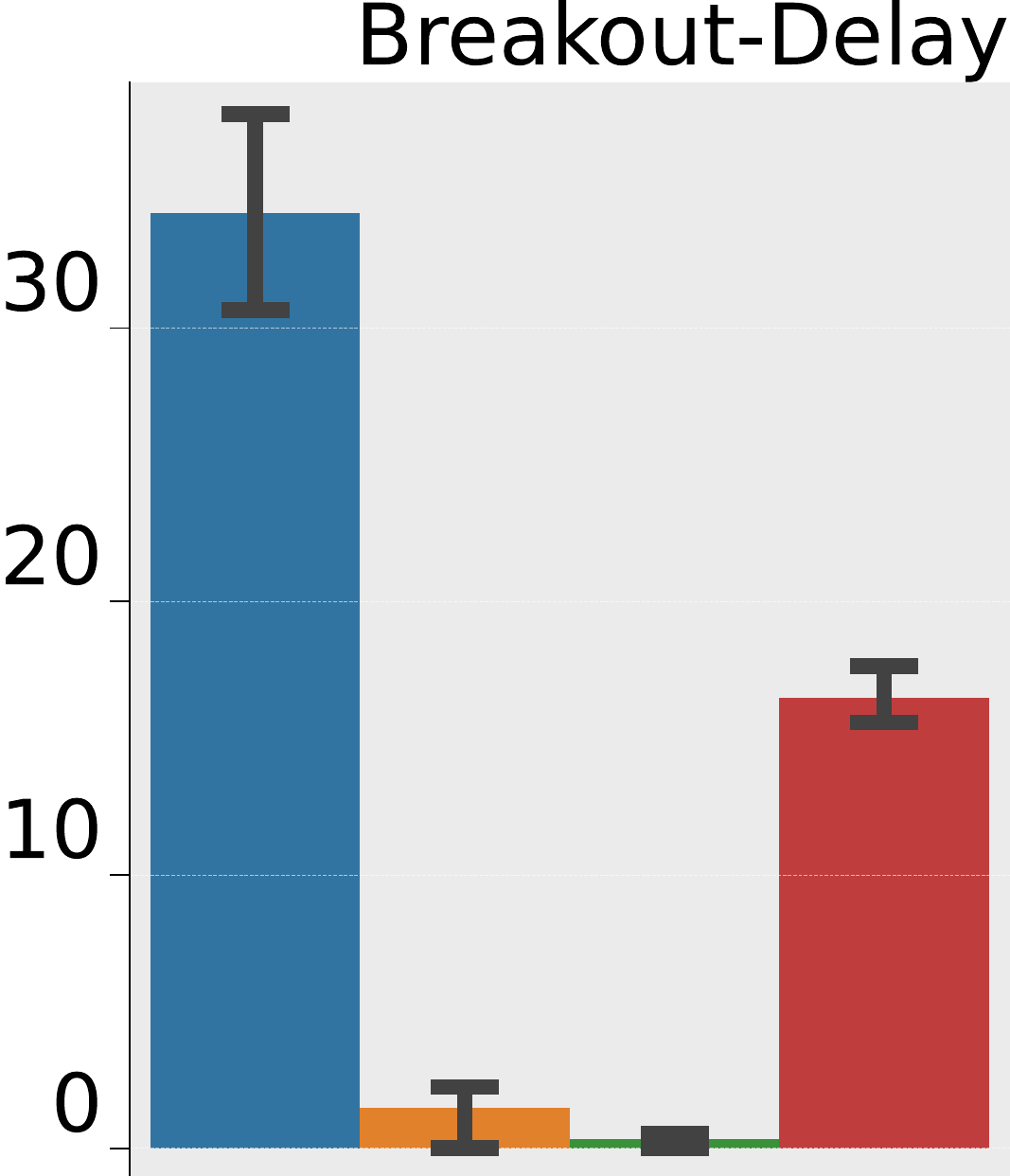}
            }
            \hspace*{0.5in}
         \subfloat[ Lookahead Depth]{
            \label{fig_n_step}
                \includegraphics[height=\height\linewidth]{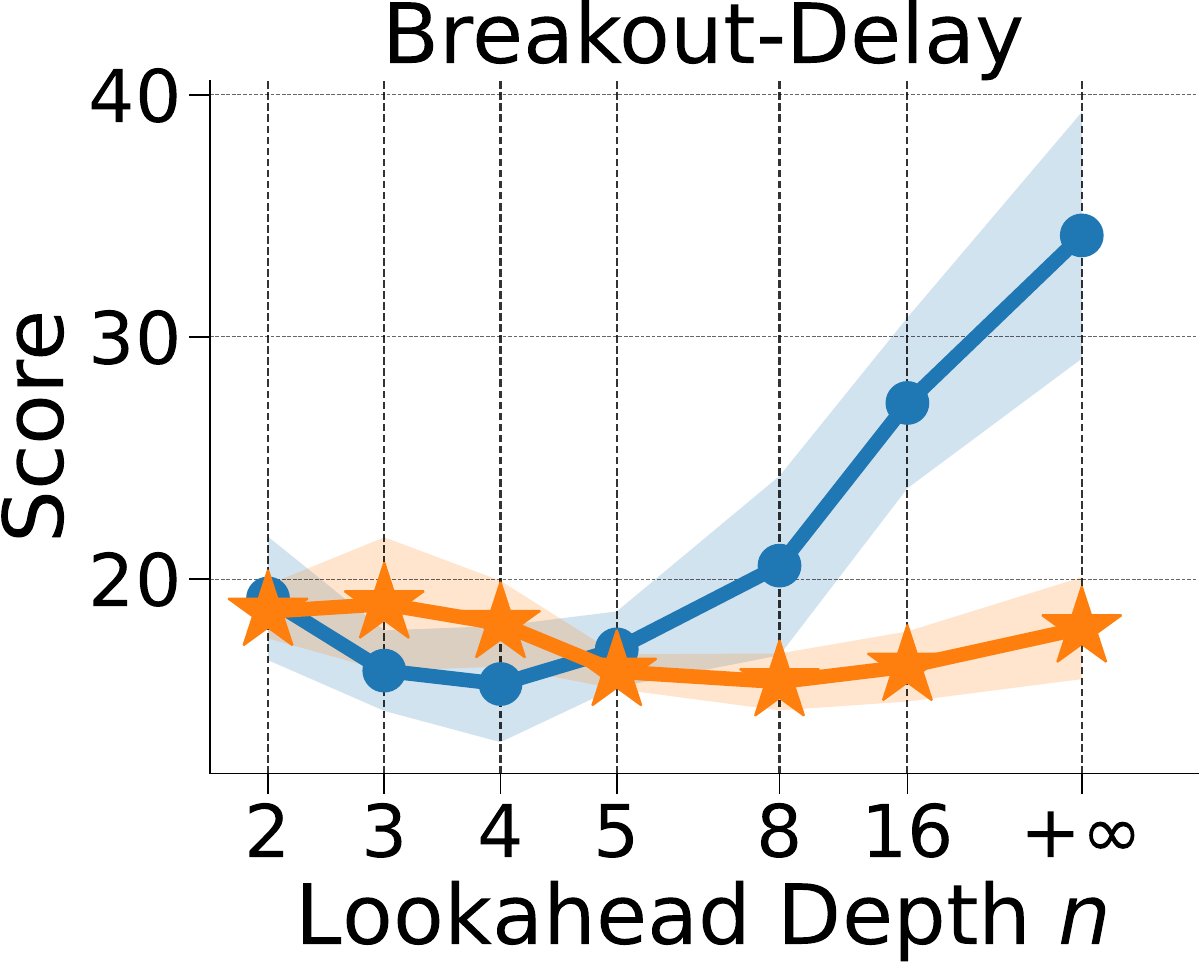}
        }
    }
    \caption{
    Ablation study results of our $n$-step \highwayDQN/.
    \subref{fig_highway_gate} shows the scores of variants of $n$-step \highwayDQN/ using different types of highway gates. 
    \subref{fig_n_step} shows the scores with varying lookahead depths for $n$-step Highway DQN and $n$-step DQN, respectively.
    }\label{Ablation_Study}
\end{figure*}

\begin{figure*}[!b]
\def\widthproperty{0.19}
\def\widthpropertylegend{0.3}
\def\height{0.26}
    \centering
    \hspace*{.08in}\centerline{
        \includegraphics[height=0.025\linewidth]{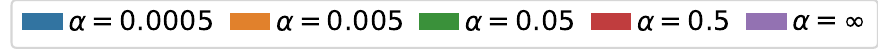}
        \hspace*{1.in}
        \includegraphics[height=0.025\linewidth]{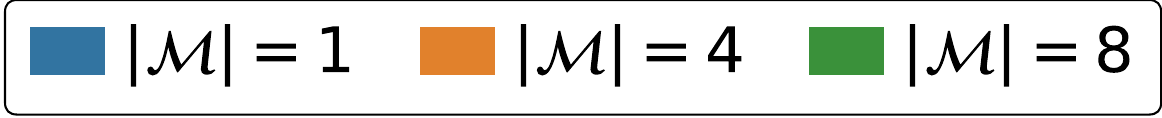}
    }
    \centerline{
          \subfloat[Softmax Temperature]{
                  \label{fig_softmax}
            \includegraphics[height=\height\linewidth]{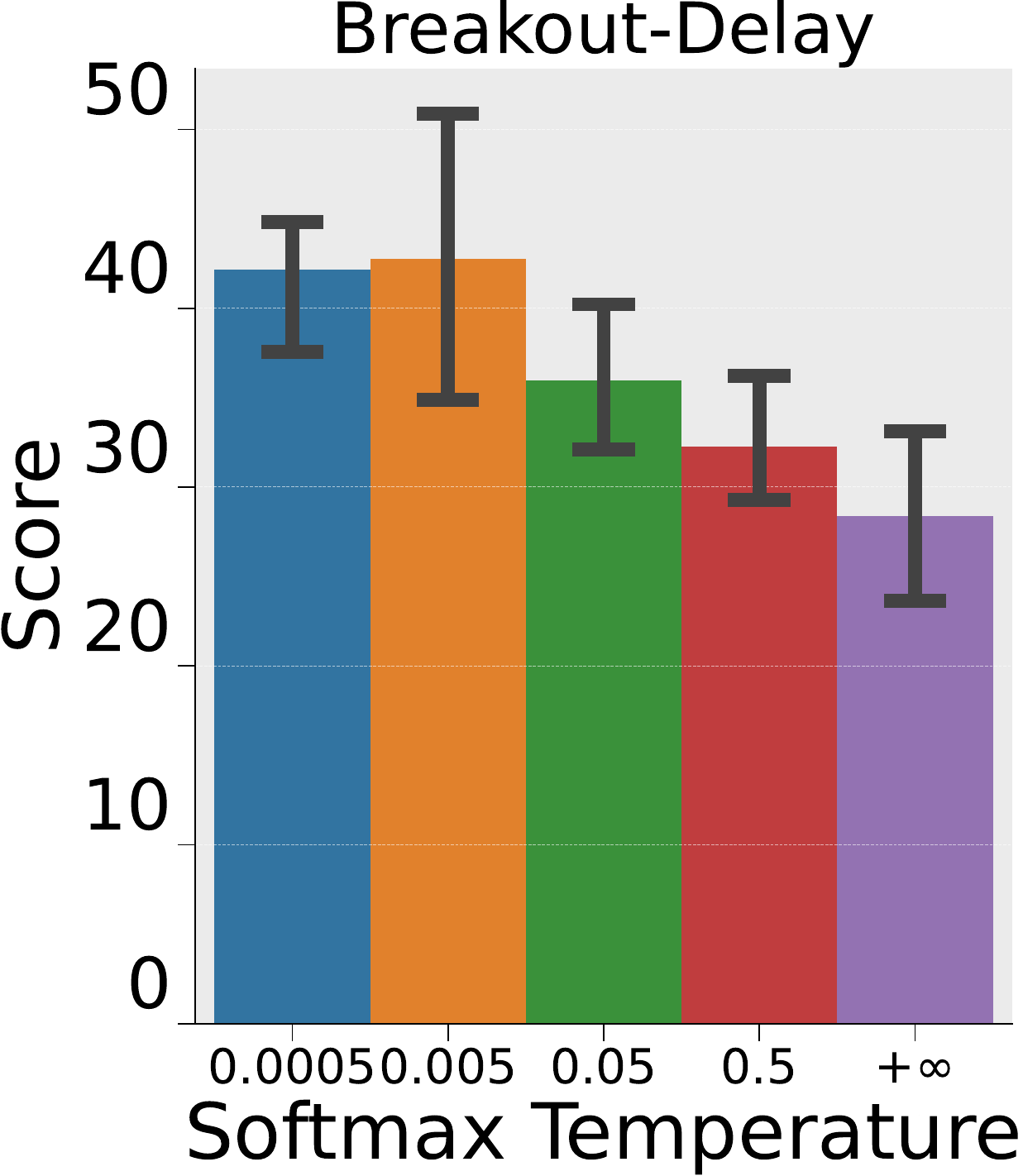}
            \includegraphics[height=\height\linewidth]{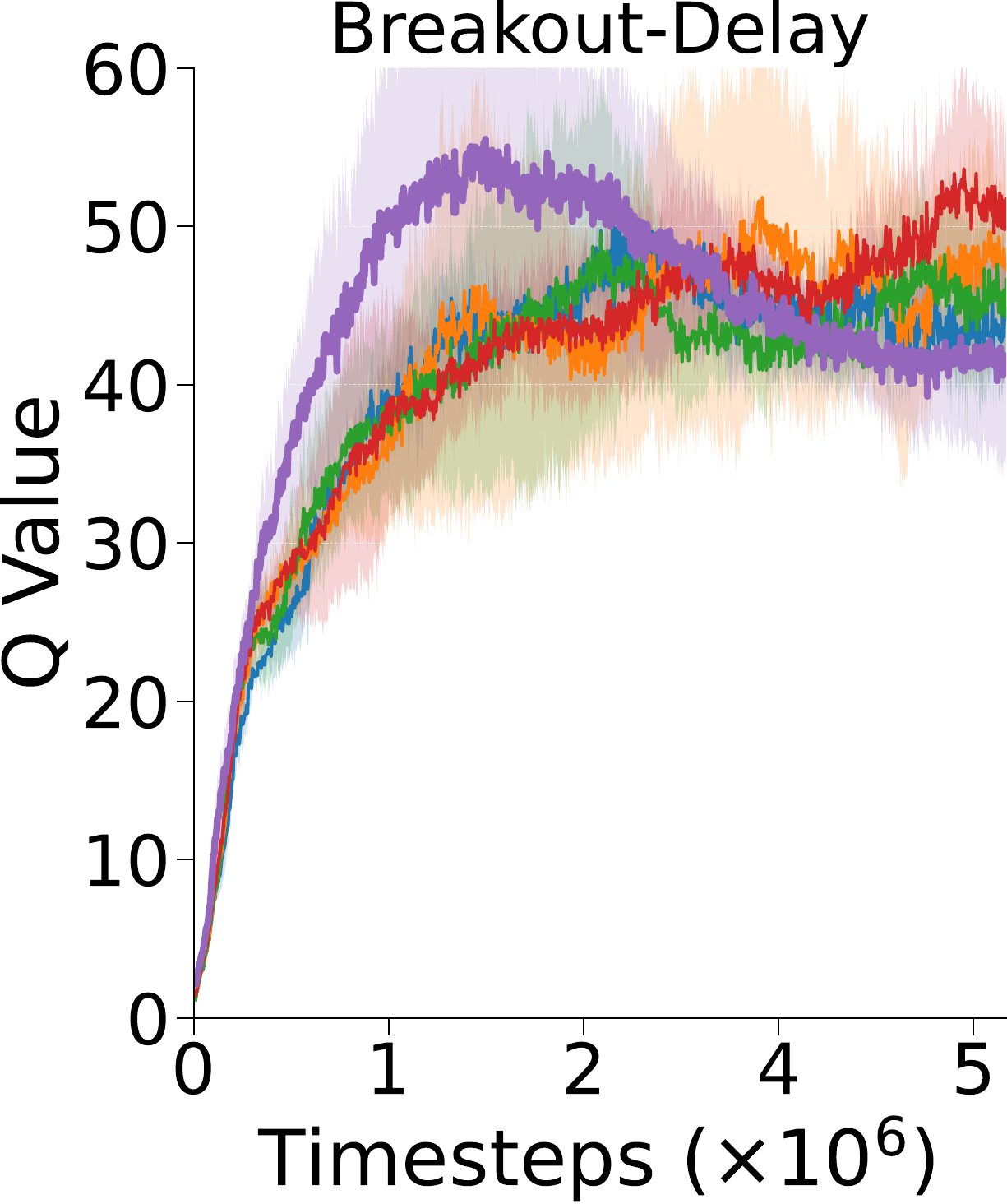}
         }
          \subfloat[Number of Lookahead {Policies}]{
                  \label{fig_n_policies}
            \includegraphics[height=\height\linewidth]{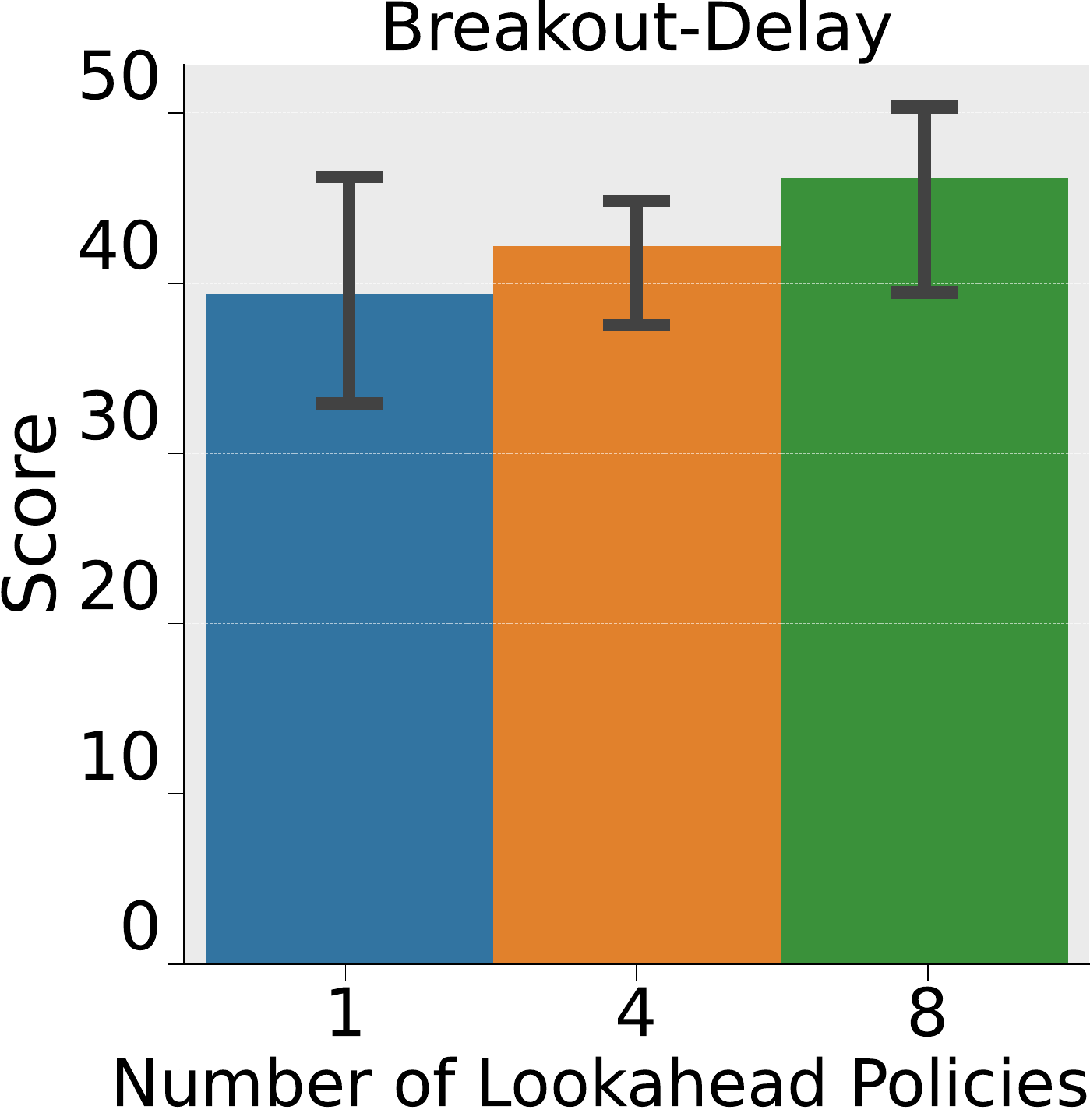}
            \includegraphics[height=\height\linewidth]{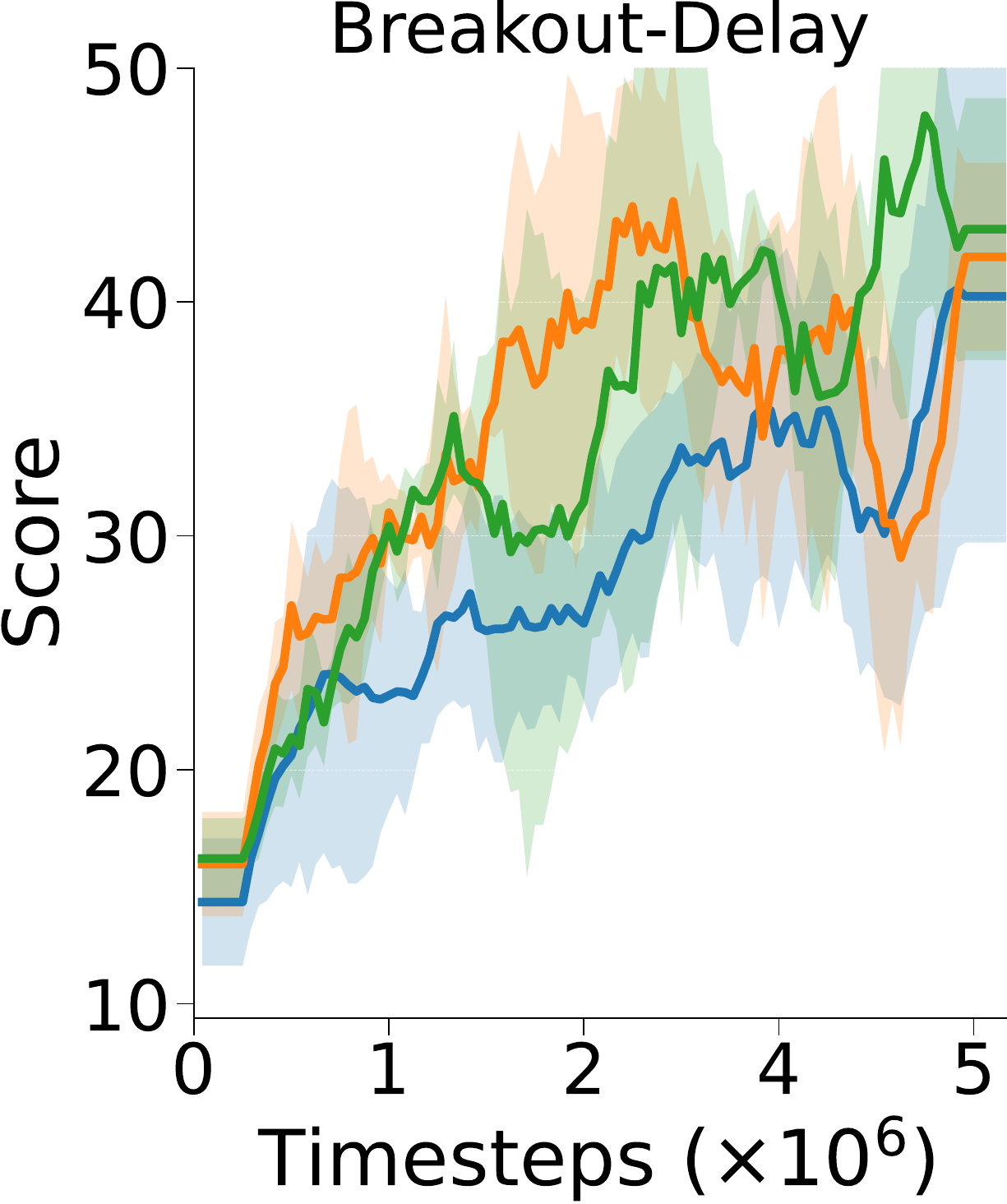}
         }
    }
    \caption{
    Ablation study results of our \softHighwayDQN/.
    \subref{fig_softmax} shows the scores (left) and Q values (right) of \softHighwayDQN/ with varying softmax temperatures during the training steps.
    \subref{fig_n_policies} shows the final scores and the scores during the training process of \softHighwayDQN/ with varying numbers of lookahead policies.
    }\label{Ablation_Study}
\end{figure*}

\Cref{fig_score_minatar_greedy_step} illustrates the performance of the algorithms in MinAtar and MinAtar-Delay tasks.
\highwayDQN/ outperforms all competitors significantly in almost all tasks.
In the delayed reward setting, the performance of the competitors significantly deteriorates, while both our $n$-step Highway DQN and Soft Highway DQN maintain a considerable level of performance.
Regarding the advanced multi-step off-policy method Retrace($\lambda$), although it performs comparably with our \highwayDQN/ in the original setting, it fails to learn a policy in the delayed reward setting.
One possible reason for this is that the off-policy correction term in Retrace($\lambda$), which is the product of the clipped importance sampling ratios, exponentially decreases as the lookahead depth increases (refer to \Cref{ap_sec_experiment_minatar}). {This could lead to the loss of information when the reward is delayed greatly.}
Notably, in tasks like Freeway-Delay, where the delay can reach up to 2000, most competitors fail to achieve a score higher than 5. In contrast, our Soft and $n$-step Highway DQN can achieve scores close to 50.

\paragraph{Ablation Study.} 
We conduct several ablation studies to investigate the importance of the components of our \highwayDQN/ algorithms, including the operations within
$
{\color{softmaxPolicyStep} \underset{m\in \policySetModelFree_{s,a} }{\mathop {smax}\nolimits^{\alpha}}}
{\color{softmaxPolicyStep}\,
 \underset{n\in \mathcal{N} }{\mathop {smax}\nolimits^{\alpha}}
}$
$
{\color{maxStepOneN} \max\limits_{n'\in \{1,n\} }}\left( \cdot \right)$ 
(see \Crefnop{eq_highwayDQN}).
(1) \emph{The effect of the highway gate ${\color{maxStep} \max\limits_{n \in \{1,n\} }} $}. 
Following our discussion of the importance of the {highway gate} in \Cref{sec_highwayOperator}, we evaluate several variants of the $n$-step \highwayDQN/: two variants that replace the highway gate with 
${\color{maxStep} \max\limits_{n \in \{ {\color{BrickRed}3},n\} }} $ and ${\color{maxStep} \max\limits_{n \in \{ {\color{BrickRed}0},n\} }} $ respectively, and one variant that does not employ the highway gate (referred to as \emph{w/o $\max_{n'\in \{1, n\}}$}). 
As shown in \Cref{fig_highway_gate}, all the variants of the highway gates fail.
(2) \emph{The performance with various lookahead depth $n$}. 
We evaluate both our $n$-step \highwayDQN/ and $n$-step DQN with varying lookahead depths $n$.
As shown in \Cref{fig_n_step}, the performance of $n$-step DQN decreases as $n$ increases.
In contrast, our $n$-step \highwayDQN/ performs better with larger $n$.
(3) \emph{The effect of the softmax operation.} 
As shown in \Cref{fig_softmax}, our \highwayDQN/ performs well over a range of softmax temperatures from $0.0005$ to $0.005$. 
We also evaluate a softmax temperature of $\infty$ (which implies that the softmax operator degrades to the max operator), and the performance notably decreases.
(4) \emph{The effect of the number of lookahead policies.} 
As shown in \Cref{fig_n_policies}, incorporating more policies can often improve both the performance of the trained policy and sample efficiency.

\section{Conclusions}
We have introduced a novel multi-step Bellman Optimality Equation for efficient multi-step credit assignment in reinforcement learning (RL).
We proved that the solution of this equation is the optimal VF, and the corresponding operator generally converges faster than the traditional Bellman Optimality operator.
Our {\em Highway RL} methods combine the best of direct policy search---where credit assignment is performed without trying to identify useful environmental states or subgoals during runtime---and standard RL, which finds useful states and subgoals through dynamic programming. Highway RL quickly and safely extracts useful sub-policies derived from data obtained through previously tested policies.
The derived algorithms offer several advantages over existing off-policy algorithms and have been demonstrated to be feasible and effective on various standard benchmark datasets. 
Future work will focus on conducting a theoretical analysis of the algorithm's behavior in the model-free case.

\section*{Acknowledgments}
We would like to acknowledge {Dr.} A. Rupam Mahmood at the University of Alberta for useful feedback on this paper.
We acknowledge Pengcheng He for his assistance in the experiments.
Thanks to Dr. Jose A. Arjona-Medina for providing implementation details and experiment results on model-free toy tasks. This work was supported by the European Research Council (ERC, Advanced Grant Number 742870) and the Swiss National Science Foundation (SNF, Grant Number 200021 192356).

\bibliography{lib_my/bib,lib_my/bib_juergen}
\clearpage
\appendix

\renewcommand{\thetheorem}{\thesection.\arabic{theorem}}
\setcounter{theorem}{0}

\section{Theoretical Analysis}\label{ap_sec_theoy}

\subsection{Proofs of Theorems in \Cref{sec_foundation} of the Paper}

\theoremMultiStepBOOperatorBiased*
\begin{proof}
\def\QFixedPoint{  \widehat{Q}^* }
\def\nExist{ \grave{n} }
\def\piExist{ \grave{\pi} }
In the proof we will denote $r(s,a) := \E_{r\sim\mathcal{R}(\cdot|s,a)}r$.
For convenience we will also abbreviate $\E_{\pi\sim\policyDist}$ and $\E_{n\sim\stepDist}$ by $\E_{\pi}$ and $\E_{n}$ respectively, assuming distributions $\policyDist$, $\stepDist$ implicitely.

First, we claim that $\multistepBOOperator$ is a contraction
on complete metric space $l_{\infty}(\mathcal{S}\times\mathcal{A})$.
For arbitrary 
$Q,Q'\in l_{\infty}(\mathcal{S}\times\mathcal{A})$ it holds:
\begin{align*} 
    \left\|  \multistepBOOperator Q - \multistepBOOperator Q' \right\| 
    &
    \leq \mathbb{E} _{\pi \sim \policyDist,n\sim \stepDist}\left\| \left( \mathcal{B} ^{\pi} \right) ^{n-1}\mathcal{B} Q-\left( \mathcal{B} ^{\pi} \right) ^{n-1}{\mathcal{B} }Q' \right\| 
    \\
    &
    \leq \max_{\pi \in \policySet,n\in \stepSet}\left\| \left( \mathcal{B} ^{\pi} \right) ^{n-1}\mathcal{B} Q-\left( \mathcal{B} ^{\pi} \right) ^{n-1}{\mathcal{B} }Q' \right\| 
    \\
    &
    \le \gamma^n \left\| Q-Q' \right\|, 
\end{align*}
This proves the claim.

1) 
Applying Banach's fixed point theorem to $\multistepBOOperator$ (using the contraction property proved in the claim) we know that $\multistepBOOperator$ has exactly one fixed point.

The equation (\ref{eq:BpiBineq}) implies that:
\begin{equation}\label{app-eq_jksajajisisajaeqwwwsi}
Q^* 
\ge \E_{ \pi,n  } (\oneStepQQOperatorMath[\pi])^{n-1} \oneStepQQOperatorMath Q^*  
=\multistepBOOperator Q^*. 
\end{equation}

Using the monotonicity of $\multistepBOOperator$ ($Q \leq Q'$ implies $\multistepBOOperator Q \leq \multistepBOOperator Q'$), we get the monotonic sequence $Q^* \ge \multistepBOOperator Q^* \ge (\multistepBOOperator)^2 Q^* \ge \cdots$ which converges to the fixed point $\QFixedPoint$ based on the contraction property and Banach fixed point theorem:
\begin{equation}\label{app-eq_jaisjisiqhjiwjhjwiw}
Q^* \ge \multistepBOOperator Q^* \ge (\multistepBOOperator)^2 Q^* \ge \cdots
\ge (\multistepBOOperator)^k Q^* \searrow \QFixedPoint. 
\end{equation}

2)
For the implication
"$\Longleftarrow$"
it suffices to show for all $\pi \in \policySet$, $\policyDist(\pi) >0$
$\mathcal{B}^{\pi} Q^* = Q^*$. Then we get:
$$
\multistepBOOperator Q^*
=
\E_{\pi,n} (\mathcal{B}^{\pi})^{n-1} \mathcal{B} Q^*
=
\E_{\pi,n} Q^* = Q^*
$$
and the implication is proved.
Thus let us fix $\pi \in \policySet$, $\policyDist(\pi) >0$ and $s_0 \in \mathcal{S}, a_0 \in \mathcal{A}$ the
following holds
$
(\mathcal{B}^{\pi}Q^*)(s_0,a_0) =
\E_{s_1 \sim \mathcal{T}(\cdot|s_0,a_0)}
\E_{a_1 \sim \pi(\cdot|s_1)}
[r(s_0,a_0) + \gamma Q^*(s_1,a_1) ]
.
$
Since in the first expectation, we just care about $s_1$ for which
$\mathcal{T}(s_1|s_0,a_0) > 0$, we can assume $s_1 \in U$. As $\pi$ on
$U$ can be replaced by the optimal policy $\pi^*$ from the assumption, we get:
\begin{align*}
(\mathcal{B}^{\pi}Q^*)(s_0,a_0)
&=
\E_{s_1 \sim \mathcal{T}(\cdot|s_0,a_0)}
\E_{a_1 \sim \pi(\cdot|s_1)}
[r(s_0,a_0) + \gamma Q^*(s_1,a_1) ]
\\
&=
\E_{s_1 \sim \mathcal{T}(\cdot|s_0,a_0)}
\E_{a_1 \sim \pi^*(\cdot|s_1)}
[r(s_0,a_0) + \gamma Q^*(s_1,a_1) ] = Q^*(s_0,a_0)
.
\end{align*}

The remaining implication "$\Longrightarrow$" will be proved by
contradiction.
Assume that the conclusion does not hold, i.e., there exist
$\piExist \in \policySet$, $\policyDist(\piExist) >0$ and
$s_1 \in U$ such that
$\policyDist(\piExist) > 0$ and $\piExist(\cdot|s_1)$ is not optimal.
Since $s_1\in U$ there exists $s_0 \in \mathcal{S}, a_0 \in \mathcal{A}$
such that $\mathcal{T}(s_1|s_0,a_0) > 0$.
First, we aim to prove the following inequality:
$$
(\mathcal{B}^{\piExist}Q^*)(s_0,a_0) < Q^*(s_0,a_0) 
.
$$
Since $\piExist(\cdot|s_1)$ assigns positive probability to non-optimal action,
it is easy to obtain (especially for finite $\mathcal{A}$) that:
$$
\E_{a_1 \sim \piExist(\cdot|s_1)} Q^*(s_1,a_1) < V^*(s_1)
.
$$
For other states different from $s_1$ we can still have equality, but the countable sum leaves the inequality strict:
\begin{align*}
(\mathcal{B}^{\piExist} Q^*) (s_0,a_0)
&= 
r(s_0,a_0) + 
\E_{s_1' \sim \mathcal{T}(\cdot|s_0,a_0)}
\E_{a_1 \sim \piExist(\cdot|s_1')} Q^*(s_1',a_1)
\\
&<
r(s_0,a_0) + \E_{s_1' \sim \mathcal{T}(\cdot|s_0,a_0)} V^*(s_1')
= 
Q^* (s_0,a_0)
.
\end{align*}
Now since $(\mathcal{B}^{\piExist})^{\nExist-2} Q^* \leq Q^* $ (here we used the assumption $\nExist >1$) we get:
\begin{equation}
(\mathcal{B}^{\piExist})^{\nExist-1}\mathcal{B}Q^* 
=
(\mathcal{B}^{\piExist})^{\nExist-1}Q^* 
=
\mathcal{B}^{\piExist} (\mathcal{B}^{\piExist})^{\nExist-2}Q^* 
\leq
\mathcal{B}^{\piExist} Q^*
.\label{eq:app-mod1}
\end{equation}
Using the previous result we obtain:
\begin{equation}
\left(\left(\mathcal{B}^{\piExist}\right)^{\nExist-1}\mathcal{B} Q^*\right) (s_0,a_0) \leq \left(\mathcal{B}^{\piExist} Q^*\right) (s_0,a_0)
<
Q^* (s_0,a_0),\label{eq:app-mod2}
\end{equation}
Since $\policyDist(\piExist) >0$ and $\stepDist(\nExist) >0$ we get:
\begin{equation}
\left(\multistepBOOperator Q^*\right) (s_0,a_0) = \left(\E_{\pi,n} \left(\mathcal{B}^{\pi}\right)^{n-1}\mathcal{B} Q^*\right)(s_0,a_0) < Q^* (s_0,a_0)
,\label{eq:app-mod3}
\end{equation}
which can be combined with \cref{app-eq_jaisjisiqhjiwjhjwiw} to show:
$$
\QFixedPoint (s_0,a_0) < Q^* (s_0,a_0)
.
$$
\end{proof}

\theoremMultiStepBOOperatorDiverge*
\begin{proof}

1) 
Based on \Cref{theoremMultiStepBOOperatorBiased}, we have $\widehat{Q}_{N}^{*}, \widehat{Q}_{N^{\prime}}^{*} \leq Q^*$.
Therefore, we just need to prove that $\widehat{Q}_{N}^{*}\le \widehat{Q}_{N^{\prime}}^{*} $.
which  is  equivalent  to  prove that  $\widehat{Q}_{nN^{\prime}}^{*}\le \widehat{Q}_{N^{\prime}}^{*}$  for  any  $n\ge 2$ and  $n\in \mathbb{Z}$. The proof follows by induction.

First,  we  prove  that  $\left( \mathcal{B} ^{\pi} \right) ^{2N^{\prime}-1}\mathcal{B} \widehat{Q}_{N^{\prime}}^{*}\le \widehat{Q}_{N^{\prime}}^{*}$.

Note  that  $\mathcal{B} ^{\pi}\left( \mathcal{B} ^{\pi} \right) ^{N^{\prime}-1}\mathcal{B} \widehat{Q}_{N^{\prime}}^{*}\le \mathcal{B} \left( \mathcal{B} ^{\pi} \right) ^{N^{\prime}-1}\mathcal{B} \widehat{Q}_{N^{\prime}}^{*}$. By  applying  $\left( \mathcal{B} ^{\pi} \right) ^{N^{\prime}-1}$  to  both  side,  we  have
\begin{align*}
&\left( \mathcal{B} ^{\pi} \right) ^{N^{\prime}-1}\mathcal{B} ^{\pi}\left( \mathcal{B} ^{\pi} \right) ^{N^{\prime}-1}\mathcal{B} \widehat{Q}_{N^{\prime}}^{*}
=  \left( \mathcal{B} ^{\pi} \right) ^{2N^{\prime}-1}\mathcal{B} \widehat{Q}_{N^{\prime}}^{*}
\\
\le & \left( \mathcal{B} ^{\pi} \right) ^{N^{\prime}-1}\mathcal{B} \left( \mathcal{B} ^{\pi} \right) ^{N^{\prime}-1}\mathcal{B} \widehat{Q}_{N^{\prime}}^{*}
=\widehat{Q}_{N^{\prime}}^{*}
\end{align*}

Then,  we prove the induction hypothesis. Assume  $\left( \mathcal{B} ^{\pi} \right) ^{nN^{\prime}-1}\mathcal{B} \widehat{Q}_{N^{\prime}}^{*}\le \widehat{Q}_{N^{\prime}}^{*}$  for  a  specific  $n\ge 2$,  then we want to show that  $\left( \mathcal{B} ^{\pi} \right) ^{\left( n+1 \right) N^{\prime}-1}\mathcal{B} \widehat{Q}_{N^{\prime}}^{*} \le \widehat{Q}_{N^{\prime}}^{*}$.

Note  that  $\mathcal{B} ^{\pi}\left( \mathcal{B} ^{\pi} \right) ^{nN^{\prime}-1}\mathcal{B} \widehat{Q}_{N^{\prime}}^{*}\le \mathcal{B} \widehat{Q}_{N^{\prime}}^{*}$, By  applying  $\left( \mathcal{B} ^{\pi} \right) ^{N^{\prime}-1}$  to  both  side,  we  have
\begin{align*}
&\left( \mathcal{B} ^{\pi} \right) ^{N^{\prime}-1}\mathcal{B} ^{\pi}\left( \mathcal{B} ^{\pi} \right) ^{nN^{\prime}-1}\mathcal{B} \widehat{Q}_{N^{\prime}}^{*}=\left( \mathcal{B} ^{\pi} \right) ^{\left( n+1 \right) N^{\prime}-1}\mathcal{B} \widehat{Q}_{N^{\prime}}^{*}
\\
\le & \left( \mathcal{B} ^{\pi} \right) ^{N^{\prime}-1}\mathcal{B} \widehat{Q}_{N^{\prime}}^{*}=\widehat{Q}_{N^{\prime}}^{*}
\end{align*}

Therefore, by induction  we  have  $\left( \mathcal{B} ^{\pi} \right) ^{nN^{\prime}-1}\mathcal{B} \widehat{Q}_{N^{\prime}}^{*}\le \widehat{Q}_{N^{\prime}}^{*}$  for  any  $n\ge 2$.

Finally,  $\widehat{Q}_{N^{\prime}}^{*}\ge \left( \mathcal{B} ^{\pi} \right) ^{nN^{\prime}-1}\mathcal{B} \widehat{Q}_{N^{\prime}}^{*}\ge \left( \left( \mathcal{B} ^{\pi} \right) ^{nN^{\prime}-1}\mathcal{B} \right) ^2\widehat{Q}_{N^{\prime}}^{*}\ge \cdots =\widehat{Q}_{nN^{\prime}}^{*}$.

2) is obvious by the definition of $Q^\pi$ and \multistepBOOperatorText/ $\multistepBOOperator$.

\end{proof}

\theoremHighwayEquation*
\begin{proof}
1) The contraction property can be obtained:
\begin{eqnarray*}
 & &
 \left\|  \highwayOperator Q - \highwayOperator Q' \right\|
 \\
 &
    =& \left\| \E_{ \pi \sim \policyDist[], n  \sim \stepDist[] }
    \max_{n' \in \{1, n\} }
    \left( \left( \mathcal{B} ^{\pi} \right) ^{n'-1}\mathcal{B}  Q \right)   -\E_{ \pi \sim \policyDist[], n  \sim \stepDist[] }
    \max_{n' \in \{1, n\} }
    \left( \left( \mathcal{B} ^{\pi} \right) ^{n'-1}\mathcal{B}  Q' \right) \right\|\\
    &\leq& \max_{\pi \in \widehat{\Pi }} \max_{n \in \{1\}\cup\mathcal{N}}\left\| \left( \mathcal{B} ^{\pi} \right) ^{n-1}\mathcal{B}  Q    - \left( \mathcal{B} ^{\pi} \right) ^{n-1}\mathcal{B}  Q'  \right\|\\
    &\leq& \left\|\mathcal{B}  Q   - \mathcal{B}  Q'\right\|\\
    &\leq& \gamma 	\left\|   Q -  Q' \right\|.
\end{eqnarray*}
2) The fixed point uniqueness of the operator $\highwayOperator$ follows from the Banach fixed point theorem. Therefore it is enough to verify that $Q^*$ is a fixed point of $\highwayOperator$. Using the fact that $\mathcal{B}^{\pi}Q^*\leq\mathcal{B}Q^*=Q^*$ and the monotonicity of $\mathcal{B}^{\pi}$, we obtain:
\begin{equation}
\left(\mathcal{B} ^{\pi} \right)^{n-1}\mathcal{B}  Q^*\leq Q^*,
\label{eq:BpiBineq}
\end{equation}
for any $\pi$ and $n\geq 1$(with equality when $n=1$). Then, we have:
\begin{eqnarray*}
    \E_{ \pi \sim \policyDist[] }
\E_{ n  \sim \stepDist[] }
    \max_{n' \in \{1, n\} }
     \left( \mathcal{B} ^{\pi} \right) ^{n'-1}\mathcal{B}  Q^*  =Q^*.
\end{eqnarray*}
\end{proof}

\theoremNoonestepBiased*
\begin{proof}
\def\nExist{ \grave{n} }
Most of the proof for this remark was omitted for brevity because it is the same as the proof for \Cref{theoremMultiStepBOOperatorBiased} up to minor, trivial modifications
stemming from extra maximization over the set $\{n,n_1\}$.
We comment just on the end of point 2) where (\ref{eq:app-mod1}) and (\ref{eq:app-mod2}) have to be repeated also for $n_1 > 1$ to get the following:
$$
\left(\max_{n'\in \{n_1,\nExist\}} (\mathcal{B}^{\pi})^{n'-1}\mathcal{B} Q^*\right)(s_0,a_0) < Q^*(s_0,a_0)
$$
thus drawing an analogy to (\ref{eq:app-mod3}) for $\wronghighwayOperator Q^*$.
\end{proof}

\theoremCompareHighwayOperatorHighwayOptOperator*
\begin{proof}
    Firstly, if $Q\leq Q^*$, then the monotonicity of $\highwayOperator$ gives us    
    \begin{equation}
    (\forall \policyDist, \forall \stepDist): \quad
        Q^*\geq \max_{\policyDist,\stepDist}\highwayOperator Q\geq \highwayOperator Q.\label{Gppleqopt}
    \end{equation}
    Then, for any $(s,a)$ we have:
    \begin{multline*}
        \min_{\policyDist,\stepDist}\left|(\highwayOperator Q)(s,a)-Q^*(s,a) \right|
        = Q^*(s,a) - \max_{\policyDist,\stepDist} (\highwayOperator Q)(s,a) \\
        \\
    \begin{aligned}                
         &= Q^*(s,a) - \max_{\policyDist,\stepDist}\E_{ \pi \sim \policyDist[] }
\E_{ n  \sim \stepDist[] }
    \max_{n' \in \{1, n\} }
     \left(\left( \mathcal{B} ^{\pi} \right) ^{n'-1}\mathcal{B}  Q\right)(s,a) \\
     &= Q^*(s,a) - \max_{\pi\in\policySet}\max_{n\in\stepSet}
    \max_{n' \in \{1, n\} }
     \left(\left( \mathcal{B} ^{\pi} \right) ^{n'-1}\mathcal{B}  Q\right)(s,a) \\
     &= Q^*(s,a) - (\highwayOptOperator Q) (s,a) \\
     &= |Q^* - (\highwayOptOperator Q)| (s,a),
    \end{aligned}        
    \end{multline*}
    where in the last step we again used assumption $Q\leq Q^*$ and monotonicity of $\highwayOptOperator$ implying
    $\highwayOptOperator \leq Q^*$.

\end{proof}

\theoremCompareOnestepBOOperator*

\begin{proof}

1) Note that due to maximization over the set $\{1,n\}$ in the $\highwayOperator$ definition $\highwayOperator Q \geq \mathcal B Q$ for all $Q \in l_{\infty}(\mathcal{S}\times\mathcal{A})$.
Further, if $Q \leq Q^*$ then $\highwayOperator Q \leq Q^*$ and $\mathcal{B} Q \leq Q^*$ (since both operators are monotonic).
Putting all together we obtain:
$$
0 \leq Q^* - \highwayOperator Q \leq Q^* - \mathcal{B} Q,
$$
from which the $\distance[\highwayOperator[][]] \leq \distance[\BOOperator] $ follows.

Note that for all $\stepDist,\policyDist$ we have
    \begin{eqnarray}
        (\multistepBOOperator Q)(s,a)&=&\E_{ \pi \sim \policyDist }
\E_{ n  \sim \stepDist }
    \left( \left( \mathcal{B} ^{\pi'} \right) ^{n-1}\mathcal{B} Q\right)(s,a)\notag\\
    &\leq&\E_{ \pi \sim \policyDist }
\E_{ n  \sim \stepDist } \max_{ n' \in \left\{1,n\right\} }\left( \left( \mathcal{B} ^{\pi'} \right) ^{n'-1}\mathcal{B} Q\right)(s,a)\notag\\
    &\leq&(\highwayOperator Q)(s,a).\label{GmmgeqBppip}
    \end{eqnarray}
    Combining this with $Q^*\geq \highwayOperator Q$ (monotonicity of $\highwayOperator$) we have:
    \begin{align*}
         \left|(\multistepBOOperator Q)(s,a) - Q^*(s,a)\right| 
        = & Q^*(s,a) - (\multistepBOOperator Q)(s,a) 
        \\
        \geq & Q^* (s,a)- (\highwayOperator Q)(s,a) 
        \\
        = & \left|(\highwayOperator Q)(s,a) - Q^*(s,a) \right|,
    \end{align*}
    for all $(s,a)\in\sSpace\times\aSpace$, which directly implies $\distance[\highwayOperator[][]] \leq \distance[\multistepBOOperator[][]] $.

2) Using \Cref{theoremHighwayEquation} point 1) with $Q':= Q^*$ and together with the fact that 
$\highwayOperator Q^* = Q^*$ (\Cref{theoremHighwayEquation} point 2)) we have:
$$
(\forall Q \in l_{\infty}):\quad
\distanceAll[\highwayOperator](Q) = \| \highwayOperator Q - Q^* \| \leq \| \mathcal{B} Q - Q^* \| = \distanceAll[\mathcal{B}](Q)
.
$$

\end{proof}

\theoremHighwayStrictlyBetterOnestepMultistepBO*
\begin{proof}
Suppose condition 1) is satisfied. Note that ${\argmax _{n \in  \{1, n' \}  } \left( \left( \mathcal{B} ^{\pi'} \right) ^{n-1}\mathcal{B}  Q \right) \left( s, a \right) > 1}$ means that
\begin{eqnarray*}
    \max_{n\in\{1,n'\}}\left( \left( \mathcal{B} ^{\pi'} \right) ^{n-1}\mathcal{B}  Q \right) \left( s, a \right) >\left(\mathcal{B}  Q \right) \left( s, a \right).
\end{eqnarray*}
Given that $\highwayOperator Q \geq \mathcal B Q$ and that $\pi' \in \supp \policyDist$ and $n' \in \supp \stepDist$, taking expectation on both sides leads to
$\highwayOperator Q >\mathcal B Q.$
Combining this with $\highwayOperator Q \leq Q^*$ and $\mathcal{B} Q \leq Q^*$,gives
    $$
0 \leq Q^*(s,a) - \highwayOperator Q(s,a) < Q^*(s,a) - \mathcal{B} Q(s,a),
$$
which means $$\distance[\highwayOperator[][]] <  \distance[\BOOperator].$$
Suppose condition 2) is satisfied. Note that ${\argmax _{n \in  \{1, n' \}  } \left( \left( \mathcal{B} ^{\pi'} \right) ^{n-1}\mathcal{B}  Q \right) \left( s, a \right) = 1}$ means that
\begin{eqnarray*}
    \max_{n\in\{1,n'\}}\left( \left( \mathcal{B} ^{\pi'} \right) ^{n-1}\mathcal{B}  Q \right) \left( s, a \right) >\left( \left( \mathcal{B} ^{\pi'} \right) ^{n'-1}\mathcal{B}  Q \right) \left( s, a \right).
\end{eqnarray*}
Given (\ref{GmmgeqBppip}) and that $\pi' \in \supp \policyDist$ and $n' \in \supp \stepDist$, taking expectation on both sides leads to
$\highwayOperator Q >\multistepBOOperator Q.$
Combining this with $\highwayOperator Q \leq Q^*$ and $\multistepBOOperator Q \leq Q^*$, gives
\begin{eqnarray*}
        0\leq   Q^* (s,a)- \left(\highwayOperator Q\right)(s,a) \leq Q^*(s,a) - \left(\multistepBOOperator Q\right)(s,a),
    \end{eqnarray*}
which means $\distance[\highwayOperator[][]] <\distance[\multistepBOOperator[][]].$
If conditions 1) and 2) are satisfied for some specific $(s,a)$ and $(s',a')$ respectively, then by combining the above results we have
$$\distance[\highwayOperator[][]] < \min \left\{ \distance[\BOOperator], \distance[\multistepBOOperator[][]] \right\}.$$
\end{proof}

\theoremAlwaysLeqQopt*
\begin{proof}
    Assume $Q \leq Q^*$.
From the inequality (\ref{eq:BpiBineq}) and monotonicity of $\left(\mathcal{B} ^{\pi} \right)^{n-1}\mathcal{B}$ follows $\left(\mathcal{B} ^{\pi} \right)^{n-1}\mathcal{B} Q \leq Q^*$ which further implies  $\multistepBOOperator Q \leq Q^*$ and $Q^*\geq \highwayOperator Q$ for all $\policyDist$ and $\stepDist$. Additonally from (\ref{GmmgeqBppip}) and $\highwayOptOperator Q\leq Q^*$ we have $\multistepBEOperator Q\leq Q^*$. Now, note that the statement of the \Cref{theoremAlwaysLeqQopt} follows from the induction on $k$.
\end{proof}

\subsection{Other Theoretical Properties of \highwayOptOperatorTextMath/} \label{ap_sec_other_properties}

In this section, we provide some theoretical properties that are not discussed in the {main paper}.
We provide theoretical analysis for the case where \policySetTextMath/ changes over each iteration, as used in our \Cref{ap_alg_GreedyMultistepValueIteration}.
\begin{theorem}
For any $Q_0 \in l_{\infty}( \sSpace \times \aSpace )$ and any sequence of policy sets $( \policySet_k)$,$k\in \mathbb{N}$, the sequence
$(\highwayOptOperator[\policySet_k][\stepSet]
\circ\highwayOptOperator[\policySet_{k-1}][\stepSet]
\circ\ldots
\circ\highwayOptOperator[\policySet_1][\stepSet])[Q_0]$,$k \in \mathbb{N}$ converges R-linearly to $Q^*$ with convergence rate $\gamma$.
\end{theorem}
\begin{proof}
We claim that $\highwayOptOperator$ is a contraction on complete metric space
$l_{\infty}(\mathcal{S}\times\mathcal{A})$,i.e.,
$$
(\forall Q,Q' \in l_{\infty}(\mathcal{S}\times\mathcal{A})):
\quad
\|\highwayOptOperator Q - \highwayOptOperator Q'\| \leq \gamma \| Q-Q'\|,
$$
with $Q^*$ being its unique fixed point. The proof of the claim is
analogous to that of \Cref{theoremHighwayEquation} and is thus omitted for brevity.
From the claim, the following fact follows:
$$
(\forall Q \in l_{\infty}(\mathcal{S}\times\mathcal{A})):
\quad
\|\highwayOptOperator Q - Q^*\|  \leq \gamma \| Q-Q^*\|
,
$$
which implies that
$\|\highwayOptOperator[\policySet_k][\stepSet]
\highwayOptOperator[\policySet_{k-1}][\stepSet]
\ldots \highwayOptOperator[\policySet_1][\stepSet] Q_0 - Q^* \|
=
\|\highwayOptOperator[\policySet_{k}][\stepSet] (
\highwayOptOperator[\policySet_{k-1}][\stepSet]
\ldots \highwayOptOperator[\policySet_{1}][\stepSet] Q_0 )
-
\highwayOptOperator[\policySet_{k}][\stepSet] Q^* \|
\leq
\gamma \|\highwayOptOperator[\policySet_{k-1}][\stepSet]
\ldots \highwayOptOperator[\policySet_{1}][\stepSet] Q_0
- Q^*\|$.
By repeating the same argument, we end up with
$\|\highwayOptOperator[\policySet_{k}][\stepSet]
\highwayOptOperator[\policySet_{k-1}][\stepSet]
 \ldots \highwayOptOperator[\policySet_{1}][\stepSet]Q_0 - Q^* \| 
\leq \gamma^k \|Q_0 - Q^*\|$
from which the statement follows.

\end{proof}

\subsection{Theoretical Properties of  \highwaySoftmaxOperatorTextMath/ }

Finally, we show the convergence properties of \highwaySoftmaxOperatorText/ 
$\highwaySoftmaxOperatorMath$.
\def\theoremsofmaxHighwayOperator/{
    For any $\alpha$, any $\policySet$, and any $\stepSet$, we have
    $(\forall Q\in l_{\infty}(\sSpace\times\aSpace)):\:\| \highwaySoftmaxOperatorMath Q - Q^* \| 
    \leq
    \gamma \| Q - Q^* \| $
    and
    $\highwaySoftmaxOperatorMath Q^* = Q^* $.
}
\begin{theorem}\label{theorem_sofmaxHighwayOperator}
\theoremsofmaxHighwayOperator/
\end{theorem}
\begin{proof}
Recall that the \highwaySoftmaxOperatorText/ is defined as:
$$
    \highwaySoftmaxOperatorMath Q \triangleq 
{ \underset{\pi \in \widehat{\Pi }}{\mathop {smax}\nolimits^{\alpha}} } 
{ \underset{n'\in \mathcal{N}} {\mathop {smax}\nolimits^{\alpha} }}
\max_{n \in \{1, n'\} }
    \left( \mathcal{B} ^{\pi} \right) ^{n-1}\mathcal{B} Q.
$$

First, following the proof in \Cref{theoremHighwayEquation} equation (\ref{eq:BpiBineq}), we have
$
\left( \mathcal{B} ^{\pi} \right) ^{n-1}\mathcal{B} Q^* \le Q^*, 
$ 
for any $\pi$ and $n \geq 1$ (with equality when $n=1$).
Then, we have: 
$$
{ \underset{\pi \in \widehat{\Pi }}{\mathop {smax}\nolimits^{\alpha}} } 
{ \underset{n'\in \mathcal{N}} {\mathop {smax}\nolimits^{\alpha} }}
\max_{n \in \{1, n'\} }
    \left( \mathcal{B} ^{\pi} \right) ^{n-1}\mathcal{B} Q^*=Q^*.
$$

Given an action VF $Q$, 
let us define two distributions over $\widehat{\Pi }$ and $\mathcal{N}$, denoted by $\mathcal{P} _{\widehat{\Pi }}^{s,a}$ and $\mathcal{P} _{\mathcal{N}}^{s,a}$ such that:
$$
\mathbb{E} _{\pi \sim \mathcal{P} _{\widehat{\Pi }}^{s,a},n^{'}\sim \mathcal{P} _{\mathcal{N}}^{s,a}}\max_{n\in \{1,n'\}} \left(\left( \mathcal{B} ^{\pi} \right) ^{n-1}\mathcal{B} Q \right) \left( s,a \right) =\underset{\pi \in \widehat{\Pi }}{\mathop {smax}\nolimits^{\alpha}}\underset{n'\in \mathcal{N}}{\mathop{smax}\nolimits^{\alpha}}\max_{n\in \{1,n'\}} \left(\left( \mathcal{B} ^{\pi} \right) ^{n-1}\mathcal{B} Q \right)\left( s ,a \right). 
$$
It follows:
\begin{align*}
\left\| \highwaySoftmaxOperatorMath Q-  Q^* \right\| 
&
=\left\| \underset{\pi \in \widehat{\Pi }}{\mathop {smax}\nolimits^{\alpha}}\underset{n'\in \mathcal{N}}{\mathop {smax}\nolimits^{\alpha}}\max_{n\in \{1,n'\}} \left( \mathcal{B} ^{\pi} \right) ^{n-1}\mathcal{B} Q-Q^* \right\| 
\\ &
=\left\| \mathbb{E} _{\pi \sim \mathcal{P} _{\widehat{\Pi }},n^{'}\sim \mathcal{P} _{\mathcal{N}}}\max_{n\in \{1,n'\}} \left( \mathcal{B} ^{\pi} \right) ^{n-1}\mathcal{B} Q-\mathbb{E} _{\pi \sim \mathcal{P} _{\widehat{\Pi }},n^{'}\sim \mathcal{P} _{\mathcal{N}}}\max_{n\in \{1,n'\}} \left( \mathcal{B} ^{\pi} \right) ^{n-1}\mathcal{B} Q^* \right\| 
\\ &
\le \max_{\pi \in \widehat{\Pi }} \max_{n\in \mathcal{N}} \left\| \left( \mathcal{B} ^{\pi} \right) ^{n-1}\mathcal{B} Q-\left( \mathcal{B} ^{\pi} \right) ^{n-1}\mathcal{B} Q^* \right\| 
\\ &
\le \gamma \left\| Q-Q^* \right\|
.
\end{align*}

\end{proof}

\subsection{Advanced Multi-Step Off-Policy Operators}\label{sec_advanced_multistep_offpolicy_operator}

In this section, we describe several multi-step off-policy operators,   
{which aim to evaluate the value function of a policy $\pi'$ (called \emph{target policy}) using the data collected by a different policy $\pi$ (called \emph{behavior policy}).}
The classical methods are based on importance sampling techniques \citep{sutton2018reinforcement,cortes2010learning}, where the underlying operator can be defined as follows:
\begin{equation}\label{eq_multistepBEOperator}
\begin{aligned}
    (\multistepBEOperator Q)(s,a) 
     \triangleq 
\EPolicyStep
    \EE_{\trajnew[\pi][s,a][n]  } 
     \Bigg[
\sum_{\stepIndex=0}^{n-1} {\gamma^\stepIndex}\zeta _{\pi',\pi}^{1:\stepIndex}r_{\stepIndex}+\gamma ^n\zeta_{\pi',\pi}^{1:n}Q(s_n, a_{n})
    \Bigg].
\end{aligned}
\end{equation}
where $\pi'$ is the target policy, $\zeta _{\pi',\pi}^{1:\stepIndex}\triangleq \prod_{\stepIndexTwo=1}^{\stepIndex}{\frac{\pi' (a_{\stepIndexTwo}|s_{\stepIndexTwo})}{\pi(a_{\stepIndexTwo}|s_{\stepIndexTwo})}}$
is the product of {IS ratios}.
This operator is unbiased as we have $\multistepBEOperator Q^{\pi'} = Q^{\pi'}.$ However, its variance is high because it depends on the exponentiated R\'enyi divergence between the behavioral and target policies \citep{cortes2010learning}.

\citeauthor{munos2016safe} propose the following general operator for a return-based off-policy algorithm:
\begin{equation}\label{eq_retrace}
\mathcal{R} Q(s,a)\triangleq Q(s,a)+\mathbb{E} _{ \pi \sim \policyDist, \tau _{s,a}^{N}\sim \pi}\left[ \sum_{t=0}^{N-1}{\gamma ^t}\left( \prod_{t^{'}=1}^t{\zeta _{t^{'}}} \right) \left( r_t+\gamma \mathbb{E} _{\pi ^{'}}Q\left( s_{t+1},\cdot \right) -Q\left( s_t,a_t \right) \right) \right], 
\end{equation}
where $\pi$ is the behavioral policy, $\pi'$ is the target policy for policy evaluation, $(\zeta_{t'})$ are called \emph{traces} in the Retrace($\lambda$) paper \citep{munos2016safe}.
For Retrace($\lambda$) \citep{munos2016safe}, $\zeta_{t'}=\lambda \min \left( 1, \frac{  \pi'(a_{t'}|s_{t'}) }{ \pi(a_{t'}|s_{t'}) } \right)$.
For Q($\lambda$) \citep{harutyunyan2016q}, $\zeta_{t'}=\lambda $.

\section{Method}\label{ap_sec_algorithm}

\subsection{\highwayQLearningFull/}

We now present our \highwayQLearningFull/ for model-free RL with tabular action-VF $Q$.
The $k$-th VF $Q_k$ is updated in the following way: 

\begin{equation}
Q_{k+1}\left( s,a \right) =   
{\color{maxPolicyStep} 
 \underset{m\in \mathcal{M}_{s,a} }{\max}
 \underset{n\in \mathcal{N}_{s,a}}{\max}
}
{\color{maxStepOneN} 
    \max_{n'\in \{1,n\} }
}
\widehat{\mathbb{E} }_{\mathcal{D}_{s,a}^{( m )}}
\left[
\nstepReturn[Q_k][] ( \traj[n'][s,a] )
\right]
\end{equation}
where $\dataset[m][s,a]=\{ \traj[n][s,a] | \traj[n][s,a]\sim \policy[m] \}$ denotes the trajectory data collected by the policy $\policy[m]$,
$ \policySetModelFree_{s,a} \subseteq \left\{ m \big| | \dataset[m][s,a] | \neq 0 \right\}$ is a subset of indexes of the dataset that are not empty under $(s,a)$;
$\widehat{\mathbb{E} }_{\mathcal{D}_{s,a}^{( m )}}\left[  \cdot \right]=\frac{1}{\left|\mathcal{D} _{s,a}^{\left( m \right)}\right|}\sum_{\tau _{s,a}^{n}\in \mathcal{D} _{s,a}^{\left( m \right)}}^{}{\left[ \cdot \right]}$ is the empirical averaged value.
The \highwayQLearningFull/ algorithm is presented as \Cref{ap_alg_GM_QLearning}.

\begin{algorithm}[t]
\small
	\caption{\highwayQLearningFull/}
\begin{algorithmic}[1]\label{ap_alg_GM_QLearning}
\STATE \textbf{Input:} 
\highwayAlgorithmInput
\STATE \textbf{Initialize:} 
$k=0$;
State-action replay buffer $\dataset[][] = \emptyset$;
Initialize value function $Q\K[0] \in \R ^{|\sSpace| \times |\aSpace|} $;
.
\FOR{$m=1,\cdots,\epochRunAlg $}
    \STATE Set $\pi_m$ to be $\epsilon$-greedy policy with $Q\K$
    \STATE $\dataset[m][s,a] \leftarrow \emptyset$, for all $ (s,a) \in \sSpace \times \aSpace$
    \FOR{$j=1,\cdots,\epochRolloutPolicy $ }
    	\STATE Starting from initial state $s_0$, collect a trajectory $\tau=( s_{0}, a_{0}, r_{0}, s_{1}, a_{1}, r_1  , \cdots, s_T)$ with $\policy[m]$. 
    	\FOR{$t=0,1,\cdots,T-1$ }
    	    \STATE 
    		Add $( s_{t}, a_{t}, r_{t}, s_{t+1}, a_{t+1}, r_{t+1}  , \cdots, s_T)$ to $\dataset[m][s_t, a_t] $
    		\STATE
    		Add $( s_{t}, a_{t})$ to $\dataset[][] $
    	\ENDFOR
    \ENDFOR
	\FOR{$j=1, \cdots, \epochUpdateValueFunction$ }
	\STATE Sample a $(s, a)$ from $\dataset[][]$ and then sample a mini-batch $\policySetModelFree_{s,a}$ with size $M$ from $\left\{ m \big| | \dataset[m][s,a] | \neq 0 \right\}$.
	\STATE 
	Update the value function by the following rules
\begin{small}
	\begin{equation*}
    Q_{k+1}\left( s,a \right) = 
{\color{black} 
 \underset{m\in \mathcal{M}_{s,a} }{\max}
 \underset{n\in \mathcal{N}_{s,a}}{\max}
}
{\color{black} 
    \max_{n'\in \{1,n\} }
}
\widehat{\mathbb{E} }_{\mathcal{D}_{s,a}^{( m )}}
\left[
\nstepReturn[Q_k][] ( \traj[n'][s,a] )
\right]
    \end{equation*}
\end{small}
\hspace{-10.in}  
\highwayAlgorithmWhere
    \STATE $k=k+1$
	\ENDFOR
\ENDFOR
\end{algorithmic}
\end{algorithm}

\section{Experimental Results}\label{ap_sec_experiment}

\subsection{Experiments with Model-Based Algorithms}\label{ap_sec_model_based}

\subsubsection{Details of the Environments}

\begin{figure}
    \centering
    \includegraphics[width=0.6\linewidth]{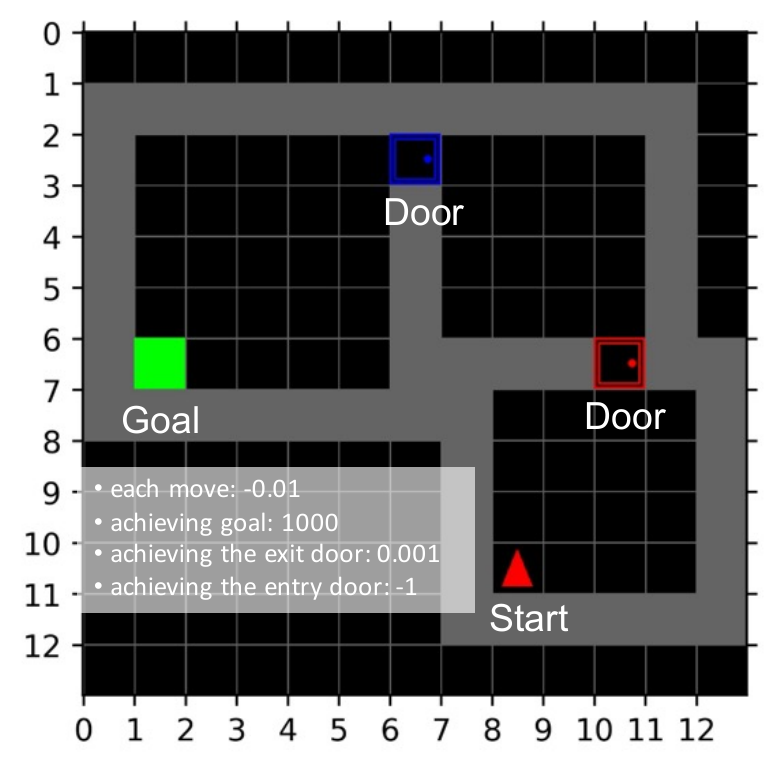}
    \caption{Illustration of the Minimalistic Gridworld Environment}
    \label{ap_fig_multi_room}
\end{figure}

\noindent
\textbf{Multi-Room}
is a grid world environment with multiple rooms connected by doors.
The agent's goal is to reach a goal square in the opposite corner and receive a reward ($r=1000$). 
In addition, the agent will receive a small reward of $r=0.001$ when it finds the exit door of the room.
We use implementations based on gym-minigrid \citep{gym_minigrid}.

\subsubsection{Implementation Details of the Algorithms}
We compare our \highwayValueIteration/ to Policy Iteration and Value Iteration.
For Policy Iteration, the lookahead depth is set to $N=10$.
For our \highwayValueIteration/ method, the \stepSetText/ is set by $\stepSet=\rangeInt[1][10]$, where $\rangeInt[l][u] \triangleq  \{ l, l+1, l+2,\cdots,u-1,u \}$ denote the set of all integers between $l$ and $u$; the size of lookahead policies is set by $|{\widehat \Pi}|=5$; the current policy is added to the set of lookahead policies every $K=7$ iteration.
The error bound ($ \| V\K[k]-V\K[k-1] \|_\infty \leq \epsilon $) for all algorithms is set to $\epsilon=10^{-10}$.

\subsection{Experiments with Model-Free Algorithms in Toy Tasks}

Here we present the details of our model-free algorithms.
We adopt the experimental setting from the RUDDER paper \citep{NEURIPS2019_16105fb9}.

\subsubsection{Details of the Environments}
We evaluate the model-free algorithms on two toy tasks involving delayed rewards \citep{NEURIPS2019_16105fb9}, where a reward is only provided at the end of each trial and is associated with the previous actions.
For example, in the task ``Trace Back,'' the final reward depends on the first two actions.
Each task is run with 100 random seeds.
In the task ``Choice,''  the stochastic reward depends on the first action at the beginning, and
the final reward depends on the first two actions.
Please refer to the RUDDER paper \citep{NEURIPS2019_16105fb9} for more details.

\subsubsection{Implementation Details of the Algorithms}

The following methods are compared:
\begin{itemize}
    \item RUDDER with reward redistribution for $Q$-value estimation, and RUDDER applied on top of $Q$-learning.
    \item $Q$-learning with eligibility traces according to Watkins ($Q(\lambda)$).
    \item  SARSA with eligibility traces (SARSA($\lambda$)).
    \item Monte Carlo.
\end{itemize}

The algorithms are evaluated until the task is solved. In all experiments, an $\epsilon$-greedy policy with $\epsilon=0.2$ is adopted. 
For our \highwayQLearningFull/,
the number of lookahead policies is set to $3$.
the \stepSetText/ $\stepSet$ is set to $\rangeInt[1][N]$, where $N=T-t$ represents the length until the end ($T$ is the length of the trajectory, and $t$ the timestep of the current state). The number of epochs for rolling out the policy is set to $\epochRolloutPolicy=1$. The number of epochs for updating the value function is set to $\epochUpdateValueFunction=T$.

For RUDDER, we use the default setting of \citep{NEURIPS2019_16105fb9}.
For Q($\lambda$) and SARSA($\lambda$), the hyperparameter of eligibility traces is $\lambda=0.9$. For Q($\lambda$), we use Watkins' implementation.

For Monte Carlo (MC), the Q-values are calculated as the exponential moving average of the episode return.

\subsection{Experiments with Model-Free Algorithms in MinAtar Tasks}\label{ap_sec_experiment_minatar}

\subsubsection{Details of the Environments}
This paper uses five games in MinAtar:
Asterix, Breakout, Freeway, Seaquest, and Space Invaders \citep{Kyo}.
To increase the challenge of reward delay, we introduce new \emph{MinAtar-Delay} tasks, where only the total score is provided at the end of the game, with no intermediate rewards during the agent's interaction.
To ensure that the reward function satisfies the Markov property, the delayed count, representing the number of nonzero original rewards that are delayed, is encoded within the state representation using one hot encoding.

\subsubsection{Implementation Details of the Algorithms}

\begin{table*}[t]
\centering
\begin{tabular}{cc}
\toprule
\textbf{Hyperparameter}                  & \textbf{Value}                                                                                                                                                                                                                       \\ \midrule
Optimizer                       & RMSprop                                                                                                                                                                                                                     \\ \midrule
Batch size                      & 32                                                                                                                                                                                                                          \\ \midrule
Gradient Clipping               & \begin{tabular}[c]{@{}c@{}} 1\end{tabular}                                                                                                                                                              \\ \midrule
Target network update frequency & \begin{tabular}[c]{@{}c@{}} 1000\end{tabular}                                                                                                                                                         \\ \midrule
Hidden layers of Q network       & \multicolumn{1}{l}{\begin{tabular}[c]{@{}l@{}}
 \hspace*{0.2in}      Conv. Layer(out channels=16, kernel size=3, stride=1) \\
 \hspace*{0.2in}      Linear Layer(out dimension=128)
\end{tabular}} 
\\ \midrule
Activation function             & ReLU                                                                                                                                                                                                                        \\ \midrule
Buffer size                     & \begin{tabular}[c]{@{}c@{}} 10$^5$\end{tabular}                                                                                                                                                    \\ \midrule
Discount factor $\gamma$        & 
\multicolumn{1}{l}{\begin{tabular}[c]{@{}l@{}}
\hspace*{0.2in}      0.996 (for MinAtar-Delay) \\
 \hspace*{0.2in}     0.99 (for MinAtar)
\end{tabular}} 
\\ \midrule
Exploration rate $\epsilon$     & Linearly decay starting from $1.0$ to $0.1$                      
\\ \midrule
Learning rate     & $2.5 \times 10 ^{-4}$                      \\ 
\bottomrule
\hspace*{0.2in}
\end{tabular}

\caption{
Hyperparameters of the implemented algorithms.
The hyperparameters and settings of the neural networks are those in the Maxmin DQN paper \citep{Lan2020Maxmin}. 
}
\label{ap_tab_hyperparameters_minatar}
\end{table*}

We compare our \highwayDQN/ to several advanced multi-step off-policy methods, including \emph{Retrace($\lambda$)} \citep{munos2016safe} and \emph{Multi-step DQN}.
All multistep methods, including our \highwayDQN/, are implemented on top of the open-source code of \emph{Maxmin DQN} \citep{Lan2020Maxmin, Explorer}.
The optimal number of target networks, which is a hyperparameter, is chosen from the set $\{1, 2,4\}$.
To further alleviate the overestimation issue, we replace the maximization with the expectation w.r.t. the $\epsilon$-greedy distribution $d$ when taking the values from the $Q$ function for computing the Q target, i.e., $\E_{ a_n' \sim d }[ Q_{\theta'}(s_{n'}, a_n') ]$.
We reuse the hyperparameters and settings of the neural networks from the Maxmin DQN paper \citep{Lan2020Maxmin}, which are also detailed in \Cref{ap_tab_hyperparameters_minatar}.

\subsubsection{Additional Results of Retrace($\lambda$)}
\Cref{fig_retrace_IS} shows the weights of the product of the clipped importance sampling ratios of Retrace($\lambda$).
As shown in the figure, these ratios exponentially decrease as the lookahead depth increases.
This can result in information loss when the reward is significantly delayed.

\begin{figure*}[b]
    \centering
   	\def\widthproperty{0.185}
   	\centerline{
	   	\includegraphics[width=0.4\linewidth]{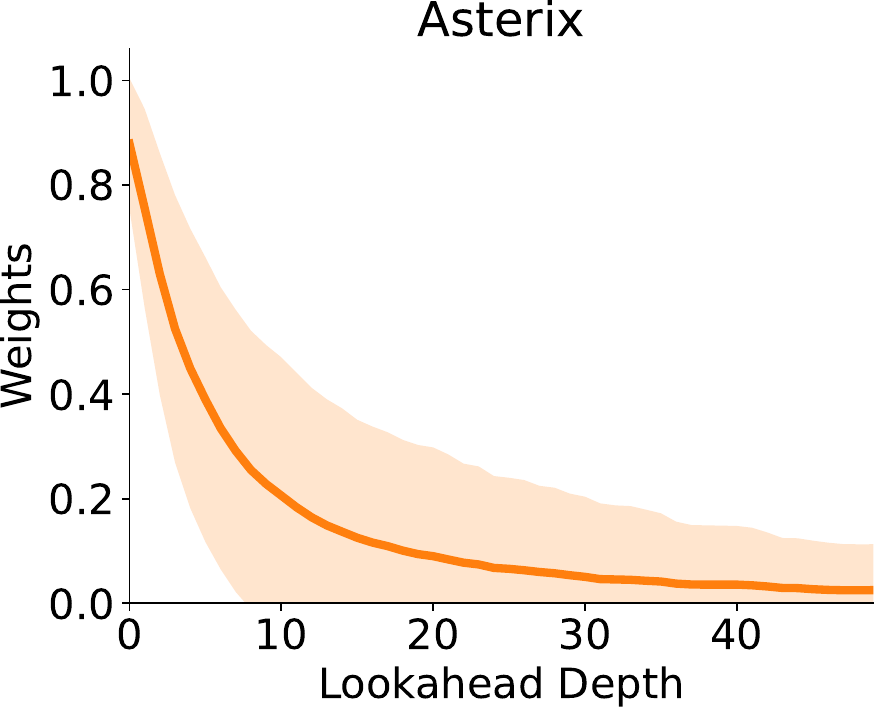}
   	}
    \caption{
    The weights of the product of the clipped importance sampling ratios of Retrace($\lambda$) as a function of the lookahead depth (bootstrapping step).
    }\label{fig_retrace_IS}
\end{figure*}

\subsubsection{Run Time}

In practice, we employ several measures to accelerate the computation of Highway DQN using GPU parallel computing. 
Firstly, we pre-compute and cache the $n$-step accumulated rewards $\sum_{n=0}^{N-1} \gamma^n r_{t+n}$ for each reward sequence $(r_t, r_{t+1}, \cdots, r_{t+N-1})$, which is accomplished through GPU parallel computing.
Secondly, the $n$-step returns and the softmax/max operation are implemented using matrix computation techniques, which are known to be fast on GPUs.

The running wall-clock times for the algorithms are as follows:
\emph{\highwayDQN/}: 8.9 h;
\emph{Retrace($\lambda$)}: 29.7 h;
\emph{Multi-step DQN ($N=4$)}: 5.6h;
\emph{DQN}: 4.3h.
Retrace($\lambda$) takes the most running time because it requires computing the product of importance sampling per timestep.
All results are obtained using the same experimental setup. The testing platform utilized is a workstation with 32GB memory and an NVIDIA RTX 2080 TI GPU.

\section{Limitations}
The theoretical analysis for our proposed operator focuses on the model-based case.
More work is required for a theoretical analysis of the algorithm's behavior in the model-free case, especially when the value function is approximated by a non-linear neural network.
Besides, we have empirically shown that our method is more efficient than the advanced Importance Sampling-based method Retrace($\lambda$) in environments with {greatly} delayed rewards.
More research is needed, however, to understand how the variance of our method compares to that of methods like Retrace($\lambda$). 

In practice, our algorithm balances the trade-off between accuracy and sample efficiency by deciding the number of trials per policy, the size of the search space (of behavioral policies and lookahead depths), and the softmax temperature.
While more trials per policy may improve the estimation accuracy, they may also cost more samples and reduce sample efficiency. 
On the other hand, while a larger search space may increase efficiency, it  might incur overestimation issues when the estimate is biased, leading to high variance.

\end{document}